\documentclass{fmcad-arxiv}  %

\usepackage[T1]{fontenc}

\usepackage{graphicx}
\usepackage{amsmath}
\usepackage{amsfonts}
\usepackage{amssymb}
\usepackage{amsthm}
\usepackage{mathtools} 
\usepackage{algorithm}
\usepackage{algpseudocode}

\algtext*{EndIf}

\usepackage{enumitem}

\usepackage{upgreek}

\usepackage{multirow}
\usepackage{tabularx}
\usepackage{threeparttable}
\usepackage{array}

\usepackage{wrapfig}
\usepackage{subcaption}
\usepackage{cleveref}

\usepackage{tikz}
\usetikzlibrary{external, positioning, shadows, arrows.meta, fit, backgrounds, calc}

\usepackage{booktabs}
\usepackage{tabularx}

\usepackage{pgfplots}
\pgfplotsset{compat=1.18}
\usepgfplotslibrary{fillbetween}

\usepackage[disable]{todonotes} %

\makeatletter
\newcommand{\alglabel}[1]{%
  \begingroup
    \addtocounter{ALG@line}{-1}%
    \refstepcounter{ALG@line}%
    \label{#1}%
  \endgroup
}
\makeatother

\crefname{ALG@line}{Line}{Lines}
\Crefname{ALG@line}{Line}{Lines}

\newcolumntype{L}[1]{>{\raggedright\let\newline\\\arraybackslash\hspace{0pt}}m{#1}}
\newcolumntype{C}[1]{>{\centering\let\newline\\\arraybackslash\hspace{0pt}}m{#1}}
\newcolumntype{R}[1]{>{\raggedleft\let\newline\\\arraybackslash\hspace{0pt}}m{#1}}

\renewcommand{\vec}[1]{\boldsymbol{\mathbf{#1}}}
\let\lo\underline
\let\up\overline

\DeclareMathOperator{\relu}{\mathrm{ReLU}}
\DeclareMathOperator{\elu}{\mathrm{ELU}}

\DeclareMathOperator{\arctanh}{\mathrm{arctanh}}
\DeclareMathOperator{\NN}{\mathcal{N}}

\DeclareMathOperator{\Opt}{\mathcal{P}_{back}}

\DeclareMathOperator{\I}{\mathbb{I}}
\DeclareMathOperator{\N}{\mathbb{N}}
\DeclareMathOperator{\Z}{\mathbb{Z}}

\DeclareMathOperator{\R}{\mathbb{R}}
\DeclareMathOperator{\X}{\mathcal{X}}

\DeclareMathOperator*{\argmin}{\mathrm{argmin}}

\newcommand{\aCROWN}[0]{$\alpha$-\textsc{CROWN}}
\newcommand{\abCROWN}[0]{$\alpha$-$\beta$-\textsc{CROWN}}
\newcommand{\autoLirpa}[0]{\textsc{AutoLiRPA}}
\newcommand{\linSyn}[0]{\textsc{LinSyn}}
\newcommand{\sol}[0]{\textsc{SOL}}
\newcommand{\eran}[0]{\textsc{ERAN}}

\newcommand{\swish}[0]{\textsc{Swish}}
\newcommand{\gelu}[0]{\textsc{GELU}}
\newcommand{\mish}[0]{\textsc{Mish}}
\newcommand{\lisht}[0]{\textsc{LiSHT}}
\newcommand{\atansq}[0]{\textsc{AtanSq}}
\newcommand{\loglog}[0]{\textsc{LogLog}}

\DeclarePairedDelimiter\abs{\lvert}{\rvert}%
\DeclarePairedDelimiter\norm{\lVert}{\rVert}%

\makeatletter
\let\oldabs\abs
\def\abs{\@ifstar{\oldabs}{\oldabs*}}
\let\oldnorm\norm
\def\norm{\@ifstar{\oldnorm}{\oldnorm*}}
\makeatother

\newtheorem{lemma}{Lemma}
\newtheorem{proposition}{Proposition}
\newtheorem{definition}{Definition}

\title{
    Shifting-based Optimizable Linear Relaxations for\\General Activation Functions
    \thanks{Thanks to Philipp for the nice collaboration.}
}

\author{
\IEEEauthorblockN{Philipp Kern}
\IEEEauthorblockA{
Karlsruhe Institute of Technology\\
philipp.kern@kit.edu
}
\and
\IEEEauthorblockN{L{\'a}szl{\'o} Antal and Erika {\'A}brah{\'a}m}
\IEEEauthorblockA{
RWTH Aachen University\\
\{antal,abraham\}@cs.rwth-aachen.de
}
\and
\IEEEauthorblockN{Carsten Sinz}
\IEEEauthorblockA{
Karlsruhe University of Applied Sciences\\
carsten.sinz@h-ka.de
}
}

\begin{document}

\maketitle

\begin{abstract}
The use of neural networks (NNs) is rapidly increasing, including in safety- and security-critical domains. To provide formal guarantees about NN behavior, many verification methods rely on optimizable linear relaxations of activation functions.
However, existing techniques depend on hand-crafted relaxations for each activation function.
Extension to state-of-the-art activation functions therefore requires substantial manual effort.

In contrast, our approach SLiR (Shifting-based Linear Relaxations) is broadly applicable, requiring only a Lipschitz constant or a set of critical points.
SLiR parameterizes relaxations by their slope and computes the corresponding offset via a shifting procedure that ensures sound upper and lower bounds over the input domain, enabling efficient optimization while maintaining correctness.
Our experiments show that SLiR produces tight relaxations across a wide range of practical activation functions and enables verification of up to $7.8 \times$ more properties compared to state-of-the-art methods.
\end{abstract}

\begin{figure*}[t]
    \centering
    \resizebox{\textwidth}{!}{
        \pgfdeclarelayer{outerstack}
\pgfdeclarelayer{innerstack}
\pgfdeclarelayer{midground}
\pgfsetlayers{outerstack,innerstack,midground,background,main}

\begin{tikzpicture}[
    neuronbox/.style={
        draw,
        fill=white,
        thick,
        rounded corners=2pt,
        text width=3cm,
        align=center,
        inner ysep=2pt,
        inner xsep=2pt,
        font=\large,
    },
    relaxbox/.style={
        draw,
        dashed,
        fill=blue!5,
        fill opacity=0.4,
        text opacity=1,
        rounded corners=5pt,
        inner xsep=2pt,
        inner ysep=2pt
    },
    arrow/.style={-{Stealth[scale=1.2]}, thick}
]

    \coordinate (inputsCoord) at (0,0);

    \coordinate (nnCoord) at ($(inputsCoord)+(0.5,0)$);

    \node[neuronbox, anchor=west] (startBounds) at (nnCoord) {Neural network};
    \node[neuronbox, anchor=west, xshift=2pt] (startNet) at ($(startBounds.east)+(0.25,0)$) {Input bounds};
    \node[neuronbox, anchor=west, xshift=2pt, text width=4cm] (startConstr) at ($(startNet.east)+(0.25,0)$) {Output constraints};

    \coordinate (startTop) at ($(startBounds.north)+(0.01,0.45)$);

    \begin{pgfonlayer}{outerstack}
        \node[
            neuronbox,
            fit=(startBounds) (startNet) (startConstr) (startTop),
            inner xsep=6pt,
            inner ysep=4pt,
            fill=gray!10,
            label={[anchor=north west, font=\bfseries, xshift=0.2cm]north west:Inputs}
        ] (start) {};
    \end{pgfonlayer}

    \coordinate[left=0.6cm of start] (startArrow);

    \coordinate (splitPoint) at ($(start.south west) + (0.8cm,-5.2cm)$);

    \node[neuronbox, right=0.5cm of splitPoint] (pwlApprox)
        {Piecewise linear approx. of activation function (Sec. \ref{subsec:pwl-envelopes})};

    \node[neuronbox, right=0.5cm of pwlApprox] (pwlEnvelope)
        {Sound lower and upper piecewise linear envelope (Sec. \ref{subsec:pwl-envelopes})};

    \node[neuronbox, above=0.3cm of pwlApprox] (exactOpt)
        {Exact optimization possible (Sec. \ref{subsec:exact-shift})};

    \coordinate[right=0.2cm of pwlEnvelope] (mergePoint);

    \node[neuronbox, right=1.2cm of pwlEnvelope] (chooseSlope) {Choose\\ slope $m_j$};
    \node[neuronbox, right=0.5cm of chooseSlope] (shift) {Compute\\ shift};

    \node[neuronbox, above=0.7cm of shift] (compVol)
        {
            Compute volume (local) 
            (Sec.~\ref{ssec:param-init})
        };

    \node[neuronbox, above=0.5cm of chooseSlope] (optVol)
        {
            Minimize volume (local) 
            (Sec.~\ref{ssec:param-init})
        };

    \draw[arrow] (chooseSlope) -- (shift);
    \draw[arrow] (shift) -- (compVol);
    \draw[arrow] (compVol) -- (optVol);
    \draw[arrow] (optVol) -- (chooseSlope);

    \node[
        anchor=north west,
        font=\bfseries\normalsize,
        inner sep=1pt,
        above of=optVol,
        xshift=1.2cm,
        yshift=0.3cm
    ] (neuronlevel) {Neuron-level optimization: Neuron $j$};

    \node[
        anchor=north west,
        font=\bfseries,
        above of=exactOpt,
        xshift=0.9cm,
        yshift=0.5cm
    ] (layerwise) {Layer-wise initialization: Layer $i$};

    \node[
        anchor=north west,
        font=\normalsize,
        below of=chooseSlope,
        xshift=1.7cm,
        yshift=0.3cm,
        inner sep=1pt
    ] (linrel) {Linear relaxation};

    \begin{pgfonlayer}{midground}
        \node[relaxbox, fit=(chooseSlope) (shift) (linrel), inner ysep=3pt] (neuronRelax) {};
    \end{pgfonlayer}

    \begin{pgfonlayer}{innerstack}
        \foreach \i in {3,2,1} {
            \node[
                neuronbox,
                fit=(neuronRelax) (compVol) (optVol) (neuronlevel) (linrel),
                inner sep=6pt,
                shift={(-\i*0.15,+\i*0.15)},
                fill=white
            ] (neuronLevelOpt\i) {};
        }
    \end{pgfonlayer}

    \begin{pgfonlayer}{outerstack}
        \foreach \i in {3,2,1} {
            \node[
                neuronbox,
                fit=(pwlApprox) (pwlEnvelope) (exactOpt) (layerwise) (neuronLevelOpt1) (splitPoint),
                inner xsep=7pt,
                inner ysep=7pt,
                shift={(-\i*0.15,\i*0.15)},
                fill=gray!10
            ] (initLayer\i) {};
        }
    \end{pgfonlayer}

    \node[
        neuronbox,
        inner xsep=2pt,
        text width=3.1cm,
        right=0.9cm of shift
    ] (computeBounds) {Compute\\ output bounds\\ (Sec. \ref{sec:background})};

    \node[
        neuronbox,
        inner xsep=2pt,
        text width=3.1cm,
        right=1.0cm of computeBounds
    ] (optBounds) {
        Optimize output bounds (global)
        (Sec.~\ref{ssec:param-opt})
    };

    \node[
        neuronbox,
        inner xsep=2pt,
        text width=3.1cm,
        above=0.75cm of optBounds
    ] (chooseSlopeNet) {Choose\\ slopes $m_1,...,m_k$};

    \node[
        neuronbox,
        inner xsep=2pt,
        text width=3.1cm,
        above=0.60cm of computeBounds,
        yshift=3pt
    ] (shiftNet) {Compute\\ shift};

    \node[neuronbox, below=0.25cm of optBounds] (result) {Verification result};

    \node[above=0cm of shiftNet, font=\normalsize, xshift=2.0cm] (lin)
        {Linear output relaxation};

    \begin{pgfonlayer}{midground}
        \node[relaxbox, fit=(chooseSlopeNet) (shiftNet) (lin)] (neuronRelaxNet) {};
    \end{pgfonlayer}

    \coordinate[above=0.8cm of neuronRelaxNet] (netOptTop);
    \begin{pgfonlayer}{outerstack}
        \node[
            neuronbox,
            fit=(computeBounds) (optBounds) (neuronRelaxNet) (netOptTop),
            inner xsep=4pt,
            inner ysep=4pt,
            fill=gray!10
        ] (netOpt) {};
        \node[
            anchor=north west,
            font=\bfseries,
            xshift=1cm,
            yshift=-0.1cm
        ] at (netOpt.north west) {Network-level optimization};
    \end{pgfonlayer}

    \draw[thick] ($(start.south west) + (0.8cm,0)$) -- (splitPoint);

    \draw[arrow] (splitPoint) |- (pwlApprox.west);
    \draw[arrow] (splitPoint) |- (exactOpt.west);
    \draw[arrow] (pwlApprox) -- (pwlEnvelope);

    \draw[thick] (pwlEnvelope.east) -| (mergePoint);
    \draw[thick] (exactOpt.east) -| (mergePoint);
    \draw[arrow] (mergePoint) -- (chooseSlope.west);

    \draw[arrow] (shift.east) -- (computeBounds.west);

    \draw[arrow] (computeBounds) -- (optBounds);
    \draw[arrow] (optBounds) -- (chooseSlopeNet);
    \draw[arrow] (chooseSlopeNet) -- (shiftNet);
    \draw[arrow] (shiftNet) -- (computeBounds);

    \draw[arrow] (computeBounds.south) -- ($(result.west) + (-2.7cm,0)$) -- (result.west);

    \draw[arrow] (startArrow) -- (start.west);

\end{tikzpicture}
    }
    \caption{Overview of our Shifting-based Linear Relaxation (SLiR) approach for neural network verification.
    \textcolor{red}{
    }
    }
    \label{fig:visual-abstract}
\end{figure*}
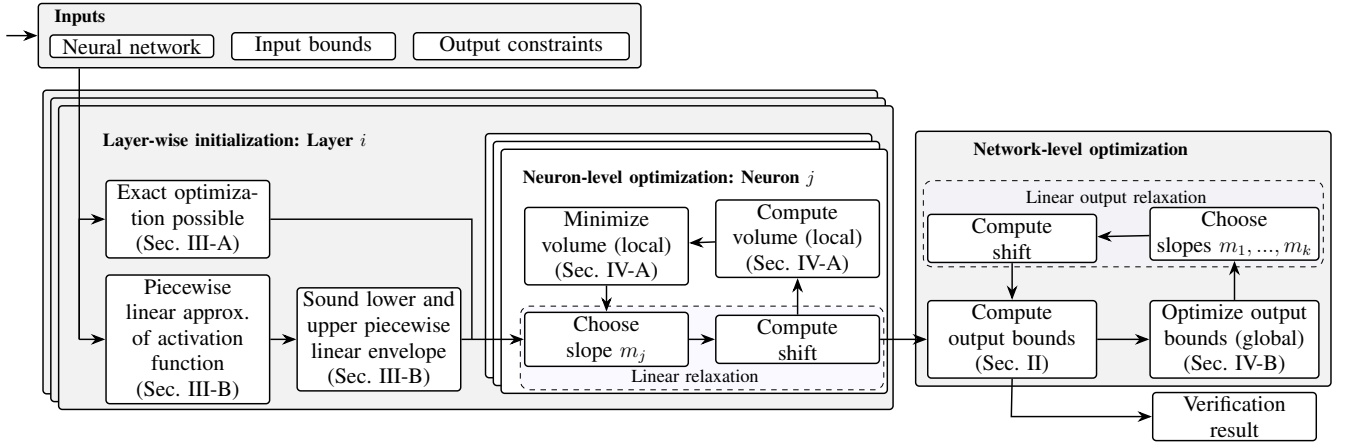

\section{Introduction}
\label{sec:intro}

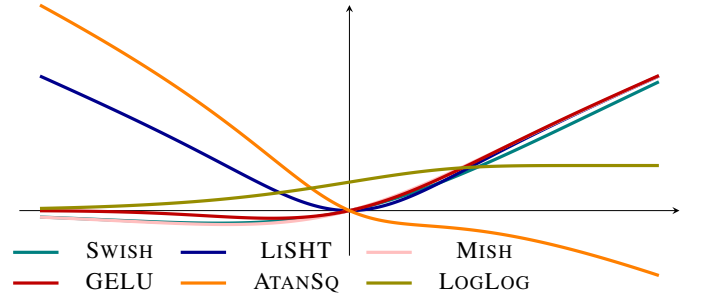
\begin{figure*}[t]
\centering

\begin{minipage}[t]{0.475\linewidth}
\centering
\setlength{\tabcolsep}{12pt}
\renewcommand{\arraystretch}{1.05}
\begin{tabularx}{\columnwidth}{@{}lX@{}}
\toprule
\textbf{Name} & \textbf{Definition} \\
\midrule
\swish{}  & $x \cdot \sigma(x)$, where $\sigma(x) = (1 + e^{-x})^{-1}$ \\[0.15em]
\gelu{}   & $x \cdot \Phi(x)$, where $\Phi(x)$ is the CDF of the Gaussian distribution \\[0.15em]
\mish{}   & $x \cdot \tanh\bigl(\ln(1 + e^x)\bigr)$ \\[0.15em]
\lisht{}  & $x \cdot \tanh(x)$ \\[0.15em]
\atansq{} & $\arctan(x)^2 - x$ \\[0.15em]
\loglog{} & $1 - e^{-e^{x}}$ \\
\bottomrule
\end{tabularx}%
\end{minipage}%
\hfill
\begin{minipage}[t]{0.525\linewidth}
\centering
\vspace{-6em}
\hspace*{2.4em}
\begin{tikzpicture}
\begin{axis}[
    width=\columnwidth,
    axis lines=middle,
    axis line style={-stealth, very thin},
    xmin=-3.20, xmax=3.20,
    ymin=-1.0, yscale=0.50,
    xscale=1.10,
    ticks=none,
    clip=false,
    legend style={
        draw=none,
        fill=none,
        font=\small,
        at={(0.35, 0.275)},
        anchor=north,
        legend columns=3,
        column sep=0.75em
    },
    samples=200,
    domain=-3.0:3.0,
]
\addplot[very thick, color=teal] { x * (1/(1 + exp(-x))) };
\addlegendentry{\swish{}}

\addplot[very thick, color=blue!55!black] { x * tanh(x) };
\addlegendentry{\lisht{}}

\addplot[very thick, color=pink] { x * tanh(ln(1 + exp(x))) };
\addlegendentry{\mish{}}

\addplot[very thick, color=red!75!black]
    { 0.5*x*(1 + tanh(sqrt(2/pi)*(x + 0.044715*x^3))) };
\addlegendentry{\gelu{}}

\addplot[very thick, color=orange] { (rad(atan(x)))^2 - x };
\addlegendentry{\atansq{}}

\addplot[very thick, color=olive] { 1 - exp(-exp(x)) };
\addlegendentry{\loglog{}}
\end{axis}
\end{tikzpicture}%
\end{minipage}%

\caption{Activation function definitions (left) and their visual illustration (right).}
\label{fig:activations_defs_plot}
\end{figure*}

Bound-propagation-based techniques are among the most successful approaches for scalable neural network verification. 
These methods overapproximate the nonlinear behavior of a network by propagating symbolic linear bounds through its layers, enabling efficient reasoning about safety or robustness.
A key factor behind the effectiveness of recent approaches is the use of \emph{optimizable linear relaxations} of activation functions, whose parameters (e.g.,  slope or a tangent point) are tuned to minimize concrete upper bounds on the output.
State-of-the-art verifiers such as \aCROWN{}~\cite{Xu21} and its successors rely heavily on gradient-based optimization.

Although this paradigm is effective for networks with ReLUs (and some other well-studied activation functions), extending optimizable linear relaxations to \emph{general activation functions} remains a challenge. 
Modern NNs increasingly rely on complex activation functions (see Fig. \ref{fig:activations_defs_plot}) such as \gelu{} (used in GPT~\cite{Radford2018}, BERT~\cite{Devlin2019}, and Vision Transformers~\cite{Dosovitskiy2021}), \swish{} and its variants (appearing in models such as LLaMA~\cite{Touvron2023}, YOLOv4~\cite{Bochkovskiy2020}, and non-edge versions of YOLOv7~\cite{Wang2023Silu}), \mish{}, or adaptively learned functions (used in physics-informed neural networks (PINNs)~\cite{wang2023learning}), for which no simple closed-form relaxations are available.
For such functions, existing approaches typically construct \emph{static} linear relaxations using sampling, SMT solving, or Lipschitz-based global optimization. 
Although these techniques can yield tight relaxations, they are fundamentally ill-suited for extending to slope optimization: soundness must be re-established from scratch for each candidate slope, rendering gradient-based optimization prohibitively expensive.

Our key observation is that, for any activation function $f(x)$,
a sound upper linear relaxation can be obtained by fixing a slope 
$m$ and computing an offset
\begin{equation}
b(m) = \max_{x \in [l, u]} f(x) - m x \label{eq:max-shift}
\end{equation}
that shifts the line $m x$ to upper-bound $f$ for its domain $[l,u]$.

This can be computed efficiently, either given a closed form for the critical points\footnote{A critical point of Eq.~\eqref{eq:max-shift} is a point $x\! \in\! [l, u]$ where the derivative
of $f(x)\!-\!mx$ is zero, i.e. where the slope of  $f$ equals the candidate slope $m$.} of Eq.~\eqref{eq:max-shift}, or by optimizing over a sound piecewise linear (PWL) surrogate that we need to construct  \emph{only once} before gradient-based optimization.
We show that closed-form solutions are available for a broad class of activation functions, and for the remaining cases PWL surrogates can be computed using Lipschitz-based optimization.
Since any choice of the slope~$m$ yields a valid upper relaxation (or a lower relaxation when taking the minimum), we can optimize over the slope parameter $m$ to obtain relaxations that lead to tighter bounds of the network output.
Our \emph{Shifting-based Linear Relaxations (SLiR)} approach, illustrated in Fig.~\ref{fig:visual-abstract}, integrates directly with modern bound-propagation frameworks and supports slope-optimized verification for activation functions unsupported by existing methods.

Our contributions in this paper are as follows:
\begin{enumerate}[label=(\textbf{C\arabic*}), leftmargin=*, itemsep=0.3em]
\item We introduce a novel framework for \emph{optimizable} linear relaxations, based on parameterization and a shifting formalism, which is applicable to general activation functions
  (Section~\ref{sec:approach}).
    \item If no closed form for the critical points is available, we replace the original activation function by a piecewise linear envelope \emph{once} using Lipschitz optimization and then optimize over that (Section~\ref{subsec:pwl-envelopes}) for efficiency. 
    \item Improvements for the Lipschitz optimization method tailored to our problem setting (Section~\ref{subsec:pyavskii-optim}), resulting in tighter overapproximation and faster convergence.
    \item We derive an admissible range of slope values from the convex hull of the graph of the activation function and apply parameter transformation to significantly improve (or enable in the first place) convergence of the output bound optimization presented in this work (Section~\ref{sec:slope-opt}).
\end{enumerate}

Proofs for lemmas and propositions stated in our paper can be found in Appendix~\ref{sec:proofs}.
An artifact with our code and experimental evaluation is available online at: \url{https://zenodo.org/records/20133926}.

Our approach significantly reduces the effort required for implementing new activation functions --- requiring only a formula for computing the critical points (in the exact case) or an overapproximation of the Lipschitz constant: 
while relaxations for the set of shared univariate activation functions require 1,340 lines of code in $\alpha$-CROWN, only 154 lines of code are sufficient using our approach. 
Moreover, handcrafting custom relaxations can be complex and difficult to handle, even for domain experts.
For the $\gelu$ and $\tanh$ functions, our approach yields better initializations than the hand-crafted implementation in $\alpha$-CROWN (see Appendix~\ref{ssec:failure-cases}).

\paragraph{Related Work}
Constructing tight linear relaxations has received significant attention in recent years, as the effectiveness of many NN verification algorithms depends on the quality of these bounds.
Various works focus on synthesizing \emph{non-optimizable} relaxations for \emph{specific} activation functions.
In particular, \cite{Balunovic2019,Ryou2021,Laurel2023} follow a two-stage approach:
they first approximate the function over the input domain using sampling-based techniques, and then restore soundness using dedicated correction procedures.
While these methods can yield tight bounds, they are typically tailored to individual activation functions and do not directly support optimization.

Paulsen et al.~\cite{Paulsen2022linsyn} also employ a sampling-based approach combined with SMT-based verification to ensure soundness.
Although applicable to a broad class of activation functions, this procedure is computationally expensive, as soundness must be re-established for each relaxation.
In \cite{Paulsen2022example}, the authors propose a template-based approach combined with machine learning.
However, this method requires predefined ranges for the bounds and may fall back to less precise techniques outside these ranges.

Biktairov et al.~\cite{Biktairov2023} synthesize linear overapproximations that are close to optimal, but compute only a single relaxation rather than a family of parameterized bounds.
Similarly, Ma et al.~\cite{Ma2025} construct convex hulls for multi-neuron constraints, but do not support slope optimization.
Related but more distant are synthesis- and SMT-based approaches such as \cite{Wang2023}, which construct linear or mixed-integer overapproximations but are not designed for efficient optimization.

Overall, existing methods for constructing \emph{optimizable} linear relaxations remain difficult to apply in practice and often require substantial expert knowledge (see, e.g., \cite{Shi2025}).
Decomposition-based approaches, such as \cite{Xu21}, enable optimization by breaking functions into simpler components, but can introduce loose bounds and increased computational cost.
Methods based on sampling, SMT solving, or repeated (expensive) soundness checks
are likewise poorly suited for gradient-based optimization, since soundness must be re-established at every iteration.

In contrast, our approach decouples soundness from parameter optimization by construction.
By parameterizing relaxations via their slope and computing sound offsets through a shifting procedure, we obtain a family of valid relaxations that can be optimized efficiently for general activation functions.

In the next section we provide background on linear relaxations for NN verification.
SLiR extends this method by changing the construction of linear upper and lower relaxations for activation functions in Eq.~\eqref{eq:act-approx}.

\section{Background}
\label{sec:background}

\paragraph{Neural Networks} A feed-forward NN is a function $\NN: \R^{d_0} \to \R^{d_L}$ with an input-, an output- and $L-1$ hidden layers.
The $k$-th layer consists of $d_k$ neurons and its activation values $\vec{n}_k$ and preactivation values $\hat{\vec{n}}_k$ can be computed as
\begin{equation*}
    \hat{\vec{n}}_k = W_k \vec{n}_{k-1} + \vec{b}_k \qquad
    \vec{n}_k = f\!\left(\hat{\vec{n}}_k\right) \;,
\end{equation*}
where $W_k$ and $\vec{b}_k$ are the weight matrix and bias of the $k$-th layer,
and $f: \R \to \R$ is a non-linear activation function applied element-wise.
To compute $\vec{y} = \NN(\vec{x})$, we identify the input $\vec{x} \in \R^{d_0}$ with $\vec{n}_0$ and the output $\vec{y} \in \R^{d_L}$ with $\hat{\vec{n}}_L$.
This representation also covers convolutional networks, since convolutions can be represented as matrix multiplications \cite{DBLP:journals/corr/abs-2002-12920}.

\paragraph{NN Verification} The NN verification problem that we consider consists of proving properties of input-output relations of the network. 
These properties are expressed as linear inequalities $\vec{a}^T \vec{y} \leq d$ over the output $\vec{y} = \NN(\vec{x})$, which have to hold for all inputs $\vec{x}$ within a hyper-rectangle $\X \subseteq \R^{d_0}$.

Instead of directly proving these universally quantified properties, we consider the equivalent optimization problem
\begin{equation}
    v^* = \max_{\vec{x} \in \X} \ \vec{a}^T \vec{y} - d \quad
                        \textit{such that} \quad \vec{y} = \NN(\vec{x})\enspace,
     \label{eq:opt_nn}
\end{equation}
that finds the \emph{maximum violation} $v^*$ of the property $\vec{a}^T\vec{y}\leq d$
(properties can be checked independently).
If $v^*$ is non-positive, the property is proven safe;
otherwise, it is violated.

\paragraph{Linear Relaxations for NN Verification} 

This optimization problem, due to the activation functions, is highly non-convex in general and therefore difficult to solve exactly.
Thus, many approaches focus on solving a relaxation of Eq.~\eqref{eq:opt_nn}, where the activation functions are overapproximated using linear constraints, yielding a sound but incomplete linear programming (LP) formulation.
However, modern NN verifiers~(\cite{Singh19deeppoly,Xu2020,Zelazny2022}) do not rely on standard LP solvers:
Restricting the overapproximation of each activation function to exactly two constraints
\begin{equation}
  \lo{f}(x) := \lo{\alpha} x + \lo{\beta} \leq f(x) \leq \up{\alpha} x + \up{\beta} =: \up{f}(x) \label{eq:act-approx}  
\end{equation}
valid for all preactivation values $x \in [l, u]$, enables them to instead use a specialized bound propagation procedure.
This technique, referred to as \emph{backsubstitution}~(\cite{Singh19deeppoly,Xu2020,Wong2018provable}), computes the same bounds on $v^*$ as a standard LP solver on the two-constraint LP~(\cite{Salman2019barrier,Lyu2020fastened}) under mild conditions, but is significantly more efficient (see Appendix~\ref{sec:backsubs} for more details).

The preactivation bounds necessary to construct the linear relaxations given in Eq.~\eqref{eq:act-approx} can be computed by applying the same procedure recursively to bound preactivation values of earlier layers --- once from above and once from below.
We denote this procedure to compute
an upper bound $\up{v^*}$ of the maximum violation $v^*$ as %
\begin{align}
    \up{v^*} = \textsc{Backsubstitute}(\NN, \X) \;. \label{eq:fun-backsubs}
\end{align}
If $\up{v^*} \leq 0$, then the property is proven successfully; otherwise, no conclusions can be drawn.
Due to the recursive bounds computation, the runtime of this approach is in $\mathcal{O}(L^2)$. %

Concerning the linear upper and lower relaxations in Eq.~\eqref{eq:act-approx}, computed for each neuron,
most methods (\cite{Paulsen2022linsyn,Paulsen2022example,Biktairov2023,Singh19deeppoly}) choose $\lo{\alpha}, \lo{\beta}, \up{\alpha}$ and $\up{\beta}$ such that the area between $\up{\alpha} x + \up{\beta}$ and $\lo{\alpha} x + \lo{\beta}$ over a given bounded range of $x \in [l, u]$ is minimized.
However, while these relaxations are \emph{locally} optimal, they may not lead to the tightest possible concrete upper bound on $v^*$ achievable by linear relaxations \cite{Xu21}. 
Some approaches therefore parameterize the relaxations as
\[
\lo{f}(x \mid \theta)
= \lo{\alpha}(\theta)\, x + \lo{\beta}(\theta)
\qquad
\up{f}(x \mid \theta)
= \up{\alpha}(\theta)\, x + \up{\beta}(\theta)\;,
\]
where $\lo{\alpha}(\theta), \lo{\beta}(\theta), \up{\alpha}(\theta)$ and $\up{\beta}(\theta)$ yield slope and bias coefficients for valid lower and upper relaxations for each parameter~$\theta$.
The search for the smallest possible upper bound can then be expressed as another optimization problem $\Opt$ over the stacked parameters $\vec{\theta}$ for all neurons in $\NN$:
\begin{align}
    \Opt : \quad \min_{\vec{\theta}} ~ \textsc{Backsubstitute}(\NN, \X \mid \vec{\theta}) \;. \label{eq:bound-opt}
\end{align} %
This optimization problem can then be solved using gradient-based methods.
If multiple linear properties need to be verified, $\Opt$ minimizes the sum of their violations.

\section{General Approach}
\label{sec:approach}

\begin{figure}[t]
    \centering
    \includegraphics[width=0.8\linewidth]{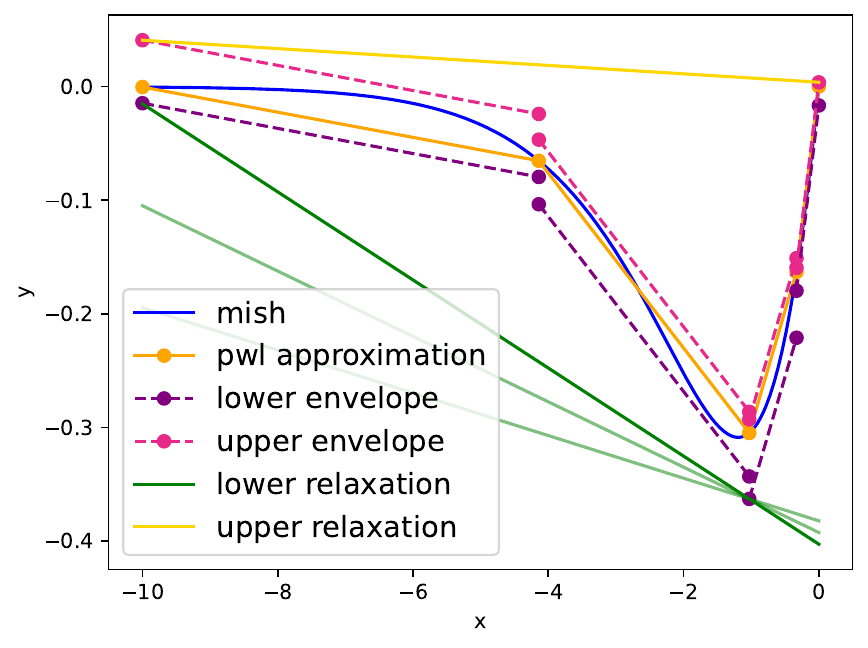}
    \caption{Illustrating our approach for \mish{}: (i) Fit a PWL approximation of the function, (ii) extend it to lower/upper envelopes, and  (iii) use them to construct linear relaxations.}
    \label{fig:overview}
\end{figure}

Solving problem $\Opt$ %
requires a parameterization of relaxations $\lo{f}(x \mid \theta) = \lo{\alpha}(\theta)x + \lo{\beta}(\theta)$ and $\up{f}(x \mid \theta) = \up{\alpha}(\theta)x + \up{\beta}(\theta)$ in the first place.
In our approach, we only optimize over the slope values $\lo{\alpha}(\theta)$ and $\up{\alpha}(\theta)$ ensuring soundness of the relaxations via \emph{shifting} up or down by a sufficient amount
\begin{align}
    \up{\beta}(\theta) = \max_{x \in [l, u]} f(x) - \up{\alpha}(\theta) \cdot x \;, \label{eq:problem-def}
\end{align}
where $[l, u]$ are the preactivation bounds for the neuron that needs to be relaxed.
Similarly, we set $\lo{\beta}(\theta)$ to the corresponding minimum value for the lower relaxation.

In cases --- as discussed in Sec. \ref{subsec:exact-shift} --- where the critical points of $f(x) - \up{\alpha}(\theta) \cdot x$ for $x \in [l, u]$ can be computed exactly, we can directly solve Eq.~$\eqref{eq:problem-def}$ to obtain sound relaxations.
Otherwise, the approximate solution to Eq.~\eqref{eq:problem-def} must guarantee soundness.
Prior work 
has already followed a similar approach using interval (\cite{Paulsen2022linsyn,Paulsen2022example}) or Lipschitz optimization \cite{Biktairov2023} as overapproximating methods.
However, calling these procedures for every neuron during every iteration of the gradient-descent process for $\Opt$ is prohibitively expensive.

Our solution to this problem is illustrated in Fig. \ref{fig:overview}:
Instead of calling these solvers in every iteration, we compute a PWL approximation and extend it to sound PWL lower and upper envelopes  before the first pass of the gradient-descent iteration (\Cref{sssec:envelopes}).
Then, we use the lower or upper envelopes as surrogates for the original activation function in Eq. \eqref{eq:problem-def}.
This approach is efficient, since optimization over PWL functions is cheap and guarantees soundness, as the lower and upper envelopes enclose the original activation function.

In both cases, (1) exact and (2) overapproximate optimization, our approach is easy to implement and to extend to new functions.
In the latter case, we only require a Lipschitz constant.
In the first case, we only need a formula (see \Cref{subsec:exact-shift} for further details) for computing the critical points.

In the remainder of this section, we use a simplified form of Eq.~\eqref{eq:problem-def}, disregarding the parametrization: 
$b=\max_{x\in[l,u]} f(x)-m x$, 
where $f$ is the activation function, $m$ a \emph{fixed} candidate slope, and $[l,u]$ the input domain.
\todo{Is this just simplified notation? Upper/lower min/max becomes a bit fuzzy by that.}

\subsection{Exact Shifting}
\label{subsec:exact-shift}

For some functions, we can exactly solve Problem~\eqref{eq:problem-def}.
This is the case for functions $f$ where the pre-image $f'^{-1}(y) = \left\{x \mid f'(x) = y\right\}$ of the derivative of $f$ is computable:
\begin{proposition}
    Let $f: \R \to \R$ be differentiable over the interval~$[l, u]$ and let $f'^{-1}(y) = \left\{x \mid f'(x) = y\right\}$ be the pre-image of the derivative of $f$. Then
    \[
    \begin{array}{lcl}
        b \!\!\!\! &=& \!\!\!\! \max_{x \in [l, u]} f(x) - m \cdot x \\[0.5ex]
          \!\!\!\! &=& \!\!\!\! \max \! \big\{ f(x^*\!) \! - \! m \!\cdot\! x^* \, | \, x^* \!\in\! \big(\{l, u \} \!\cup\! \big(f'^{-1}(m) \!\cap\! [l, u] \big)\big) \big\} \;. 
    \end{array}
    \]
    \label{prop:preim}
\end{proposition}
If there is a closed form expression for $f'^{-1}(m)$ and the set $f'^{-1}(m) \cap [l, u]$ is finite, then we can compute the maximum by just evaluating a \emph{finite} number of points.
Many functions fit this category.
Examples include sigmoid, $\tanh$ and $\elu$ \cite{Clevert2015elu}:
\begin{align*}
    \sigma(x) \! = \frac{1}{1 \! + \! e^{-x}} ,
    \hspace{0.95em} \tanh(x) ,
    \hspace{0.95em} \mathrm{ELU}(x) \! = \!
    \begin{cases}
        x, & \!\!\! x > 0\\
        e^x \! - \! 1, & \!\!\! \text{otherwise}
    \end{cases}
\end{align*}
where the %
critical points can be computed using
\[\begin{array}{lcl}
  \multicolumn{3}{l}{\sigma'^{-1}(m) \! = \! \big\{
    \!\! -\!\log ( \frac{1}{s_i} \!-\! 1 ) \, \big| \, s_{1,2} = \frac{1}{2} \pm \sqrt{\frac{1}{4} - m}, ~m \!\in\! (0, \frac{1}{4}]
  \big\}} \\
    \tanh'^{-1}(m) \!\!\!\!\! &=& \!\!\!\! \left\{\arctan(s_i) \; \middle| \; s_{1,2} = \pm \sqrt{1 - m}, ~m \in (0, 1]\right\}\\
    \mathrm{ELU}'^{-1}(m) \!\!\!\!\! &=& \!\!\!\! \{0\} \cup \left\{ \log(m) \mid m > 0 \right\} \;.
\end{array}
\]
Further functions of this category are listed in Appendix~\ref{sec:closed-form-crit}.

\subsection{Verified Piecewise Linear Envelopes}
\label{subsec:pwl-envelopes}

If there is no closed form for $f'^{-1}(m)$, then we replace the function $f$ in Eq.~\eqref{eq:problem-def} with a sound upper PWL envelope $f_{\mathit{pwl}}$, and maximize $f_{\mathit{pwl}}(x) - m \cdot x$.

As a first step, we construct a PWL \emph{approximation} $\tilde{f}_{\mathit{pwl}}$
(that is not necessarily sound for verification)
and then apply a sufficiently large shift \emph{to each linear segment} to achieve a valid upper envelope $f_{\mathit{pwl}}$.
Our approach works for any $\tilde{f}_{\mathit{pwl}}$, however, its tightness will affect the tightness of the upper envelope and thus the linear relaxations.
Additionally, Eq.~\eqref{eq:problem-def} needs to be solved many times during gradient-based search for the best slopes.
Here, fewer linear segments lead to higher efficiency.
Therefore, we use an approach that tries to construct $\tilde{f}_{\mathit{pwl}}$ with linear segments that are maximally long while approximately maintaining a target approximation error (see Appendix~\ref{sssec:pwl-approx}).

\subsubsection*{Construction of Sound Envelopes}
\label{sssec:envelopes}

Given any PWL approximation $\tilde{f}_{\mathit{pwl}}(x)$ of an activation function $f: \R \to \R$, specified by the set of segment boundaries $\Xi = \left\{ \xi_i \mid i = 1, \dots, n+1 \right\}$, we compute an overapproximation
\begin{align}
    \hat{y}_i^* \geq \max_{x \in [\xi_i, \xi_{i+1}]} f(x) - \gamma_i \cdot x, \quad \forall i = 1, \dots, n
    \label{eq:overapprox-max}
\end{align}
of the shifting distance for $\gamma_i \cdot x$ to be a valid overapproximation of $f(x)$ over the interval $[\xi_i, \xi_{i+1}]$, where $\gamma_i \cdot x + \delta_i$ is the $i$-th linear segment of $\tilde{f}_{\mathit{pwl}}$ over the interval $[\xi_i, \xi_{i+1}]$.
The extension of the approximation $\tilde{f}_{\mathit{pwl}}$ to a verified upper envelope is then
\begin{align}
    f_{\mathit{pwl}}(x) = \begin{cases}
        \gamma_i \cdot x + \hat{y}_i^* , &x \in [\xi_i, \xi_{i+1}), ~i \leq n \\
        \gamma_{n} \cdot x + \hat{y}_{n}^* , &x = \xi_{n+1}\\
        \bot , & \text{otherwise} \; .
    \end{cases}
\end{align}
Similarly, we use an underapproximation of the minimum to construct a valid lower envelope.

Note that, due to different $\hat{y}_i^*$ in neighboring segments, $f_{\mathit{pwl}}$ is not necessarily continuous (see Fig.~\ref{fig:overview}).
This does not affect correctness, since verification only requires pointwise overapproximation of the activation function.

To compute the overapproximate shift $\hat{y}_i^*$, we use an extension to the Lipschitz-based Piyavskii algorithm
\cite{PIYAVSKII197257}, which we introduce in Section \ref{subsec:pyavskii-optim}.
This kind of overapproximate optimization methods can often compute tighter bounds on smaller intervals.
When the segment bounds $[\xi_i, \xi_{i+1}]$ of $\tilde{f}_{\mathit{pwl}}$ are smaller than the domains defined by the pre-activation bounds $[l_j, u_j]$ of each neuron $j$, our approach can benefit from tighter bounds compared to approaches that require verified optimization of $f(x) - m \cdot x$ for each neuron in a NN (\cite{Paulsen2022linsyn,Biktairov2023}).

Additionally, the number of optimization problems we have to solve only depends on the width $[l, u] \!=\! [\min_j l_j, \max_j u_j]$ of the pre-activation bounds of all neurons $j$ in a layer of the NN and the corresponding number of linear segments of $\tilde{f}_{\mathit{pwl}}$.
This can be advantageous when the number of neurons is large compared to the number of linear segments of~$\tilde{f}_{\mathit{pwl}}$.

\subsubsection*{Optimization over PWL Envelopes}
\label{sssec:opt-envelopes}

Because $f_{\mathit{pwl}}$ is not necessarily continuous, it is not sufficient to consider just $\{f_{\mathit{pwl}}(x) - m \cdot x \mid x \in \{l, u, \xi_1, ..., \xi_{n+1}\} \cap [l, u]\}$, which only evaluates each linear segment for its left boundary point, as candidates for the optimal value of Eq.~\eqref{eq:problem-def}.
Instead, the set containing evaluations of both boundary points for each segment
\begin{align}
    &\{\gamma_1 \cdot l + \hat{y}_1^* - m \cdot l, ~~ \gamma_n \cdot u + \hat{y}_n^* - m \cdot u\}\\ 
    \cup ~ &\{\gamma_i \cdot \xi_i + \hat{y}_i^*  - m \cdot \xi_i \mid i \in 1,...,n\}\\
    \cup ~ &\{\gamma_i \cdot \xi_{i+1} + \hat{y}_i^*  - m \cdot \xi_{i+1} \mid i \in 1,...,n\}
\end{align}
is guaranteed to contain the optimal value.
Here, the indices $1,...,n$ correspond to the segments intersecting with $[l, u]$.
Using that the linear terms $(\gamma_i - m) \cdot x + \hat{y}_i^*$ are monotonic, we can halve the number of boundary points that need to be evaluated.
The maximum value over each linear segment is then
\begin{align}
    \max_{x \in [\hat{l}, \hat{u}]} \! (\gamma_i \!-\! m) \!\cdot\! x + \hat{y}_i^* \!=\! \begin{cases}
        \!(\gamma_i \!-\! m) \!\cdot\! \hat{l} \!+\! \hat{y}_i^*, & \!\!\!  \gamma_i \!-\! m \leq 0 \\
        \!(\gamma_i \!-\! m) \!\cdot\! \hat{u} \!+\! \hat{y}_i^*, & \!\!\! \text{otherwise}
    \end{cases}
\end{align}
where $[\hat{l}, \hat{u}]\! =\! [l, u]\! \cap\! [\xi_i, \xi_{i+1}]$.
Minimization is analogous.

\subsection{Piyavskii Optimization}
\label{subsec:pyavskii-optim}

\begin{figure*}[t]
    \centering
    \begin{subfigure}[t]{0.45\textwidth}
        \centering
        \begin{tikzpicture}
\begin{axis}[
    axis lines=middle,
    xlabel={$x$},
    ylabel={$g(x)$},
    xlabel style={anchor=west},
    ylabel style={anchor=south},
    xmin=-4.50, xmax=2.50,
    ymin=-3.5, ymax=5.5,
    samples=400,
    domain=-4.00:2.00,
    xtick=\empty,
    ytick=\empty,
    clip=false,
    yscale=0.65
]

\pgfmathsetmacro{\LB}{-4.0}
\pgfmathsetmacro{\UB}{+2.0}
\pgfmathsetmacro{\Lsw}{1.0908}     %
\pgfmathsetmacro{\slope}{0.4}      %
\pgfmathsetmacro{\Lg}{\Lsw + abs(\slope)} %

\pgfmathdeclarefunction{f}{1}{%
  \pgfmathparse{#1/(1 + exp(-#1))}%
}
\pgfmathdeclarefunction{g}{1}{%
  \pgfmathparse{f(#1) - \slope*(#1)}%
}

\pgfmathsetmacro{\fLB}{f(\LB)}
\pgfmathsetmacro{\fUB}{f(\UB)}
\pgfmathsetmacro{\gLB}{g(\LB)}
\pgfmathsetmacro{\gUB}{g(\UB)}

\addplot[
    blue!60,
    dashed,
    opacity=0.5,
    very thick
] coordinates {
    (\LB,\pgfkeysvalueof{/pgfplots/ymin})
    (\LB,\pgfkeysvalueof{/pgfplots/ymax})
};

\addplot[
    blue!60,
    dashed,
    opacity=0.5,
    very thick
] coordinates {
    (\UB,\pgfkeysvalueof{/pgfplots/ymin})
    (\UB,\pgfkeysvalueof{/pgfplots/ymax})
};

\addplot[very thick, blue!70!black, domain=\LB:\UB] {g(x)};
\addplot[only marks, mark=*, mark size=2.6pt, blue!70!black]
  coordinates {(\LB,\gLB) (\UB,\gUB)};

\node[blue!70!black, anchor=south east]
at (axis cs:\LB,\gLB)
{\scriptsize$(\ell,\, g(\ell))$};

\node[blue!70!black, anchor=south west]
at (axis cs:\UB,\gUB)
{\scriptsize$(u,\, g(u))$};

\pgfmathsetmacro{\xP}{(\gUB - \gLB + \Lg*(\LB + \UB)) / (2*\Lg)}
\pgfmathsetmacro{\yP}{\gLB + \Lg*(\xP - \LB)}

\pgfmathsetmacro{\xO}{(\gLB - \gUB + \Lg*(\LB + \UB)) / (2*\Lg)}
\pgfmathsetmacro{\yO}{\gLB - \Lg*(\xO - \LB)}

\addplot[draw=none, fill=violet, opacity=0.13] coordinates
  {(\LB,\gLB) (\xP,\yP) (\UB,\gUB) (\xO,\yO)} -- cycle;

\addplot[very thick, violet, domain=\LB:\xP]
  {\gLB + \Lg*(x - \LB)};
\addplot[very thick, violet, domain=\xP:\UB]
  {\gUB - \Lg*(x - \UB)};

\addplot[very thick, violet, domain=\LB:\xO]
  {\gLB - \Lg*(x - \LB)};
\addplot[very thick, violet, domain=\xO:\UB]
  {\gUB + \Lg*(x - \UB)};

\pgfmathsetmacro{\xU}{(\fUB - \fLB + \Lsw*(\LB + \UB)) / (2*\Lsw)}
\pgfmathsetmacro{\xL}{(\fLB - \fUB + \Lsw*(\LB + \UB)) / (2*\Lsw)}

\pgfmathsetmacro{\yU}{\fLB + \Lsw*(\xU - \LB) - \slope*\xU}
\pgfmathsetmacro{\yL}{\fLB - \Lsw*(\xL - \LB) - \slope*\xL}

\addplot[draw=none, fill=orange, opacity=0.10] coordinates
  {(\LB,\gLB) (\xU,\yU) (\UB,\gUB) (\xL,\yL)} -- cycle;

\addplot[very thick, dashed, orange, domain=\LB:\xU]
  {\fLB + \Lsw*(x - \LB) - \slope*x};
\addplot[very thick, dashed, orange, domain=\xU:\UB]
  {\fUB + \Lsw*(\UB - x) - \slope*x};

\addplot[very thick, dashed, orange, domain=\LB:\xL]
  {\fLB - \Lsw*(x - \LB) - \slope*x};
\addplot[very thick, dashed, orange, domain=\xL:\UB]
  {\fUB - \Lsw*(\UB - x) - \slope*x};
\def\EqSize{\scriptsize}

\node[orange!80!black, rotate=-32.5, anchor=north,
      xshift=-0pt, yshift=-0pt]
at (axis cs:{0.5*(\LB+\xL)},
    {\fLB - \Lsw*(0.5*(\xL-\LB)) - \slope*(0.5*(\LB+\xL))})
{\EqSize$\displaystyle f(\ell)-L_f(x-\ell)-\gamma x$};

\node[orange!80!black, rotate=-32.5, anchor=south,
      xshift=0pt, yshift=0pt]
at (axis cs:{0.5*(\xU+\UB)},
    {\fUB + \Lsw*(\UB-0.5*(\xU+\UB)) - \slope*(0.5*(\xU+\UB))})
{\EqSize$\displaystyle f(u)+L_f(u-x)-\gamma x$};

\node[violet!90!black, rotate=32.5, anchor=south,
      xshift=-0pt, yshift=0pt]
at (axis cs:{0.5*(\LB+\xP)},
    {\gLB + \Lg*(0.5*(\xP-\LB))})
{\EqSize$\displaystyle f(\ell)-\gamma\ell+(L_f+|\gamma|)(x-\ell)$};

\node[orange!80!black, rotate=15, anchor=north,
      xshift=0pt, yshift=-0pt]
at (axis cs:{0.5*(\LB+\xU)},
    {\fLB + \Lsw*(0.5*(\xU-\LB)) - \slope*(0.5*(\LB+\xU))})
{\EqSize$\displaystyle f(\ell)+L_f(x-\ell)-\gamma x$};

\node[violet!90!black, rotate=32.5, anchor=north,
      xshift=0pt, yshift=-0pt]
at (axis cs:{0.65*(\xO+\UB)},
    {\gUB + \Lg*(0.65*(\xO+\UB)-\UB)})
{\EqSize$\displaystyle f(u)-\gamma u+(L_f+|\gamma|)(x-u)$};

\node[orange!80!black, rotate=15.0, anchor=south,
      xshift=-0pt, yshift=0pt]
at (axis cs:{0.65*(\xL+\UB)},
    {\fUB - \Lsw*(\UB-0.65*(\xL+\UB)) - \slope*(0.65*(\xL+\UB))})
{\EqSize$\displaystyle f(u)-L_f(u-x)-\gamma x$};

\end{axis}
\end{tikzpicture}
        \label{fig:lipschitz-envelope}
    \end{subfigure}\hfill
    \begin{subfigure}[t]{0.45\textwidth}
        \centering
        \begin{tikzpicture}
\begin{axis}[
    axis lines=middle,
    xlabel={$x$},
    ylabel={$g(x)$},
    xlabel style={anchor=west},
    ylabel style={anchor=south},
    xmin=-8.5, xmax=2.5,
    ymin=-2.75, ymax=2.25,
    samples=100,
    domain=-8:2,
    xtick=\empty,
    ytick=\empty,
    clip=false,
    yscale=0.65
]

\pgfmathsetmacro{\LB}{-8}
\pgfmathsetmacro{\UB}{ 2}
\pgfmathsetmacro{\MB}{0.5*(\LB+\UB)} %

\pgfmathdeclarefunction{sig}{1}{%
  \pgfmathparse{1/(1 + exp(-#1))}%
}
\pgfmathdeclarefunction{swish}{1}{%
  \pgfmathparse{(#1) * sig(#1)}%
}
\pgfmathdeclarefunction{dswish}{1}{%
  \pgfmathparse{sig(#1) + (#1)*sig(#1)*(1 - sig(#1))}%
}
\pgfmathdeclarefunction{dabs}{1}{%
  \pgfmathparse{abs(dswish(#1))}%
}

\foreach \c in {-8,-3,2}{
  \addplot[dashed, black!25]
    coordinates {(\c,\pgfkeysvalueof{/pgfplots/ymin})
                 (\c,\pgfkeysvalueof{/pgfplots/ymax})};
}

\pgfmathsetmacro{\Lone}{dabs(-5.0)} %
\pgfmathsetmacro{\Ltwo}{dabs(-3.0)} %
\pgfmathsetmacro{\Lthr}{0.09986}    %
\pgfmathsetmacro{\Lfor}{0.50}       %
\pgfmathsetmacro{\Lfiv}{dabs(1.0)}  %
\pgfmathsetmacro{\Lsix}{dabs(2.0)}  %

\pgfmathsetmacro{\Lglob}{max(\Lone, max(\Ltwo, max(\Lthr, max(\Lfor, max(\Lfiv, \Lsix)))))}

\pgfmathsetmacro{\LlocLeft}{max(\Lone,\Ltwo)}
\pgfmathsetmacro{\LlocRight}{max(\Lthr, max(\Lfor, max(\Lfiv,\Lsix)))}

\newcommand{\DrawMinorant}[4]{%
  \pgfmathsetmacro{\a}{#1}
  \pgfmathsetmacro{\b}{#2}
  \pgfmathsetmacro{\L}{#3}
  \pgfmathsetmacro{\fa}{swish(\a)}
  \pgfmathsetmacro{\fb}{swish(\b)}
  \pgfmathsetmacro{\xk}{(\fa - \fb + \L*(\a + \b)) / (2*\L)}
  \addplot[#4, domain=\a:\xk] {\fa - \L*(x - \a)};
  \addplot[#4, domain=\xk:\b] {\fb - \L*(\b - x)};
}

\addplot[very thick, blue!70!black, domain=\LB:\UB] {swish(x)};

\pgfmathsetmacro{\fLB}{swish(\LB)}
\pgfmathsetmacro{\fMB}{swish(\MB)}
\pgfmathsetmacro{\fUB}{swish(\UB)}
\addplot[only marks, mark=*, mark size=2.2pt, blue!70!black]
  coordinates {(\LB,\fLB) (\MB,\fMB) (\UB,\fUB)};

\node[anchor=south, yshift=1pt] at (axis cs:\LB,\fLB) {$x_1$};
\node[anchor=south, yshift=4pt] at (axis cs:\MB,\fMB) {$x_3$};
\node[anchor=south, yshift=2pt] at (axis cs:\UB,\fUB) {$x_2$};

\DrawMinorant{\LB}{\MB}{\Lglob}{very thick, dashed, violet}
\DrawMinorant{\MB}{\UB}{\Lglob}{very thick, dashed, violet}

\DrawMinorant{\LB}{\MB}{\LlocLeft}{thick, dashed, dash phase=2pt, orange}
\DrawMinorant{\MB}{\UB}{\LlocRight}{thick, dashed, dash phase=2pt, orange}

\pgfmathsetmacro{\faG}{swish(\LB)}
\pgfmathsetmacro{\fbG}{swish(\MB)}
\pgfmathsetmacro{\fcG}{swish(\UB)}

\pgfmathsetmacro{\xkGL}{(\faG - \fbG + \Lglob*(\LB + \MB)) / (2*\Lglob)}
\pgfmathsetmacro{\ykGL}{\faG - \Lglob*(\xkGL - \LB)}

\pgfmathsetmacro{\xkGR}{(\fbG - \fcG + \Lglob*(\MB + \UB)) / (2*\Lglob)}
\pgfmathsetmacro{\ykGR}{\fbG - \Lglob*(\xkGR - \MB)}

\pgfmathsetmacro{\xNewG}{ifthenelse(\ykGL < \ykGR, \xkGL, \xkGR)}
\pgfmathsetmacro{\yNewGm}{min(\ykGL,\ykGR)}      %
\pgfmathsetmacro{\yNewG}{swish(\xNewG)}          %

\addplot[only marks, mark=*, mark size=2.4pt, red!80!black]
  coordinates {(\xNewG,\yNewG)};
\node[anchor=south, yshift=1pt, text=violet]
  at (axis cs:\xNewG,\yNewG) {$x_4$};

\addplot[only marks, mark=*, mark size=2.2pt, violet]
  coordinates {(\xNewG,\yNewGm)};

\pgfmathsetmacro{\isLeftG}{ifthenelse(\ykGL < \ykGR, 1, 0)}
\pgfmathsetmacro{\aG}{ifthenelse(\isLeftG==1, \LB, \MB)}
\pgfmathsetmacro{\bG}{ifthenelse(\isLeftG==1, \MB, \UB)}
\DrawMinorant{\aG}{\xNewG}{\Lglob}{thick, dotted, violet}
\DrawMinorant{\xNewG}{\bG}{\Lglob}{thick, dotted, violet}

\pgfmathsetmacro{\xkLL}{(\faG - \fbG + \LlocLeft*(\LB + \MB)) / (2*\LlocLeft)}
\pgfmathsetmacro{\ykLL}{\faG - \LlocLeft*(\xkLL - \LB)}

\pgfmathsetmacro{\xkLR}{(\fbG - \fcG + \LlocRight*(\MB + \UB)) / (2*\LlocRight)}
\pgfmathsetmacro{\ykLR}{\fbG - \LlocRight*(\xkLR - \MB)}

\pgfmathsetmacro{\xNewL}{ifthenelse(\ykLL < \ykLR, \xkLL, \xkLR)}
\pgfmathsetmacro{\yNewLm}{min(\ykLL,\ykLR)}      %
\pgfmathsetmacro{\yNewL}{swish(\xNewL)}          %

\addplot[only marks, mark=*, mark size=2.4pt, red!80!black]
  coordinates {(\xNewL,\yNewL)};
\node[anchor=south, yshift=8pt, text=orange]
  at (axis cs:\xNewL,\yNewL) {$x_4$};

\addplot[only marks, mark=*, mark size=2.2pt, orange]
  coordinates {(\xNewL,\yNewLm)};

\pgfmathsetmacro{\isLeftL}{ifthenelse(\ykLL < \ykLR, 1, 0)}
\pgfmathsetmacro{\aL}{ifthenelse(\isLeftL==1, \LB, \MB)}
\pgfmathsetmacro{\bL}{ifthenelse(\isLeftL==1, \MB, \UB)}
\pgfmathsetmacro{\LLuse}{ifthenelse(\isLeftL==1, \LlocLeft, \LlocRight)}
\DrawMinorant{\aL}{\xNewL}{\LLuse}{thick, dotted, orange}
\DrawMinorant{\xNewL}{\bL}{\LLuse}{thick, dotted, orange}

\end{axis}
\end{tikzpicture}
        \label{fig:lipschitz-minorant}
    \end{subfigure}
    \caption{Lipschitz envelope and minorant of an objective function~$g(x) = f(x) - \gamma x$. Left: Construction of a straightforward Lipschitz envelope (purple/solid) and of the improved envelope (orange/dashed). Right: Lipschitz minorant constructed using a global Lipschitz constant (purple/dashed) and local Lipschitz constants (orange/dashed). The next trial points $x_4$ obtained from both minorants are shown. Dotted lines (purple and orange) show the minorant segments of the next iteration.}
    \label{fig:Lipschitz-example}
\end{figure*}
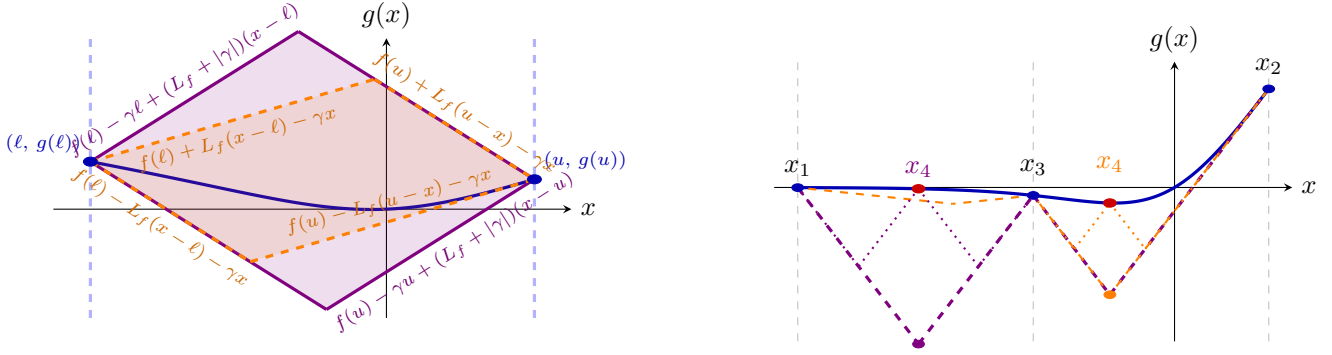

For computing $\hat{y}_i^*$ of Eq.~\eqref{eq:overapprox-max}, we resort to Lipschitz optimization methods.
These methods assume that the objective function $g$ is Lipschitz-continuous over the domain $[l, u]$, i.e.,
\begin{equation}
    \abs{g(x) - g(y)} \leq L_g \abs{x - y} \quad \forall x,y \in [l, u] \; ,
\end{equation}
for some known constant $L_g \in \R_{\geq 0}$.

The optimization method implemented in our framework is a tensor-friendly and efficient implementation of the Piyavskii--Schubert method (sawtooth method)~\cite{PIYAVSKII197257}.
Using this method, by underapproximating the minimum (overapproximating the maximum), we compute a sound lower (or upper) bound $\hat{y}_i^*$ for $g(x) = f(x) - \gamma_i x$ over $[\xi_i, \xi_{i+1}]$.
In the remainder of this section we focus on the minimization case --- maximization can be achieved by minimizing $g(x) = \gamma_i x - f(x)$. %

The intuition behind the Piyavskii method is that, given the Lipschitz constant $L_g$ (or a sound overapproximation thereof), one can construct a straightforward lower Lipschitz envelope between two sample points $x_l^i < x_r^i$ via linear lower bounds 
\begin{align}
    g(x) \!\geq\! -L_g (x \!-\! x_l^i) \!+\! g(x_l^i) && g(x) \!\geq\! L_g (x \!-\! x_r^i) \!+\! g(x_r^i) \label{eq:lipschitz-envelope}
\end{align}
(and similarly for the upper envelope) that forms a sound underpproximation of the objective function (see left sub-figure of Fig.~\ref{fig:Lipschitz-example}) without requiring an in-depth analytical treatment of $g$.
Based on this envelope, a lower piecewise linear approximation, referred to as a \emph{minorant} (see right sub-figure of Fig.~\ref{fig:Lipschitz-example}), can be constructed over a finite decomposition of the domain $[l, u]$.
The Piyavskii algorithm iteratively refines this minorant until its minimum is $\varepsilon$-close to the function $g$.
We provide pseudocode in Alg.~\ref{alg:pyavskii}.

The \textsc{initialize} method (\cref{line:init}) builds an initial minorant with $n_{ini}-1$ segments by sampling $n_{ini}$ uniformly spaced points in the interval $[l,u]$.
Each minorant segment $i$ is characterized by its left and right endpoints $x_l^i$ and $x_r^i$, their corresponding trial values $z_l^i = g(x_l^i)$ and $z_r^i = g(x_r^i)$, and the minorant intersection point $(x_m^i, z_m^i)$, computed as the intersection of the linear functions in Eq.~\eqref{eq:lipschitz-envelope}
\begin{equation}
    \begin{aligned}
    x_m^i = \frac{x_l^i + x_r^i}{2} - \frac{z_r^i - z_l^i}{2L_g},
    &&
    z_m^i = L_g (x_m^i - x_r^i) + z_r^i \; .
    \end{aligned}
    \label{eq:minorant-intersection}
\end{equation}

The \textsc{improve} step (\cref{line:improve}) refines the current approximation by splitting segment~$i$ into two subsegments via inserting a new trial point at the minimizer of the current minorant, denoted by~$x_m^i$.
The objective function is evaluated at this point, yielding $z_m^i = g(x_m^i)$, and the original segment $[x_l^i, x_r^i]$ is replaced by the two segments $[x_l^i, x_m^i]$ and $[x_m^i, x_r^i]$.
For each of the newly created segments, the corresponding minorant intersection point $(x_m^k, z_m^k)$ is computed according to Eq.~ \eqref{eq:minorant-intersection}.
The algorithm iteratively applies this refinement by selecting, at each step, the segment~$i$ whose current minorant minimum $z_m^i$ is smallest.
The algorithm terminates either when a predefined maximum number of iterations is reached or when the difference between the minimum value of the minorant $z_m^i$ and the objective function $g$ is at most~$\varepsilon$.
Upon termination, the algorithm returns a sound lower bound obtained as the minimum of the minorant, that is $\mathrm{min}\left\{z_m^i\right\}$ (\cref{line:min5-1,line:min5-2}) for the normal Piyavskii optimization, but $\mathrm{min}\left\{g(l), z_m^i, g(u)\right\}$ if we consider our improved envelope described below.

\begin{algorithm}[t]
    \begin{algorithmic}[1]
        \Require Function $g(x)$, domain $[l, u]$, Lipschitz constant $L_g$ %
        \Ensure Sound underestimation $\hat{y}^*_i$
            \State $\vec{x}_l, \vec{z}_l, \vec{x}_r, \vec{z}_r, \vec{x}_m, \vec{z}_m \leftarrow \textsc{initialize}(g, L_g, l, u, n_{ini})$  \alglabel{line:init}
            \For{$k \gets 1,\dots,max\_iter$}
                \State $i \leftarrow \argmin_i{z_m^i}$  \Comment{index of minimal segment}
                \If{$\abs{z_m^i - g(x_m^i)} \leq \varepsilon $}
                    \State \Return $\mathrm{min}\left\{g(l), z_m^i, g(u)\right\}$ \alglabel{line:min5-1}
                \EndIf
                \State $\textsc{improve}(i, g, L_g, \vec{x}_l, \vec{z}_l, \vec{x}_r, \vec{z}_r, \vec{x}_m, \vec{z}_m)$ \alglabel{line:improve}
            \EndFor

            \State $i \leftarrow \argmin_i{z_m^i}$
            \State \Return $\mathrm{min}\left\{g(l), z_m^i, g(u)\right\}$ \alglabel{line:min5-2}
    \end{algorithmic}
    \caption{Pseudocode of the Piyavskii method.}
    \label{alg:pyavskii}
\end{algorithm}

\paragraph{Improved Piyavskii} We add two improvements to the Piyavskii method that increase its efficiency.
First, we allow the practitioner to provide a piecewise constant overestimation of the Lipschitz constant of the activation function $f$
\begin{align}
    L_f(x) = \begin{cases}
        L_f^i &, x \in [a_i, a_{i+1}) \\
        \bot  &, \text{otherwise} \;,
    \end{cases}
\end{align}
for some interval bounds $a_i$, rather than a single global con\-stant over the entire domain.
That is, for each segment $[x_l^i, x_r^i]$, a tighter local Lipschitz constant $\max_{x \in [x_l^i, x_r^i]} L_f(x)$ can be used, yielding a tighter minorant (right sub-figure of Fig.~\ref{fig:Lipschitz-example}).

Secondly, we exploit that we apply Lipschitz optimization to the specific case of optimizing the difference of a general function and a linear function: $g(x)\! =\! f(x) - \gamma_i x$.
Traditional Lipschitz optimization would infer the global Lipschitz constant for $g$ as $L_g = L_f + |\gamma_i|$, where $L_f$ is a Lipschitz constant of $f$ (see Proposition 2.3.3 of~\cite{cobzacs2019lipschitz}).
However, in our case, we can derive tighter lower bounds for $x \in [x_l, x_r]$ directly from the Lipschitz envelope of~$f$:
\begin{align*}
f(x) \!\ge\! -L_f (x - x_l) + f(x_l)
\quad
f(x) \!\ge\!  L_f (x - x_r) + f(x_r) \; ,
\end{align*}
and subtracting $\gamma_i x$ from both sides we obtain %
the following valid lower bounds for $g(x)$:
\begin{align}
g(x) = f(x) - \gamma_i x & \ge -L_f (x - x_l) + f(x_l) - \gamma_i x \label{eq:improved-bounds-1}\\
g(x) = f(x) - \gamma_i x & \ge L_f (x - x_r) + f(x_r) - \gamma_i x \;. \label{eq:improved-bounds-2} 
\end{align}

These linear bounds define an envelope (see left sub-figure of Fig.~\ref{fig:Lipschitz-example}) that provably improves tightness:
\begin{proposition}
Let $f:[x_l,x_r]\to\mathbb{R}$ be Lipschitz continuous with constant $L_f$ and $g(x)=f(x)-\gamma x$.
Let the \emph{straightforward} Lipschitz envelope use $L_g=L_f+|\gamma|$, and let the
\emph{improved} envelope be obtained by applying the Lipschitz bounds to $f$
and subtracting $\gamma x$.

Then, for all $x\in[x_l,x_r]$, the improved lower envelope
\begin{align}
    g(x)\ge f(x_l)-L_f(x-x_l) - \gamma x&=:\ell^{\mathrm{imp}}_l(x) \\
    g(x)\ge f(x_r)+L_f(x-x_r) - \gamma x&=:\ell^{\mathrm{imp}}_r(x)
\end{align}

is pointwise at least as tight as
the straightforward one:
\begin{align}
    g(x)\ge g(x_l)-L_g(x-x_l)&=:\ell^{\mathrm{std}}_l(x) \\
    g(x)\ge g(x_r)+L_g(x-x_r)&=:\ell^{\mathrm{std}}_r(x) \;,
\end{align}
i.e., $\ell^{\mathrm{imp}}_l(x) \geq \ell^{\mathrm{std}}_l(x)$ and $\ell^{\mathrm{imp}}_r(x) \geq \ell^{\mathrm{std}}_r(x), \forall x\in[x_l,x_r]$.

Moreover, if $\gamma>0$, the bottom-right
segment~$\ell^{\mathrm{imp}}_r(x)$ is strictly tighter while the bottom-left coincide; if $\gamma<0$,
the bottom-left segment~$\ell^{\mathrm{imp}}_l(x)$ is strictly tighter while the bottom-right coincide.
\label{prop:imp-env}
\end{proposition}

A similar statement holds for the upper improved Lipschitz envelope.
In this case, the top-left segment is strictly tighter for $\gamma > 0$, while for $\gamma < 0$, the top-right segment is strictly tighter than the standard Lipschitz envelope.

Note that, in the standard Piyavskii method (Eq.~\eqref{eq:lipschitz-envelope}) the bounds have different slope signs, which guarantees that the minimum of the minorant is at the intersection point $x_m$ of the bounds.
However, in Eq.~\eqref{eq:improved-bounds-1} and \eqref{eq:improved-bounds-2} we might have matching slope signs, so the true minimum of the minorant might be also at the endpoints $g(x_l)$ or $g(x_r)$.

\section{Parameter Optimization}
\label{sec:slope-opt}
Following existing work \cite{Xu21}, we divide the process of finding parameters that lead to a good solution of the optimization problem $\Opt$ into two phases:
First, we heuristically set initial slopes for the lower and upper linear relaxations.
Then we use these slopes as starting point for gradient-based optimization of $\Opt$.

\subsection{Parameter Initialization}
\label{ssec:param-init}

When optimizing $\Opt$, it is beneficial to start with initial parameters that already yield close bounds.
Many NN verifiers (\cite{Paulsen2022linsyn,Paulsen2022example,Biktairov2023,Singh19deeppoly,Zhang18}) choose parameters that minimize the area between the lower and upper linear relaxation over the pre-activation bounds $[l, u]$ for each neuron, by solving
\begin{align}
    \argmin_{\lo{\alpha}, \up{\alpha},\lo{\beta}, \up{\beta}} \int_l^u \up{\alpha}x + \up{\beta} - (\lo{\alpha} x + \lo{\beta}) \;dx \;. \label{eq:min-area}
\end{align}

While other verifiers rely on manually crafted closed-form solutions for specific activation functions (\cite{Singh19deeppoly,Zhang18}), sampling and linear programming (\cite{Paulsen2022linsyn,Biktairov2023}), or precomputation \cite{Paulsen2022example}, we use gradient descent to find slopes $\lo{\alpha}, \up{\alpha}$ that result in small areas. 
Note that via Eq.~\eqref{eq:problem-def}, $\lo{\beta}$ and $\up{\beta}$ are functions of the slopes in our case, reducing Eq.~\eqref{eq:min-area} to a bivariate optimization problem.
Examples for initial relaxations found by SLiR are shown in Fig.~\ref{fig:area-min}.
\begin{figure*}[t]
    \centering
    \begin{subfigure}[t]{0.3\textwidth}
        \centering
        \includegraphics[width=0.9\linewidth]{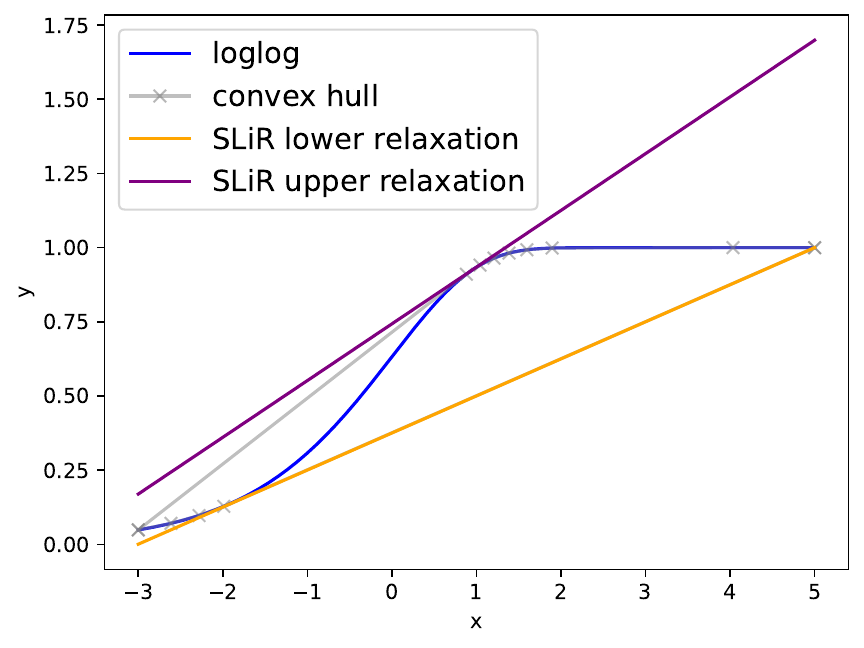}
        \caption{SLiR relaxation after area minimization.}
        \label{fig:area-min}
    \end{subfigure}\hfill
    \begin{subfigure}[t]{0.3\textwidth}
        \centering
        \includegraphics[width=0.9\linewidth]{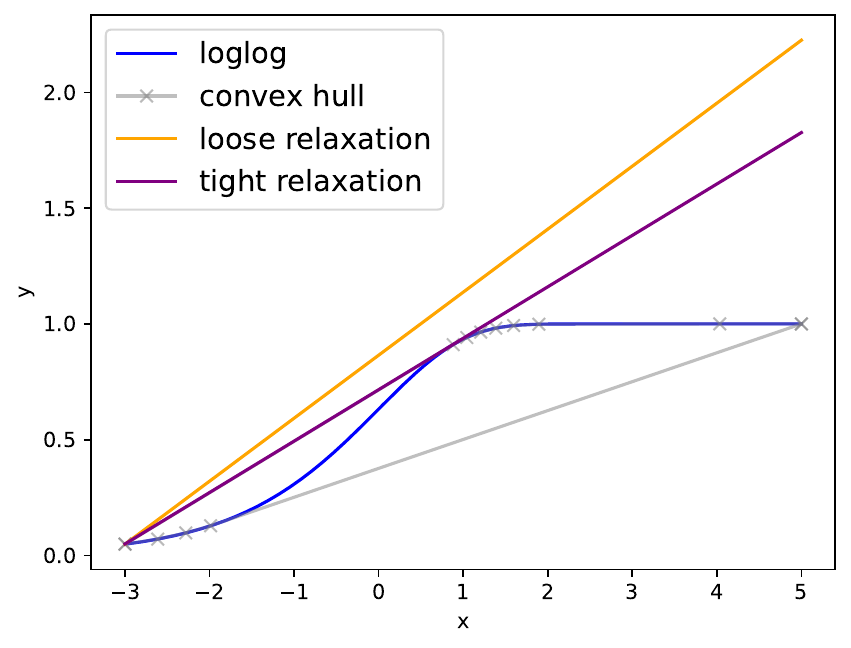}
        \caption{Tight relaxation with $\alpha = a_{max}$ and loose relaxation with $\alpha = a_{max} + 0.05$.}
        \label{fig:conv-hull-slopes}
    \end{subfigure}\hfill
    \begin{subfigure}[t]{0.3\textwidth}
        \centering
        \includegraphics[width=0.9\linewidth]{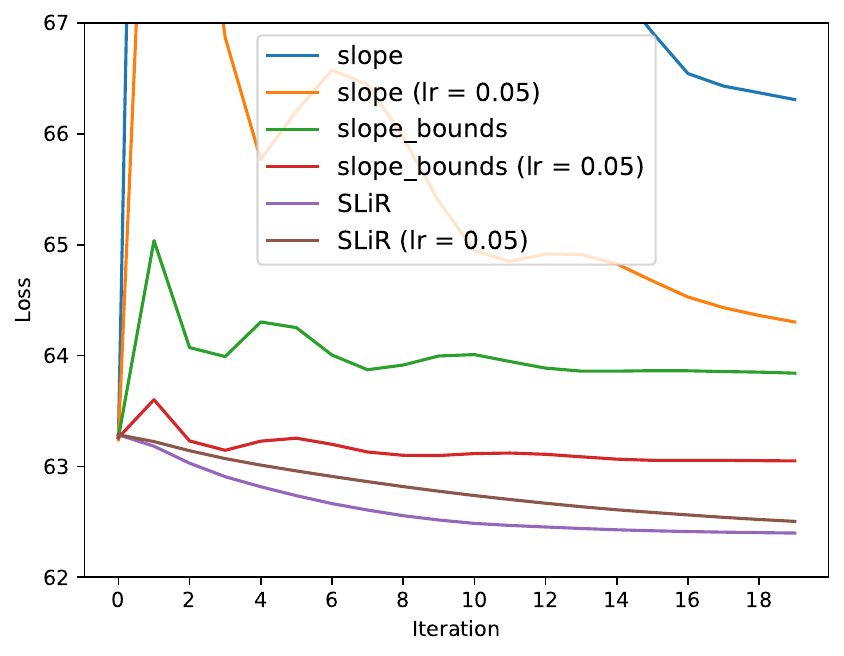}
        \caption{Loss (sum of concrete upper bounds) for a CIFAR NN with $\mathrm{AtanSq}$ activation.}
        \label{fig:conv-hull-loss}
    \end{subfigure}
    \caption{Effect of using the convex hull's lower and upper envelopes to bound the admissible slopes of the linear relaxations.}
    \label{fig:slope-bounds}
\end{figure*}

\subsection{Parameter Optimization}
\label{ssec:param-opt}

After initial slopes have been found for all neurons, we use them as starting point for solving $\Opt$.
Thereby, the choice of the parametrization functions $\lo{\alpha}(\theta), \up{\alpha}(\theta), \lo{\beta}(\theta)$ and $\up{\beta}(\theta)$ is highly relevant.
However, a naive choice of
\begin{align}
    \lo{\alpha}(\theta) = \theta \qquad \lo{\beta}(\theta) = \min_{x \in [l, u]} f(x) - \lo{\alpha}(\theta) \cdot x \label{eq:param-naive-lower}\\
    \up{\alpha}(\theta) = \theta \qquad \up{\beta}(\theta) = \max_{x \in [l, u]} f(x) - \up{\alpha}(\theta) \cdot x \label{eq:param-naive-upper}
\end{align}
for parameter $\theta \in \R$ leads to unstable optimization performance.
As shown in Fig.~\ref{fig:conv-hull-loss} (topmost two lines for learning rates $0.1$ and $0.05$), the loss curve spikes sharply after the first gradient descent iteration.
We identify two reasons for this behaviour: High sensitivity to the learning rate and suboptimal relaxations achievable via these parameters.
While SLiR generates valid linear relaxations for any slope $\lo{\alpha}, \up{\alpha} \in \R$, many of these relaxations are not tight: 
For some slopes $a$ the corresponding upper relaxations $U(x) = a \cdot x + b$ are dominated by linear upper relaxations $\hat{U}(x)$ that are strictly better over the interior of the whole domain $x \in [l, u]$ (see Fig.~\ref{fig:conv-hull-slopes}).
To address this issue, we restrict the admissible slopes to the range defined by the upper (or lower) convex hull of $f(x)$ over $[l, u]$.
This approach is justified by the following result:

\begin{lemma}
    \label{lemma:upper-max}
    Let $l < u$, $f: \R \to \R$, $a_{max}$ the maximum slope of its upper convex hull $H(x)$ over $x \in [l, u]$, $\epsilon > 0$ and $U(x) = (a_{max} + \epsilon)x + b \geq f(x)$ be a linear upper relaxation over $[l, u]$.
    Then there is another linear overapproximation $\hat{U}(x)$ with slope $a_{max}$, s.t. $f(x) \leq \hat{U}(x) < U(x)$ over $x \in (l,u]$.
\end{lemma}

A similar statement holds for the minimum slope $a_{min}$ of the upper convex hull (see Lemma~\ref{lemma:upper-min} in Appendix~\ref{sec:proofs}).
Hence, we do not compromise quality of the achievable linear relaxations by clamping the admissible slope values for the upper relaxations to $\up{\alpha} \in [a_{min}, a_{max}]$ (and similarly for lower relaxations and the lower convex hull).

We obtain $a_{min}$ and $a_{max}$, by computing the convex hull of sampled points $\left\{(x_i, f(x_i))~\mid~x_i~\in~[l, u]\right\}$ (if we computed a PWL approximation $\hat{f}_{\mathit{pwl}}$ for $f$, we use the function values at the boundary points $\left\{(l, f(l)), (u, f(u))\right\} \cup \left\{(\xi_i, f(\xi_i)) \mid \xi_i \in [l, u]\right\}$) using the \emph{monotone chain algorithm}~\cite{Andrew1979}.
This algorithm directly returns the lower and upper convex envelope and runs in $\mathcal{O}(n)$ time when the input points are sorted first by their $x$-values and then by their $y$-values.
However, our vectorized implementation may not take full advantage of the linear runtime.
Since the local input bounds $[l, u]$ can get tighter with each gradient-descent step for $\Opt$, we have to recompute the convex hull for each neuron in each iteration.

As illustrated in Fig.~\ref{fig:conv-hull-loss} (center two lines), restricting the range of the slope values drastically reduces the deterioration of the loss.
However, the optimization process is still highly sensitive to the learning rate.

In contrast to parameterized lower relaxations for the $\relu$ function, where the interval $[0,1]$ always produces tight lower relaxations for any unstable $\relu$, %
admissible ranges for other activation functions can be narrower (for the $\mathrm{loglog}$ activation function shown in Fig.~\ref{fig:conv-hull-slopes} the admissible range is just $[0.0571, 0.1248]$ for the depicted domain) and vary between each neuron (a different domain would result in different admissible ranges for the slopes).

To ensure a smooth mapping from parameter $\theta \in \R$ to slope values within $[a_{min}, a_{max}]$, we parameterize slopes via the sigmoid function $\sigma(\cdot)$:
\begin{align}
    \lo{\alpha}(\theta) &= a^l_{min} + (a^l_{max} - a^l_{min}) \cdot \sigma(\theta)\\
    \up{\alpha}(\theta) &= a^u_{min} + (a^u_{max} - a^u_{min}) \cdot \sigma(\theta) \;,
\end{align}
$a^l_{min},a^l_{max},a^u_{min}$ and $a^u_{max}$ are the minimum and maximum slope of the lower and upper convex envelopes.

Initial slopes $\alpha$ are mapped to $\theta$ by application of the inverse sigmoid function
\begin{align}
    \theta = \log\left( \frac{\alpha - a_{min}}{a_{max} - a_{min}} \right) \;, \label{eq:sigmoid-transform}
\end{align}
where we clamp the argument of the $\log$ to $[\epsilon_{\sigma}, 1 \!-\! \epsilon_{\sigma}]$~to~avoid numerical issues and $a_{min}$ and $a_{max}$ correspond to $a^l_{min}$ and $a^l_{max}$ or $a^u_{min}$ and $a^u_{max}$ respectively.
Illustrated by the lines labeled SLiR in Fig.~\ref{fig:conv-hull-loss}, this transformation allows consistent improvement in the loss curves, even for higher learning rate.

\section{Evaluation}
\label{sec:evaluation}

We evaluate (i) how our approach compares to prior work and (ii) the impact of our optimizations on performance.
We report the number of verified properties (\# certified) and total runtime (in seconds) at initialization and after $20$ gradient-based optimization steps for each lower and upper output bounds.
All experiments were run on a cluster with multiple nodes of Intel Xeon Platinum 8358 processors (64 cores, 256 GiB RAM, 2.6 GHz), with each run limited to 4 CPU cores.
Further evaluation results can be found in Appendix~\ref{sec:supplementary-eval}.

\smallskip\noindent\emph{Tools.}\
We implemented our approach in PyTorch~\cite{Paszke2019} and integrated it into the backsubstitution-based framework \autoLirpa{}~\cite{Xu2020}.
We compare against \aCROWN{}~\cite{Xu21}, the leading bound propagation method with optimizable relaxations and the basis of \abCROWN{}, winner of VNN-COMP 2021-2025~(\cite{Bak2021,Muller2022,Brix2023,Brix2024,Kaulen2025}).
While \aCROWN{} relies on custom relaxations or decomposition into elementary operations, we also compare to methods for general activation functions:
\linSyn{}~\cite{Paulsen2022linsyn} and \sol{}~\cite{Biktairov2023}, which synthesize relaxations at runtime, and the example-guided approach~\cite{Paulsen2022example}, which precomputes relaxations and falls back to interval overapproximation outside known bounds.

\smallskip\noindent\emph{Benchmarks.}\
We use benchmarks from the literature with general activation functions~(\cite{Paulsen2022linsyn,Paulsen2022example,Biktairov2023}), which are convolutional NNs trained on MNIST~\cite{LeCun1998mnist} and CIFAR10~\cite{krizhevsky2009learning} by Paulsen et al.~\cite{Paulsen2022example}.
The MNIST NNs have input dimension $28\times28$, hidden layers of size $1568, 784, 256$, and output size $10$, while the CIFAR10 NNs have input dimension $3\times32\times32$, hidden layers of size $2048, 2048, 1024, 256$, and output size $10$.
For both datasets, we consider \atansq{}~\cite{Ramachandran2018searching}, \gelu{}~\cite{Hendrycks2016gelu}, \lisht{}~\cite{Roy2022lisht}, \loglog{}~\cite{gomes2008complementary}, \mish{}~\cite{Misra2020mish}, and \swish{}~\cite{Hendrycks2016gelu} activations.

For cases where the critical points of Eq.~\eqref{eq:problem-def} can be computed exactly, we also evaluate two NNs with $6$ layers of $500$ neurons using sigmoid and $\tanh$ activations from the \eran{} benchmark suite~(\cite{Singh2018deepz,Singh19deeppoly}).
We additionally train a NN of the same architecture with $\elu{}$ activations~\cite{Clevert2015elu}.

We verify adversarial robustness~\cite{Szegedy2013} in the $L_\infty$-sense with $\epsilon=8/255$ for MNIST and $\epsilon=1/255$ for CIFAR10 on the $100$ images used in prior work~(\cite{Paulsen2022linsyn,Paulsen2022example,Biktairov2023}), skipping inputs that are already misclassified.
For the \eran{} MNIST NNs, we follow the same procedure as \cite{Singh2018deepz}, using the same $100$ images and $\epsilon=8/255$.

\renewcommand{\tablename}{Tab.}

\begin{table}[tbh]
\caption{Results for functions with computable critical points.
Total time shows runtime; \#certified counts verified instances.}
\label{tab:inverse-derivatives}
\centering
\footnotesize

\setlength{\tabcolsep}{5pt}
\renewcommand{\arraystretch}{1.15}

\begin{tikzpicture}
\node (tab) {
\begin{tabular}{llc rrr}
\toprule
\multicolumn{3}{c}{} & \multicolumn{3}{c}{\eran{} MNIST $6\times 500$} \\
\cmidrule(lr){4-6}
Method & & Metric & sigmoid & tanh & ELU \\
\midrule

\aCROWN{} & init & \#certified 
& 27 & 10 & $-^{\dag}$ \\

& & total time 
& 16.3s & 18.5s & $-^{\dag}$ \\

& opt & \#certified 
& \textbf{33} & 20 & $-^{\dag}$ \\

& & total time 
& 3\,185s & 3\,579s & $-^{\dag}$ \\

\midrule

SLiR & init & \#certified 
& 29 & 16 & 62 \\

& & total time 
& 257s & 541s & 352s \\

& opt & \#certified 
& 31 & \textbf{21} & \textbf{64} \\

& & total time 
& 5\,942s & 11\,364s & 7\,137s \\

\midrule

SLiR (exact) & init & \#certified 
& 29 & 16 & 62 \\

& & total time 
& 205s & 201s & 200s \\

& opt & \#certified 
& \textbf{33} & \textbf{21} & \textbf{64} \\

& & total time 
& 4\,966s & 4\,953s & 4\,730s \\

\bottomrule
\end{tabular}
};

\node[rotate=-90, anchor=south] at ($(tab.east)+(0.15,0)$) {\small $^{\dag}$ \aCROWN{} cannot decompose $\mathrm{ELU}$.};
\end{tikzpicture}
\end{table}

\smallskip\noindent\emph{Functions with Computable Critical Points.}\
Tab.~\ref{tab:inverse-derivatives} shows that both SLiR configurations verify more instances than \aCROWN{} after initialization, indicating tighter initial relaxations.
However, \aCROWN{} is significantly faster at initialization, as it avoids PWL approximation and gradient-based parameter search.
Optimization times for SLiR (exact) are relatively consistent across networks, whereas the PWL variant depends strongly on pre-activation bounds and approximation quality: more segments lead to more critical points to evaluate.

\begin{table*}[tbh]
\caption{Verification results for adversarial robustness.
Time (in seconds) shows the total time spent for analyzing all instances while \#cert. shows the number of instances verified as safe. UB (upper bound) indicates the number of instances for which the PGD attack was unable to find a counterexample;
it therefore is an upper bound
on the number of certifiable instances.}
\label{tab:comparison-prior-work}
\centering
\footnotesize

\setlength{\tabcolsep}{4pt}
\renewcommand{\arraystretch}{1.1}

\begin{tikzpicture}
\node (tab) {
\begin{tabular}{llc|rrrrrr|rrrrrr}
\toprule
\multicolumn{3}{c|}{} 
& \multicolumn{6}{c|}{MNIST CNN $4$-Layer (100 instances each)} 
& \multicolumn{6}{c}{CIFAR10 CNN $5$-Layer (100 instances each)} \\
\cmidrule(lr){4-9} \cmidrule(lr){10-15}
Method & & Metric
& atansq & gelu & lisht & loglog & mish & swish
& atansq & gelu & lisht & loglog & mish & swish \\
\midrule

PGD &  & UB
& 88 & 92 & 95 & 85 & 92 & 86
& 59 & 60 & 56 & 39 & 60 & 61 \\
\midrule

\sol{}$^{\S}$ &  & \#cert.
& $-^{\dag}$ & 73 & $-^{\dag}$ & 24 & $-^{\dag}$ & \textbf{76}
& $-^{\dag}$ & 20 & $-^{\dag}$ & 27 & $-^{\dag}$ & \textbf{24} \\

\linSyn{}$^{\P}$ &  & \#cert.
& $-^{\dag}$ & 72 & $-^{\dag}$ & 23 & $-^{\dag}$ & \textbf{76}
& $-^{\dag}$ & 20 & $-^{\dag}$ & 27 & $-^{\dag}$ & 23 \\
& & time
& $-^{\dag}$ & 1\,345 & $-^{\dag}$ & 1\,393 & $-^{\dag}$ & 1\,338
& $-^{\dag}$ & 1\,803 & $-^{\dag}$ & 1\,017 & $-^{\dag}$ & 1\,747 \\

E.\ Guided &  & \#cert.
& 16 & 70 & 11 & $-^{\dag}$ & 28 & 74
& 14 & 19 & 0 & $-^{\dag}$ & 19 & 23 \\

& & time
& 138 & 183 & 114 & $-^{\dag}$ & 140 & 115
& 214 & 263 & 185 & $-^{\dag}$ & 222 & 182 \\
\midrule

\aCROWN{} & init & \#cert.
& 1 & 64 & 0 & $-^{\ddag}$ & 0 & 34
& 0 & 11 & 0 & $-^{\ddag}$ & 0 & 2 \\

& & time
& 7.8 & 6.7 & 8.9 & $-^{\ddag}$ & 45.2 & 9.8
& 13.1 & 13.6 & 15.3 & $-^{\ddag}$ & 451 & 14.5 \\

& opt & \#cert.
& 3 & 70 & 9 & $-^{\ddag}$ & 6 & 47
& 3 & 14 & 0 & $-^{\ddag}$ & 1 & 5 \\

& & time
& 4\,868 & 3\,488 & 3\,625 & $-^{\ddag}$ & 16\,172 & 3\,959
& 29\,602 & 19\,465 & 17\,899 & $-^{\ddag}$ & 76\,707 & 16\,339 \\
\midrule

SLiR & init & \#cert.
& 18 & 72 & 81 & 23 & 63 & 76
& 15 & 20 & 0 & 26 & 19 & 23 \\

(ours) & & time
& 958 & 647 & 950 & 308 & 857 & 567
& 785 & 466 & 953 & 173 & 462 & 434 \\

& opt & \#cert.
& \textbf{25} & \textbf{75} & \textbf{86} & \textbf{32} & \textbf{66} & \textbf{76}
& \textbf{16} & \textbf{21} & 0 & \textbf{28} & \textbf{20} & 23 \\

& & time
& 6\,808 & 5\,794 & 6\,940 & 4\,844 & 6\,213 & 5860
& 19\,300 & 17\,880 & 21\,328 & 11\,888 & 18\,806 & 18\,372 \\
\bottomrule
\end{tabular}
};

\node[rotate=-90, anchor=south] 
    at ($(tab.east)+(1.50,0.12)$) 
    {$^{\dag}$ Activation function not supported by the tool.};

\node[rotate=-90, anchor=south] 
    at ($(tab.east)+(1.15,0.00)$) 
    {$^{\ddag}$ \aCROWN{} encountered \texttt{NaN} in the computation.};

\node[rotate=-90, anchor=south, text width=5cm, align=center] 
    at ($(tab.east)+(0.50,0)$) 
    {$^{\P}$ We had to replace the multiprocessing library, runtimes may be affected.};

\node[rotate=-90, anchor=south] 
    at ($(tab.east)+(0.15,0)$) 
    {$^{\S}$ Results taken from~\cite{Biktairov2023}.};
\end{tikzpicture}
\end{table*}
\smallskip\noindent\emph{Comparison to Prior Work.}\
\label{ssec:comparison-prior-work}
Tab. \ref{tab:comparison-prior-work} compares SLiR to \sol{}~\cite{Biktairov2023}, \linSyn{}~\cite{Paulsen2022linsyn}, the example guided (E. Guided) approach \cite{Paulsen2022example} and \aCROWN{}~\cite{Xu21}.
We used adversarial attacks (the PGD method \cite{Madry2018pgd}) to generate counterexamples for the verification problems.
The top row indicates instances with no found counterexample, serving as an upper bound on the number of provable instances.
For all approaches, we report the number of verified instances and runtime (if source code was available).
If parametric relaxations are used, we show both metrics for the initial relaxations and after optimization.

Our experiments show that optimizing the relaxations significantly boosts verification rate --- with the effects being more pronounced for the MNIST than for the CIFAR10 benchmark.
A possible reason could be that the pre-activation ranges, and thus the range of admissible slope values, of the neurons were larger in the MNIST than in the CIFAR benchmark.

Notably, SLiR verifies more instances --- both after initialization and after optimization --- than \aCROWN{} even for $\gelu$, where \aCROWN{} uses a custom relaxation.
Similarly, we can verify significantly more properties than E. Guided for the $\lisht$ and $\mish$ MNIST NNs.

While times for \linSyn{} have to be taken with a grain of salt\footnote{We had to replace the python module \texttt{multiprocessing} by \texttt{pathos.multiprocessing} to be able to run the code.}, repeating their approach for all $40$ gradient descent iterations will clearly exceed SLiR's total time after optimization.

As expected, SLiR's initialization times are slower than for \aCROWN{} and E. Guided, as these approaches use composition of custom relaxations or precomputation, whereas SLiR has to generate a PWL approximation and extend it to an overapproximation for each instance.
However, SLiR is only $1.4 - 1.9 \times$ slower than \aCROWN{} on MNIST (except for $\mish$, where it is significantly faster).
On CIFAR10, performance varies:
Notably, \aCROWN{} struggles with $\mish$ because decomposing it into a large number of elementary functions increases the NN's effective depth (see complexity of Eq.~\eqref{eq:fun-backsubs}).

\begin{table}[tbh]
\caption{Effects of different contributions (ablation study).
Top lines show \# cert.\ instances, bottom lines show times (s).}
\label{tab:ablation}
\centering
\footnotesize

\setlength{\tabcolsep}{4pt}
\renewcommand{\arraystretch}{1.12}

\begin{tabular}{c l c rrrrrr}
\toprule
\multicolumn{3}{c}{} & \multicolumn{6}{c}{MNIST CNN $4$-Layer} \\
\cmidrule(lr){4-9}
Method & & Metric 
& atansq & gelu & lisht & loglog & mish & swish \\
\midrule

\multirow{4}{*}{%
  \rotatebox[origin=c]{90}{%
    \parbox{1.4cm}{\centering SLiR}%
  }%
}
& init & \#cert.
& 18 & 72 & 81 & 23 & 63 & 76 \\
&      & time
& 666 & 490 & 681 & 270 & 579 & 437 \\
& opt  & \#cert.
& \textbf{25} & \textbf{75} & \textbf{86} & \textbf{32} & \textbf{66} & 76 \\
&      & time
& 5,763 & 5\,095 & 5\,939 & 4\,433 & 5\,105 & 5\,259 \\
\midrule

\multirow{4}{*}{%
  \rotatebox[origin=c]{90}{%
    \parbox{1.4cm}{\centering SLiR\\(base)}%
  }%
}
& init & \#cert.
& 18 & 72 & 80 & 23 & 62 & 76 \\
&      & time
& 829 & 641 & 777 & 615 & 676 & 718 \\
& opt  & \#cert.
& \textbf{25} & 74 & \textbf{86} & 31 & 64 & 76 \\
&      & time
& 6\,299 & 5\,439 & 6\,634 & 4\,880 & 5\,650 & 6\,023 \\
\midrule

\multirow{4}{*}{%
  \rotatebox[origin=c]{90}{%
    \parbox{1.5cm}{\centering SLiR\\(envelope)\\Section~\ref{subsec:pyavskii-optim}}%
  }%
}
& init & \#cert.
& 18 & 72 & 81 & 23 & 63 & 76 \\
\vspace*{0.5em}
&      & time
& 926 & 675 & 799 & 649 & 727 & 729 \\
& opt  & \#cert.
& \textbf{25} & 74 & \textbf{86} & \textbf{32} & 65 & \textbf{77} \\
&      & time
& 7\,064 & 5\,822 & 6\,759 & 5\,286 & 6\,084 & 6\,166 \\
\midrule

\multirow{4}{*}{%
  \rotatebox[origin=c]{90}{%
    \parbox{1.5cm}{\centering SLiR\\(local)\\Section~\ref{subsec:pyavskii-optim}}%
  }%
}
& init & \#cert.
& 18 & 72 & 80 & 23 & 62 & 76 \\
\vspace*{0.5em}
&      & time
& 1010 & 811 & 953 & 732 & 917 & 962 \\
& opt  & \#cert.
& \textbf{25} & 74 & \textbf{86} & \textbf{32} & 64 & 76 \\
&      & time
& 6\,556 & 5\,617 & 6\,593 & 5\,333 & 6\,230 & 6\,481 \\
\midrule

\multirow{4}{*}{%
  \rotatebox[origin=c]{90}{%
    \parbox{1.4cm}{\centering SLiR\\(loose bounds)\\Section~\ref{sec:slope-opt}}%
  }%
}
& init & \#cert.
& 18 & 72 & 74 & 23 & 63 & 76 \\
&      & time
& 870 & 544 & 808 & 272 & 763 & 524 \\
& opt  & \#cert.
& 22 & 74 & 79 & 25 & 64 & 76 \\
&      & time
& 4\,005 & 3\,389 & 3\,672 & 3\,543 & 3\,812 & 3\,819 \\
\midrule

\multirow{4}{*}{%
  \rotatebox[origin=c]{90}{%
    \parbox{1.4cm}{\centering SLiR\\(direct slope)\\Section~\ref{sec:slope-opt}}%
  }%
}
& init & \#cert.
& 18 & 72 & 81 & 23 & 63 & 76 \\
&      & time
& 878 & 585 & 867 & 282 & 811 & 533 \\
& opt  & \#cert.
& 22 & 74 & \textbf{86} & 24 & 64 & 76 \\
&      & time
& 6\,665 & 5\,727 & 6\,759 & 5\,294 & 6\,414 & 6\,331 \\
\bottomrule
\end{tabular}
\end{table}

\smallskip\noindent\emph{Ablations.}\
To evaluate our optimizations from Sec. \ref{subsec:pyavskii-optim}, we conducted an ablation study on the MNIST benchmark.
The top row in Tab.~\ref{tab:ablation} shows the number of verified instances and runtimes with all optimizations enabled.
Using only the traditional Piyavskii method for Lipschitz optimization (SLiR (base)), adding just the improved envelopes (SLiR (envelope)) or just the local Lipschitz constants (SLiR (local)) decreases the number of verified properties, showing the effectiveness of each optimization.
However, SLiR (envelope) verifies just one instance less than our final approach, showing good performance even if only a global Lipschitz constant is available.

Allowing wider bounds for the slope of the linear lower and upper relaxation within the range $[\min_{x \in [l, u]} f'(x), \max_{x \in [l, u]} f'(x)]$ (SLiR (loose-bounds)) of a neuron significantly decreases the number of verified instances after optimization.
The difference is especially evident for the NNs with $\lisht$ and $\loglog$ activation functions.
Runtimes are faster for SLiR (loose-bounds), because no convex hulls have to be computed.
The number of initially verified instances can also differ, as the slopes of the initial linear relaxations are not clipped to the tightest possible range.

Replacing the sigmoid-based parametrization of Eq.~\eqref{eq:sigmoid-transform}
with direct optimization of the slope values (SLiR (direct-slope)) likewise reduces the number of verified instances, indicating that both the convex-hull restriction and the parameter transformation contribute substantially to overall performance.

\section{Conclusion}
\label{sec:conclusion}

We presented SLiR, a shifting-based framework for constructing parameterized, optimizable linear relaxations for general activation functions. 
By separating slope optimization from synthesizing valid, tight relaxations through minimal vertical shifting computation, SLiR enables automatic and principled construction of sound relaxations without requiring expert knowledge and manual activation-specific derivations.

When the shifting distance can be computed in closed-form, the practitioner only needs to provide the set of critical points. 
If no closed-form solution exists, it suffices to provide one global or multiple local Lipschitz constants to enable sound overapproximation. 
While SLiR is not always as fast as some specialized relaxation techniques, our evaluation shows that it can verify a larger number of properties on the overwhelming majority of tested benchmarks.

\smallskip\noindent\emph{Acknowledgements.}
The authors acknowledge support by the state of Baden-Württemberg through bwHPC.

\bibliographystyle{IEEEtran}
\bibliography{optimal_bound}

\begin{thebibliography}{10}
\providecommand{\url}[1]{#1}
\csname url@samestyle\endcsname
\providecommand{\newblock}{\relax}
\providecommand{\bibinfo}[2]{#2}
\providecommand{\BIBentrySTDinterwordspacing}{\spaceskip=0pt\relax}
\providecommand{\BIBentryALTinterwordstretchfactor}{4}
\providecommand{\BIBentryALTinterwordspacing}{\spaceskip=\fontdimen2\font plus
\BIBentryALTinterwordstretchfactor\fontdimen3\font minus \fontdimen4\font\relax}
\providecommand{\BIBforeignlanguage}[2]{{%
\expandafter\ifx\csname l@#1\endcsname\relax
\typeout{** WARNING: IEEEtran.bst: No hyphenation pattern has been}%
\typeout{** loaded for the language `#1'. Using the pattern for}%
\typeout{** the default language instead.}%
\else
\language=\csname l@#1\endcsname
\fi
#2}}
\providecommand{\BIBdecl}{\relax}
\BIBdecl

\bibitem{Xu21}
K.~Xu, H.~Zhang, S.~Wang, Y.~Wang, S.~Jana, X.~Lin, and C.~Hsieh, ``Fast and complete: Enabling complete neural network verification with rapid and massively parallel incomplete verifiers,'' in \emph{Proc. of the 9th International Conference on Learning Representations({ICLR}'21), Virtual Event, Austria, May 3-7, 2021}.\hskip 1em plus 0.5em minus 0.4em\relax OpenReview.net, 2021.

\bibitem{Radford2018}
A.~Radford, K.~Narasimhan, T.~Salimans, and I.~Sutskever, ``Improving language understanding by generative pre-training,'' \emph{OpenAI}, 2018.

\bibitem{Devlin2019}
J.~Devlin, M.~Chang, K.~Lee, and K.~Toutanova, ``{BERT:} pre-training of deep bidirectional transformers for language understanding,'' in \emph{Proc. of the 2019 Conference of the North American Chapter of the Association for Computational Linguistics: Human Language Technologies ({NAACL-HLT}'19), Minneapolis, MN, USA, June 2-7, 2019}.\hskip 1em plus 0.5em minus 0.4em\relax Association for Computational Linguistics, 2019, pp. 4171--4186.

\bibitem{Dosovitskiy2021}
A.~Dosovitskiy, L.~Beyer, A.~Kolesnikov, D.~Weissenborn, X.~Zhai, T.~Unterthiner, M.~Dehghani, M.~Minderer, G.~Heigold, S.~Gelly, J.~Uszkoreit, and N.~Houlsby, ``An image is worth 16x16 words: Transformers for image recognition at scale,'' in \emph{Proc. of the 9th International Conference on Learning Representations {(ICLR'21)}, Virtual Event, Austria, May 3-7, 2021}.\hskip 1em plus 0.5em minus 0.4em\relax OpenReview.net, 2021.

\bibitem{Touvron2023}
H.~Touvron, T.~Lavril, G.~Izacard, X.~Martinet, M.~Lachaux, T.~Lacroix, B.~Rozi{\`{e}}re, N.~Goyal, E.~Hambro, F.~Azhar, A.~Rodriguez, A.~Joulin, E.~Grave, and G.~Lample, ``Llama: Open and efficient foundation language models,'' \emph{CoRR}, vol. abs/2302.13971, 2023.

\bibitem{Bochkovskiy2020}
A.~Bochkovskiy, C.~Wang, and H.~M. Liao, ``{YOLOv4}: Optimal speed and accuracy of object detection,'' \emph{CoRR}, vol. abs/2004.10934, 2020.

\bibitem{Wang2023Silu}
C.~Wang, A.~Bochkovskiy, and H.~M. Liao, ``Yolov7: Trainable bag-of-freebies sets new state-of-the-art for real-time object detectors,'' in \emph{Proc. of the {IEEE/CVF} Conference on Computer Vision and Pattern Recognition ({CVPR'23}), Vancouver, BC, Canada, June 17-24, 2023}.\hskip 1em plus 0.5em minus 0.4em\relax {IEEE}, 2023, pp. 7464--7475.

\bibitem{wang2023learning}
\BIBentryALTinterwordspacing
H.~Wang, L.~Lu, S.~Song, and G.~Huang, ``Learning specialized activation functions for physics-informed neural networks,'' \emph{ArXiv}, vol. abs/2308.04073, 2023. [Online]. Available: \url{https://api.semanticscholar.org/CorpusID:260704245}
\BIBentrySTDinterwordspacing

\bibitem{Balunovic2019}
M.~Balunovic, M.~Baader, G.~Singh, T.~Gehr, and M.~T. Vechev, ``Certifying geometric robustness of neural networks,'' in \emph{Advances in Neural Information Processing Systems 32: Annual Conference on Neural Information Processing Systems 2019 (NeurIPS'19), December 8-14, 2019, Vancouver, BC, Canada}, 2019, pp. 15\,287--15\,297.

\bibitem{Ryou2021}
W.~Ryou, J.~Chen, M.~Balunovic, G.~Singh, A.~M. Dan, and M.~T. Vechev, ``Scalable polyhedral verification of recurrent neural networks,'' in \emph{Proc. of the 33rd International Conference on Computer Aided Verification ({CAV'21}), Virtual Event, July 20-23, 2021, Proceedings, Part {I}}, ser. Lecture Notes in Computer Science, vol. 12759.\hskip 1em plus 0.5em minus 0.4em\relax Springer, 2021, pp. 225--248.

\bibitem{Laurel2023}
\BIBentryALTinterwordspacing
J.~Laurel, S.~B. Qian, G.~Singh, and S.~Misailovic, ``Synthesizing precise static analyzers for automatic differentiation,'' \emph{Proc. of the {ACM} on Programming Languages}, vol.~7, no. {OOPSLA2}, pp. 1964--1992, 2023. [Online]. Available: \url{https://doi.org/10.1145/3622867}
\BIBentrySTDinterwordspacing

\bibitem{Paulsen2022linsyn}
B.~Paulsen and C.~Wang, ``{LinSyn}: Synthesizing tight linear bounds for arbitrary neural network activation functions,'' in \emph{Proc. of the 28th International Conference on Tools and Algorithms for the Construction and Analysis of Systems ({TACAS'22}), Held as Part of the European Joint Conferences on Theory and Practice of Software, {ETAPS} 2022, Munich, Germany, April 2-7, 2022, Proceedings, Part {I}}, ser. Lecture Notes in Computer Science, vol. 13243.\hskip 1em plus 0.5em minus 0.4em\relax Springer, 2022, pp. 357--376.

\bibitem{Paulsen2022example}
\BIBentryALTinterwordspacing
------, ``Example guided synthesis of linear approximations for neural network verification,'' in \emph{Proc. of the 34th International Conference on Computer Aided Verification ({CAV'22}), Haifa, Israel, August 7-10, 2022, Proceedings, Part {I}}, ser. Lecture Notes in Computer Science, vol. 13371.\hskip 1em plus 0.5em minus 0.4em\relax Springer, 2022, pp. 149--170. [Online]. Available: \url{https://doi.org/10.1007/978-3-031-13185-1\_8}
\BIBentrySTDinterwordspacing

\bibitem{Biktairov2023}
Y.~Biktairov and J.~Deshmukh, ``{SOL:} sampling-based optimal linear bounding of arbitrary scalar functions,'' in \emph{Advances in Neural Information Processing Systems 36: Annual Conference on Neural Information Processing Systems 2023 (NeurIPS'23), New Orleans, LA, USA, December 10 - 16, 2023}, 2023.

\bibitem{Ma2025}
\BIBentryALTinterwordspacing
Z.~Ma, Z.~Wang, and G.~Bai, ``\BIBforeignlanguage{en}{Convex {Hull} {Approximation} for {Activation} {Functions}},'' \emph{\BIBforeignlanguage{en}{Proc. of the ACM on Programming Languages}}, vol.~9, no. OOPSLA2, pp. 1007--1033, 2025. [Online]. Available: \url{https://dl.acm.org/doi/10.1145/3763086}
\BIBentrySTDinterwordspacing

\bibitem{Wang2023}
J.~Wang, A.~Gupta, and C.~Wang, ``Synthesizing {MILP} constraints for efficient and robust optimization,'' \emph{Proc. of the {ACM} Programming Languages}, vol.~7, no. {PLDI}, pp. 1896--1919, 2023.

\bibitem{Shi2025}
Z.~Shi, Q.~Jin, Z.~Kolter, S.~Jana, C.-J. Hsieh, and H.~Zhang, ``Neural network verification with branch-and-bound for general nonlinearities,'' in \emph{Tools and Algorithms for the Construction and Analysis of Systems}.\hskip 1em plus 0.5em minus 0.4em\relax Cham: Springer Nature Switzerland, 2025, pp. 315--335.

\bibitem{DBLP:journals/corr/abs-2002-12920}
K.~Xu, Z.~Shi, H.~Zhang, M.~Huang, K.~Chang, B.~Kailkhura, X.~Lin, and C.~Hsieh, ``Automatic perturbation analysis on general computational graphs,'' \emph{CoRR}, vol. abs/2002.12920, 2020.

\bibitem{Singh19deeppoly}
G.~Singh, T.~Gehr, M.~P{\"{u}}schel, and M.~T. Vechev, ``An abstract domain for certifying neural networks,'' \emph{Proc. of the {ACM} Programming Languages}, vol.~3, no. {POPL}, pp. 41:1--41:30, 2019.

\bibitem{Xu2020}
K.~Xu, Z.~Shi, H.~Zhang, Y.~Wang, K.~Chang, M.~Huang, B.~Kailkhura, X.~Lin, and C.~Hsieh, ``Automatic perturbation analysis for scalable certified robustness and beyond,'' in \emph{Advances in Neural Information Processing Systems 33: Annual Conference on Neural Information Processing Systems (NeurIPS'20), December 6-12, 2020, virtual}, 2020.

\bibitem{Zelazny2022}
\BIBentryALTinterwordspacing
T.~Zelazny, H.~Wu, C.~W. Barrett, and G.~Katz, ``On optimizing back-substitution methods for neural network verification,'' in \emph{22nd Formal Methods in Computer-Aided Design, {FMCAD} 2022, Trento, Italy, October 17-21, 2022}, A.~Griggio and N.~Rungta, Eds.\hskip 1em plus 0.5em minus 0.4em\relax {IEEE}, 2022, pp. 17--26. [Online]. Available: \url{https://doi.org/10.34727/2022/isbn.978-3-85448-053-2\_7}
\BIBentrySTDinterwordspacing

\bibitem{Wong2018provable}
\BIBentryALTinterwordspacing
E.~Wong and J.~Z. Kolter, ``Provable defenses against adversarial examples via the convex outer adversarial polytope,'' in \emph{Proceedings of the 35th International Conference on Machine Learning, {ICML} 2018, Stockholmsm{\"{a}}ssan, Stockholm, Sweden, July 10-15, 2018}, ser. Proceedings of Machine Learning Research, J.~G. Dy and A.~Krause, Eds.\hskip 1em plus 0.5em minus 0.4em\relax {PMLR}, 2018, pp. 5283--5292. [Online]. Available: \url{http://proceedings.mlr.press/v80/wong18a.html}
\BIBentrySTDinterwordspacing

\bibitem{Salman2019barrier}
\BIBentryALTinterwordspacing
H.~Salman, G.~Yang, H.~Zhang, C.~Hsieh, and P.~Zhang, ``A convex relaxation barrier to tight robustness verification of neural networks,'' in \emph{Advances in Neural Information Processing Systems 32: Annual Conference on Neural Information Processing Systems 2019, NeurIPS 2019, December 8-14, 2019, Vancouver, BC, Canada}, H.~M. Wallach, H.~Larochelle, A.~Beygelzimer, F.~d'Alch{\'{e}}{-}Buc, E.~B. Fox, and R.~Garnett, Eds., 2019, pp. 9832--9842. [Online]. Available: \url{https://proceedings.neurips.cc/paper/2019/hash/246a3c5544feb054f3ea718f61adfa16-Abstract.html}
\BIBentrySTDinterwordspacing

\bibitem{Lyu2020fastened}
\BIBentryALTinterwordspacing
Z.~Lyu, C.~Ko, Z.~Kong, N.~Wong, D.~Lin, and L.~Daniel, ``Fastened {CROWN:} tightened neural network robustness certificates,'' in \emph{The Thirty-Fourth {AAAI} Conference on Artificial Intelligence, {AAAI} 2020, The Thirty-Second Innovative Applications of Artificial Intelligence Conference, {IAAI} 2020, The Tenth {AAAI} Symposium on Educational Advances in Artificial Intelligence, {EAAI} 2020, New York, NY, USA, February 7-12, 2020}.\hskip 1em plus 0.5em minus 0.4em\relax {AAAI} Press, 2020, pp. 5037--5044. [Online]. Available: \url{https://doi.org/10.1609/aaai.v34i04.5944}
\BIBentrySTDinterwordspacing

\bibitem{Clevert2015elu}
D.~Clevert, T.~Unterthiner, and S.~Hochreiter, ``Fast and accurate deep network learning by exponential linear units ({ELUs}),'' in \emph{Proc. of the 4th International Conference on Learning Representations (ICLR'16), San Juan, Puerto Rico, May 2-4, 2016}, 2016.

\bibitem{PIYAVSKII197257}
S.~A. Piyavskii, ``An algorithm for finding the absolute extremum of a function,'' \emph{Journal of {USSR} Computational Mathematics and Mathematical Physics}, vol.~12, no.~4, pp. 57--67, 1972.

\bibitem{cobzacs2019lipschitz}
{\c{S}}.~Cobza{\c{s}}, R.~Miculescu, A.~Nicolae \emph{et~al.}, \emph{Lipschitz functions}.\hskip 1em plus 0.5em minus 0.4em\relax Springer, 2019.

\bibitem{Zhang18}
H.~Zhang, T.~Weng, P.~Chen, C.~Hsieh, and L.~Daniel, ``Efficient neural network robustness certification with general activation functions,'' in \emph{Advances in Neural Information Processing Systems 31: Annual Conference on Neural Information Processing Systems (NeurIPS'18), December 3-8, 2018, Montr{\'{e}}al, Canada}, 2018, pp. 4944--4953.

\bibitem{Andrew1979}
\BIBentryALTinterwordspacing
A.~M. Andrew, ``Another efficient algorithm for convex hulls in two dimensions,'' \emph{Information Processing Letters}, vol.~9, no.~5, pp. 216--219, 1979. [Online]. Available: \url{https://doi.org/10.1016/0020-0190(79)90072-3}
\BIBentrySTDinterwordspacing

\bibitem{Paszke2019}
A.~Paszke, S.~Gross, F.~Massa, A.~Lerer, J.~Bradbury, G.~Chanan, T.~Killeen, Z.~Lin, N.~Gimelshein, L.~Antiga, A.~Desmaison, A.~K{\"{o}}pf, E.~Z. Yang, Z.~DeVito, M.~Raison, A.~Tejani, S.~Chilamkurthy, B.~Steiner, L.~Fang, J.~Bai, and S.~Chintala, ``{PyTorch}: An imperative style, high-performance deep learning library,'' in \emph{Advances in Neural Information Processing Systems 32: Annual Conference on Neural Information Processing Systems 2019, NeurIPS 2019, December 8-14, 2019, Vancouver, BC, Canada}, 2019, pp. 8024--8035.

\bibitem{Bak2021}
S.~Bak, C.~Liu, and T.~T. Johnson, ``The second international verification of neural networks competition {(VNN-COMP} 2021): Summary and results,'' \emph{CoRR}, vol. abs/2109.00498, 2021.

\bibitem{Muller2022}
M.~N. M{\"{u}}ller, C.~Brix, S.~Bak, C.~Liu, and T.~T. Johnson, ``The third international verification of neural networks competition {(VNN-COMP} 2022): Summary and results,'' \emph{CoRR}, vol. abs/2212.10376, 2022.

\bibitem{Brix2023}
C.~Brix, S.~Bak, C.~Liu, and T.~T. Johnson, ``The fourth international verification of neural networks competition {(VNN-COMP} 2023): Summary and results,'' \emph{CoRR}, vol. abs/2312.16760, 2023.

\bibitem{Brix2024}
C.~Brix, S.~Bak, T.~T. Johnson, and H.~Wu, ``The fifth international verification of neural networks competition {(VNN-COMP} 2024): Summary and results,'' \emph{CoRR}, vol. abs/2412.19985, 2024.

\bibitem{Kaulen2025}
K.~Kaulen, T.~Ladner, S.~Bak, C.~Brix, H.~Duong, T.~Flinkow, T.~T. Johnson, L.~Koller, E.~Manino, T.~H. Nguyen, and H.~Wu, ``The 6th international verification of neural networks competition {(VNN-COMP} 2025): Summary and results,'' \emph{CoRR}, vol. abs/2512.19007, 2025.

\bibitem{LeCun1998mnist}
Y.~LeCun, L.~Bottou, Y.~Bengio, and P.~Haffner, ``Gradient-based learning applied to document recognition,'' \emph{Proc. of the {IEEE}}, vol.~86, no.~11, pp. 2278--2324, 1998.

\bibitem{krizhevsky2009learning}
\BIBentryALTinterwordspacing
A.~Krizhevsky and G.~Hinton, ``Learning multiple layers of features from tiny images,'' University of Toronto, Toronto, Ontario, Tech. Rep.~0, 2009. [Online]. Available: \url{https://www.cs.toronto.edu/~kriz/learning-features-2009-TR.pdf}
\BIBentrySTDinterwordspacing

\bibitem{Ramachandran2018searching}
P.~Ramachandran, B.~Zoph, and Q.~V. Le, ``Searching for activation functions,'' in \emph{Proc. of the 6th International Conference on Learning Representations ({ICLR'18}), Vancouver, BC, Canada, April 30 - May 3, 2018, Workshop Track Proceedings}.\hskip 1em plus 0.5em minus 0.4em\relax OpenReview.net, 2018.

\bibitem{Hendrycks2016gelu}
D.~Hendrycks and K.~Gimpel, ``Bridging nonlinearities and stochastic regularizers with gaussian error linear units,'' \emph{CoRR}, vol. abs/1606.08415, 2016.

\bibitem{Roy2022lisht}
S.~K. Roy, S.~Manna, S.~R. Dubey, and B.~B. Chaudhuri, ``{LiSHT}: Non-parametric linearly scaled hyperbolic tangent activation function for neural networks,'' in \emph{Proc. of the 7th International Conference on Computer Vision and Image Processing ({CVIP'22}), Nagpur, India, November 4-6, 2022, Revised Selected Papers, Part {I}}, ser. Communications in Computer and Information Science, vol. 1776.\hskip 1em plus 0.5em minus 0.4em\relax Springer, 2022, pp. 462--476.

\bibitem{gomes2008complementary}
G.~S. d.~S. Gomes and T.~B. Ludermir, ``Complementary {Log-Log} and {Probit}: Activation functions implemented in artificial neural networks,'' in \emph{Proc. of 2008 Eighth International Conference on Hybrid Intelligent Systems, Barcelona, Spain, September 10-12, 2008}.\hskip 1em plus 0.5em minus 0.4em\relax IEEE, 2008, pp. 939--942.

\bibitem{Misra2020mish}
D.~Misra, ``Mish: {A} self regularized non-monotonic activation function,'' in \emph{Proc. of the 31st British Machine Vision Conference ({BMVC}'20), Virtual Event, UK, September 7-10, 2020}.\hskip 1em plus 0.5em minus 0.4em\relax {BMVA} Press, 2020.

\bibitem{Singh2018deepz}
G.~Singh, T.~Gehr, M.~Mirman, M.~P{\"{u}}schel, and M.~T. Vechev, ``Fast and effective robustness certification,'' in \emph{Advances in Neural Information Processing Systems 31: Annual Conference on Neural Information Processing Systems (NeurIPS'18), December 3-8, 2018, Montr{\'{e}}al, Canada}, 2018, pp. 10\,825--10\,836.

\bibitem{Szegedy2013}
\BIBentryALTinterwordspacing
C.~Szegedy, W.~Zaremba, I.~Sutskever, J.~Bruna, D.~Erhan, I.~J. Goodfellow, and R.~Fergus, ``Intriguing properties of neural networks,'' in \emph{Proc. of 2nd International Conference on Learning Representations, {ICLR} 2014, Banff, AB, Canada, April 14-16, 2014, Conference Track Proceedings}, Y.~Bengio and Y.~LeCun, Eds., 2014. [Online]. Available: \url{http://arxiv.org/abs/1312.6199}
\BIBentrySTDinterwordspacing

\bibitem{Madry2018pgd}
A.~Madry, A.~Makelov, L.~Schmidt, D.~Tsipras, and A.~Vladu, ``Towards deep learning models resistant to adversarial attacks,'' in \emph{Proc. of the 6th International Conference on Learning Representations ({ICLR}'18), Vancouver, BC, Canada, April 30 - May 3, 2018}.\hskip 1em plus 0.5em minus 0.4em\relax OpenReview.net, 2018.

\bibitem{Zhou2023pwl}
Z.~Zhou, M.~Baratchi, G.~Si, H.~H. Hoos, and G.~Huang, ``Adaptive error bounded piecewise linear approximation for time-series representation,'' \emph{Journal of Engineering Applications of Artificial Intelligence}, vol. 126, p. 106892, 2023.

\bibitem{scipy2020}
P.~Virtanen, R.~Gommers, T.~E. Oliphant, M.~Haberland, T.~Reddy, D.~Cournapeau, E.~Burovski, P.~Peterson, W.~Weckesser, J.~Bright, S.~J. {van der Walt}, M.~Brett, J.~Wilson, K.~J. Millman, N.~Mayorov, A.~R.~J. Nelson, E.~Jones, R.~Kern, E.~Larson, C.~J. Carey, {\.I}.~Polat, Y.~Feng, E.~W. Moore, J.~{VanderPlas}, D.~Laxalde, J.~Perktold, R.~Cimrman, I.~Henriksen, E.~A. Quintero, C.~R. Harris, A.~M. Archibald, A.~H. Ribeiro, F.~Pedregosa, P.~{van Mulbregt}, and {SciPy 1.0 Contributors}, ``{{SciPy} 1.0: Fundamental Algorithms for Scientific Computing in Python},'' \emph{Nature Methods}, vol.~17, pp. 261--272, 2020.

\bibitem{Brent2002}
R.~P. Brent, \emph{\BIBforeignlanguage{eng}{Algorithms for Minimization without Derivatives}}, ser. Dover books on mathematics.\hskip 1em plus 0.5em minus 0.4em\relax Mineola, NY: Dover Publications, 2002.

\end{thebibliography}

\clearpage
\appendix
\subsection{Extended Description of the Backsubstitution Procedure}
\label{sec:backsubs}

In this section, we provide more details on the backsubstitution procedure $\textsc{Backsubstitute}(\NN, \X)$.

Fundamentally, backsubstitution is a bound propagation procedure.
It computes an overapproximate solution to Eq.~\eqref{eq:opt_nn} by successively computing affine upper and lower bounds
$\up{\vec{a}}^T \vec{x} + \up{d}$ and
$\lo{\vec{a}}^T \vec{x} + \lo{d}$.
These symbolic bounds are computed for both activations and preactivations
of each layer, in the order
$\vec{n}_k \rightarrow \hat{\vec{n}}_k \rightarrow 
\vec{n}_{k-1} \rightarrow \hat{\vec{n}}_{k-1} \rightarrow \dots$,
going backwards from outputs to inputs.
We use the symbols $\up{\vec{a}}_k$ and $\up{d}_k$
in affine forms expressed in terms
of preactivation values
$\vec{\hat{n}}_k$, and $\up{\vec{\hat{a}}}_k$ and $\up{\hat{d}}_k$
expressed in terms of
activation values $\vec{n}_k$; similarly for lower bounds,
replacing the overlines by underlines.

In parallel to the \emph{symbolic bounds}, i.e.
affine expressions over earlier-layer variables,
the procedure maintains 
\emph{concrete bounds}, that is scalar intervals $[l,u]$, for
the activation and preactivation values of each
neuron, as will be detailed below.

The procedure starts with an inital affine function in the NN's output variables $\vec{y}$ expressing 
the violation $v^*$ of the output property:
$\vec{a}^T \vec{y} - d$. This also gives us the first
upper bound (in fact, it is not an approximation, but exact) in the form
$\up{\vec{a}}_L^T \vec{\hat{n}}_L + \up{d}_L$, with 
$\up{\vec{a}}_L = \vec{a}$, $\vec{\hat{n}}_L = \vec{y}$,
and $\up{d}_L = -d$. Note that, for the last layer, we only need the upper bound. 
For all other layers, upper and lower bounds are needed to compute relaxations of the activation functions.

Backsubstitution then computes, for decreasing layers,
new values $\up{\vec{a}}_{k-1}$ and
$\up{d}_{k-1}$ 
as well as $\lo{\vec{a}}_{k-1}$ and
$\lo{d}_{k-1}$
for affine, overapproximating upper and lower bound
functions
$\up{\vec{a}}_{k-1}^T \vec{\hat{n}}_{k-1} + \up{d}_{k-1}$ and
$\lo{\vec{a}}_{k-1}^T \vec{\hat{n}}_{k-1} + \lo{d}_{k-1}$ 
that only depend on variables $\vec{\hat{n}}_{k-1}$ from the previous layer.
Analogous computations are interleavingly
performed to bound $\vec{n}_{k-1}$.
Repeating this, we finally arrive at affine approximations
that only contain input variables.
The computations, in detail, are as follows:

\noindent
{\bf $\hat{\vec{n}}_k \rightarrow \vec{n}_{k-1}$:} Going from 
  preactivation to activation of the previous layer,
  defining $\up{\hat{\vec{a}}}^T_{k-1}$
  and $\up{\hat{d}}_{k-1}$ by:
  \begin{align}
    \up{\vec{a}}_k^T \hat{\vec{n}}_k + \up{d}_k &= \up{\vec{a}}_k^T \left(W_k \vec{n}_{k-1} + \vec{b}_k \right) + \up{d}_k \nonumber \\
        &= \left( \up{\vec{a}}_k^T W_k \right) \vec{n}_{k-1} + \left( \up{\vec{a}}_k^T \vec{b}_k + \up{d}_k \right) \nonumber \\
                                                &= \up{\hat{\vec{a}}}^T_{k-1} \vec{n}_{k-1} + \up{\hat{d}}_{k-1} \;. \label{eq:backsubs-lin}
  \end{align}
\noindent
{\bf $\vec{n}_k \!\rightarrow\! \hat{\vec{n}}_k$:} From
  activation to preactivation, defining~$\up{\vec{a}}^T_k$~and~$\up{d}_k$:
  \begin{align}
    \up{\hat{\vec{a}}}^T_k \vec{n}_k + \up{\hat{d}}_k
    &\leq [\up{\vec{\hat{a}}}_k]_+^T \up{\vec{f}}(\hat{\vec{n}}_k) + [\up{\vec{\hat{a}}}_k]_-^T \lo{\vec{f}}(\hat{\vec{n}}_k) + \up{\hat{d}}_k
    \nonumber \\
    &= [\up{\vec{\hat{a}}}_k]_+^T \! \left(\up{\vec{\alpha}} \odot \hat{\vec{n}}_k \!+ \up{\vec{\beta}} \right) + [\up{\vec{\hat{a}}}_k]_-^T \left(\lo{\vec{\alpha}} \odot \hat{\vec{n}}_k\! + \lo{\vec{\beta}} \right) \! + \up{\hat{d}}_k \nonumber \\
    &= \up{\vec{a}}_k^T \hat{\vec{n}}_k + \up{d}_k \;, \label{eq:backsubs-act}
\end{align}
where $[x]_+ \!=\! \max(0, x)$ and $[x]_- \!=\! \min(0, x)$, applied elem\-ent-wise, denote
the positive and negative parts of $x$.
The computations for lower bounds are, \emph{mutatis mutandis}, the same.

In this step, we also need symbolic upper and lower bounds
$\lo{\vec{f}}$ and $\up{\vec{f}}$ of the vectorized
activation function $\vec{f}$ for concrete input ranges
of each neuron.
So, if a concrete bounding interval $[l,u]$ is available for a neuron's preactivation value $\hat{n}$, we can compute one linear lower and one upper relaxations $\lo{f}(x)$ and $\up{f}(x)$ for that particular neuron such 
that
\begin{equation}
  \lo{f}(x) := \lo{\alpha} x + \lo{\beta} \leq f(x) \leq \up{\alpha} x + \up{\beta} =: \up{f}(x) \label{eq:act-approx-appendix}  
\end{equation}
for all $x \in [l, u]$.
These relaxations, gathered for all
neurons on level $k$, are then 
employed in Eq.~\eqref{eq:backsubs-act}.

As the backsubstitution process
results in a relaxation that only depends on the
NN's input variables, it can, together with
the hyperrectangle-constraints of the input region $\mathcal{X}$, be used to compute
an upper bound $\up{v^*}$ of the 
maximum violation $v^*$.

\paragraph{Invariant}

That $\up{v^*}$ is indeed an upper bound is justified
by the following invariant of the backsubstitution
process:
\begin{equation}
\up{\vec{a}}_k^T \hat{\vec{n}}_k + \up{d}_k \;\geq\;
\vec{a}^T \vec{y} - d
\quad\text{for}~ 1 \leq k \leq L, \label{eq:invariant}
\end{equation}
The invariant obviously holds for $k=L$
by definition. Using Equations~\eqref{eq:backsubs-lin}
and \eqref{eq:backsubs-act} we have
$$\up{\vec{a}}_{k-1}^T \hat{\vec{n}}_{k-1} + \up{d}_{k-1} 
\;\geq\;\, \up{\vec{a}}_k^T \hat{\vec{n}}_k + \up{d}_k \enspace,$$
maintaining the invariant for decreasing values of $k$.
We thus finally obtain
$\up{\vec{a}}_1^T \hat{\vec{n}}_1 + \up{d}_1 \geq
\vec{a}^T \vec{y} - d$, and from this, by Equation~\eqref{eq:backsubs-lin} and $\vec{n}_0 = \vec{x}$, 
\begin{align}
    \up{\vec{\hat{a}}}_0^T \vec{x} + \up{\hat{d}}_0 \geq \vec{a}^T \vec{y} - d \quad \text{for all}~ \vec{x} \in \X
\end{align}
Thus, if 
$\up{\vec{\hat{a}}}_0^T \vec{x} + \up{\hat{d}}_0 \leq 0,$
then so is $\vec{a}^T \vec{y} - d$ for all $x \in \mathcal{X}$.
Moreover, a concrete bound, $\up{v^*}$, on $v^*$ can be computed for $\X = [\lo{\vec{x}}, \up{\vec{x}}]$ using 
\begin{align}
    \up{v^*} = [\up{\vec{\hat{a}}}_0]_+^T \up{\vec{x}} + [\up{\vec{\hat{a}}}_0]_-^T \lo{\vec{x}} + \up{\hat{d}}_0 \geq 
    \up{\vec{\hat{a}}}_0^T {\vec{x}} + \up{\hat{d}}_0 \;. \label{eq:concretize}
\end{align} 
If $\up{v^*} \leq 0$, then the overapproximate approach is able to prove the property --- otherwise, no conclusions can be drawn.

The preactivation bounds necessary for the above steps are computed by recursively applying the steps in Equations \eqref{eq:backsubs-lin}-\eqref{eq:backsubs-act} and the concretization step \eqref{eq:concretize} for each neuron in the earlier layers --- once to compute an upper bound on its preactivation values and once to compute a lower bound.

\subsection{Closed-form Solutions for Critical Points}
\label{sec:closed-form-crit}

In the following, we derive closed-form expressions for the candidate critical points used in our shifting procedure.
Given an activation function $f$ and a candidate slope $m$, we consider
$g(x)=f(x)-mx$ and compute the set of critical points
$\hat{\mathcal{X}}^* \subseteq [l,u]$ such that 
\begin{align}
    \max_{x \in [l, u]} g(x) = \max_{x \in \left(\hat{X}^* \cup \{l, u\}\right)} g(x) \;.
\end{align}
For functions with non-constant derivative, $\hat{\mathcal{X}}^*$ coincides with the preimage of the derivative in Proposition \ref{prop:preim}, in other cases, it can also contain the break-points of piecewise functions.
Evaluating $g(x)$ at the points in $\hat{\mathcal{X}}^*$ determines the minimal shift $b$ such that
$U(x)=mx+b$ is a sound upper (or lower) relaxation of $f(x)$ on $[l,u]$.
In some cases, we distinguish maximizers and minimizers by $\hat{\mathcal{X}}^+$ and $\hat{\mathcal{X}}^-$, respectively; but in our implementation we simply evaluate all candidates
$\hat{\mathcal{X}}^*=\hat{\mathcal{X}}^+\cup \hat{\mathcal{X}}^- \cup \{l, u\}$.

\subsubsection{Piecewise Linear Activation Functions}
\label{ssec:piece-wise-linear}

In this subsection we present a general solution to an arbitrary piecewise linear activation function, consisting of $n$ linear pieces. The function does not need to be continuous, but we expect that it's domain covers the whole set of real numbers (i.e., $f : \R \rightarrow \R$).

We consider a finite decomposition of $\R$ into $n \in \N_{>0}$ non-empty disjoint intervals $\I_1, \I_2, \hdots, \I_n$, such that $\bigcup_{i=1}^n \I_i = \R$ and $\I_i \cap \I_j = \emptyset$ for all $i,j \in \{1, 2, \ldots, n \}$ with $i \neq j$. We use $a_i$ and $b_i$ to refer to the lower resp. upper bound of the interval $\I_i$. Each of the intervals might be either open or closed in each direction, but we assume w.l.o.g. that the intervals are ordered according to the order on $\R$: $\I_1$ is unbounded from below (i.e. $\I_1$ is left-open with $a_1=-\infty$), $\I_n$ is unbounded from above (i.e. $\I_n$ is right-open with $b_n=\infty$), and $b_i=a_{i+1}$ for $i=1,\ldots,n-1$ (note that the property of decomposition implies that $\I_i$ is right-closed if and only if $\I_{i+1}$ is left-open).

A piecewise linear function is composed of $n$ linear function, each of which is applied over one of these intervals.

\begin{definition}[Piecewise linear function]
    A \emph{piecewise linear function} is a function of type $f:\R\rightarrow\R$ such that there exists a finite decomposition of $\R$ into  $n\in\N_{>0}$ non-empty intervals $\I_1, \I_2, \hdots, \I_n$ with
    \begin{equation}
        f(x) = 
        \begin{cases} 
            f_1(x) & \text{if } x \in \I_1 \\
            f_2(x) & \text{if } x \in \I_2 \\
            \vdots & \vdots \\
            f_{n}(x) & \text{if } x \in \I_n
        \end{cases}
    \end{equation}
for some linear functions $f_i:\I_i\rightarrow\R$, $f_i(x)=m_i\cdot x + n_i$ with $m_i,n_i\in\R$ for $i=1,\ldots,n$.

For any interval $[l,u]\subseteq\R$, we call 
    \begin{eqnarray}
    F=\left\{(x,f(x))^T\,\middle|\,x\in[l,u]\subseteq\R\right\}
    \label{eq:graph}
    \end{eqnarray}
the \emph{graph} of $f$ over $[l,u]$.
\end{definition}

In this case, considering the candidate line $y = m \cdot x$, the error is also piecewise linear function of the following form:

\begin{definition}[Piecewise error function]
    \begin{equation}
        e(x) = 
        \begin{cases} 
            e_1(x) & \text{if } x \in \I_1 \\
            e_2(x) & \text{if } x \in \I_2 \\
            \vdots & \vdots \\
            e_n(x) & \text{if } x \in \I_n
        \end{cases}
    \end{equation}
\end{definition}

for some linear functions $e_i:\I_i\rightarrow\R$, $e_i(x)=f_i(x) - m \cdot x$ for $i=1,\ldots,n$.

Our goal is to find the maximum (or minimum for finding a lower relaxation) of the error function over the domain~$[l, u]$:

\begin{equation}
    b = \max_{x \in [l, u]}{e(x)} = \max_{i \in \{1, 2, \hdots, n\}}{\max_{x \in \I_i \cap [l, u]}{e_i(x)}} \; .
\end{equation}

Since each piece of $e(x)$ is linear, there are no local extrema inside the domain $\I_i$ of a single piece $e_i(x)$, thus the maximum of a single piece of the error function can happen only at the endpoints of its domain, i.e. the maximum of $e_i(x)$ on the interval $\I_i$ can be either $e_i(a_i)$ or $e_i(b_i)$, where $a_i$ and $b_i$ are the endpoints of interval $\I_i$. Then, for indices $k_1$ and $k_2$, such that $l \in \I_{k_1}$ and $u \in \I_{k_2}$, $k_1 < k_2$, the maximum of the error function $e(x)$ over the domain $[l, u]$ is:

\begin{eqnarray*}
    b = \text{max} \bigl(
        & \{ e_{k_1}(l), e_{k_1}(b_{k_1}) \} & \cup \\ 
        & \{ e_k(a_k) \; | \; k_1 < k < k_2 \} & \cup \\ 
        & \{ e_k(b_k) \; | \; k_1 < k < k_2 \} & \cup  \\
        & \{ e_{k_2}(a_{k_2}), e_{k_2}(u)\} & \bigr) \; .
\end{eqnarray*}

\subsubsection{Sigmoid}

The sigmoid function is defined as
\begin{align}
    \sigma(x) = \frac{1}{1 + e^{-x}} \;.
\end{align}
It is an injection on $\R$ and a surjection to $(0, 1)$.
Therefore it has an inverse function on $(0, 1)$:
\begin{align}
    \sigma^{-1}(x) = -\log\left( \frac{1}{x} - 1 \right) \;.
\end{align}

We want to find the extrema of $g(x) = \sigma(x) - m \cdot x$ with derivative $g'(x) = \sigma(x)(1 - \sigma(x)) - m$.
The critical points satisfy
\begin{align}
    g'(x) &= \sigma(x) (1 - \sigma(x)) - m = 0\\
    \Leftrightarrow& -\sigma(x)^2 + \sigma(x) - m = 0 \quad \text{let } s = \sigma(x)\\
    \Leftrightarrow& -s^2 + s - m = 0
\end{align}
By the quadratic formula we get the candidate solutions
\begin{align}
    s_{1,2} = \frac{1}{2} \pm \sqrt{\frac{1}{4} - m} \;,
\end{align}
if $m \in [0, 1/4]$.
To ensure there is an $x$, s.t. $s_i = \sigma(x)$, we need to check if $s_i \in (0, 1)$.
Note that this condition is fulfilled iff. $m \in (0, 1/4]$.
The set of critical points is then
\begin{align}
    \hat{\mathcal{X}}^* = \{x ~|~ s_i \in (0, 1) \wedge x = \sigma^{-1}(s_i)\} \;.
\end{align}

\subsubsection{Tanh}

The proof for $\tanh$ proceeds similarly to the proof for $\sigma$.
The $\tanh$ function is an injection on $\R$ and a surjection to $(-1, 1)$. Therefore it also has an inverse function $\arctanh$ on $(-1, 1)$.

We want to find the extrema of $g(x) = \tanh(x) - m \cdot x$ with derivative $g'(x) = 1 - \tanh(x)^2 - m$. The critical points thus have to satisfy
\begin{align}
    g'(x)& = 1 - \tanh(x)^2 - m = 0\\
    &\Leftrightarrow -t^2 + 1 - m = 0 && (\text{let } t = \tanh(x))
\end{align}
with solutions
\begin{align}
    t_{1,2} = \pm \sqrt{1 - m} \;,
\end{align}
which is defined for $m \in [0, 1]$.
To ensure that there is an $x$ s.t. $t_i = \tanh(x)$, we need to restrict our solutions to $t_i \in (-1, 1)$.
Note that this condition is fulfilled iff. $m \in (0, 1]$.
The set of critical points is then
\begin{align}
    \hat{\mathcal{X}}^* = \left\{x ~\middle|~ t_i \in (-1, 1) \wedge x = \arctanh(t_i)\right\} \;.
\end{align}

\subsubsection{Abs}

Since the absolute value function
\begin{align}
    \text{abs}(x) = \abs{x} = \begin{cases}
      -x & \text{if } x < 0 \\
       x & \text{if } x \geq 0
    \end{cases}
\end{align}
is piecewise linear, we can just apply the results of Section \ref{ssec:piece-wise-linear}.

\subsubsection{Acos}

We have $g(x) = \arccos(x) - m \cdot x$ and $g'(x) = \frac{-1}{\sqrt{1 - x^2}} - m$. 
The critical points can be found by setting the derivative equal to zero:
\begin{align}
    g'(x) = & \frac{-1}{\sqrt{1 - x^2}} - m = 0 \\
            \Leftrightarrow &~ \frac{-1}{\sqrt{1 - x^2}} = m \\ 
            \Leftrightarrow &~ \sqrt{1 - x^2} = -\frac{1}{m} &&, m \neq 0\\
            \Leftrightarrow &~ 1 - x^2 = \frac{1}{m^2} &&, m < 0 \\
            \Leftrightarrow &~ x^2 = 1 - \frac{1}{m^2} &&, m < 0\\
            \Leftrightarrow &~ x = \pm \sqrt{1 - \frac{1}{m^2}} &&, m < 0\; .
\end{align}

The function $g(x)$ has no discontinuities or piecewise definitions, so the set of critical points are
\begin{align}
    \hat{\mathcal{X}}^* = \begin{cases}
        \pm \sqrt{1 - \frac{1}{m^2}} &, m < 0\\
        \emptyset &, \text{otherwise} \;.
    \end{cases}
\end{align}

\subsubsection{Cos}

We want to solve for the critical points of $g(x) = \text{cos}(x) - m \cdot x$, with derivative $g'(x) = -\text{sin}(x) - m$. 
The solution is analogous to the solution for the sine function described in Appendix \ref{ssec:sin}.

\subsubsection{Exp}

We have $g(x) = e^x - x\cdot m$.
With derivative $g'(x) = e^x - m$.
Therefore, critical points are 
\begin{align}
    g'(x) &= e^x - m = 0 \\
    \Leftrightarrow&~ e^x = m \\
    \Leftrightarrow& x = \log m, && m > 0
\end{align}

The second order derivative $g''(x) = e^x$, which is purely positive, means that the error function is convex on the whole set of $\R$. Thus, $x = \text{log}(m)$ is a global minimum, and it cannot be the $x$ value which maximizes the error function on the bounded domain $[l, u]$.

The function $g(x)$ has no discontinuities or piecewise definitions, so the set of critical points when maximizing the error function is

\begin{align}
    \hat{\mathcal{X}}^+ = \emptyset
\end{align}

Note that for minimizing the error function, the set of critical points is

\begin{align}
    \hat{\mathcal{X}}^- = \begin{cases}
        \log m &, m > 0\\
        \emptyset &, \text{otherwise}
    \end{cases}
\end{align}

\subsubsection{HardSigmoid, ReLU, LeakyReLU, Sign}

All these functions are piecewise linear, thus our solution from Section \ref{ssec:piece-wise-linear} should work here.

\begin{equation}
    \text{HardSigmoid}(x) = \begin{cases}
        0 & \text{if} x \leq -3 \\
        1 & \text{if} x \geq +3 \\
        \frac{x}{6} + \frac{1}{2} & \text{otherwise}
    \end{cases}
\end{equation}

\begin{equation}
    \text{ReLU}(x) = \begin{cases}
        0 & \text{if } x < 0 \\
        x & \text{otherwise}
    \end{cases}
\end{equation}

\begin{equation}
    \text{LeakyReLU}(x) = \begin{cases}
        \gamma x & \text{if } x < 0 \\
        x & \text{otherwise}
    \end{cases}
\end{equation}

\begin{equation}
    \text{Sign}(x) = \begin{cases}
        -1 & \text{if } x < 0 \\
        +1 & \text{if } x > 0 \\
         0 & \text{otherwise}
    \end{cases}
\end{equation}

\subsubsection{HardSwish}

The HardSwish activation is defined as

\begin{equation}
    \mathrm{HardSwish}(x)=
    \begin{cases}
        0 & \text{if } x \le -3\\
        x & \text{if } x \ge 3\\
        \frac{x(x+3)}{6} & \text{otherwise.}
    \end{cases}
\end{equation}

We consider $g(x)=\mathrm{HardSwish}(x)-mx$.
On $(-3,3)$ we have
\begin{align}
    g(x) &= \frac{x(x+3)}{6}-mx,\\
    g'(x) &= \frac{2x+3}{6}-m.
\end{align}

Thus, any interior critical point must satisfy
\begin{align}
    g'(x)=0
    \;\Leftrightarrow\;
    \frac{2x+3}{6}=m
    \;\Leftrightarrow\;
    x^*(m)=3m-\frac{3}{2}.
\end{align}

The point is admissible iff.~$x^*(m)\in(-3,3)$, equivalently $m\in(-\frac12,\frac32)$.
Moreover, $g''(x)=\frac13>0$ on $(-3,3)$, hence $x^*(m)$ is a (strict) local minimum.
Therefore, when maximizing $g$ over a bounded interval $[l,u]$, no interior critical point needs to be considered and it suffices to check only the interval endpoints and the breakpoints $\{-3,3\}$ (when they lie in $[l,u]$).
Conversely, when minimizing $g$, the candidate critical point is $x^*(m)$ whenever $m\in(-\frac12,\frac32)$ and $x^*(m)\in[l,u]$.

\subsubsection{Sin}
\label{ssec:sin}

We have $g(x) = \sin(x) - m \cdot x$ with derivative $g'(x) = \cos(x) - m$. 
Assuming that the co-domain of $\arccos$ is $[0, \pi]$ (with $[-\pi/2, \pi/2]$ being another popular convention), i.e. $\arccos: [-1, 1] \to [0, \pi]$, the critical points \emph{inside} $[0, \pi]$ can be computed as 
\begin{align}
    g'(x) = &~\cos(x) - m = 0 \\
    \Leftrightarrow &~\cos(x) = m \\
    \Leftrightarrow &~x = \arccos(m) &&, m \in [-1, 1] \;.
\end{align}
However, there can be critical points \emph{outside} of $[0, \pi]$ as well - to find those, we use symmetry and periodicity of the $\cos$ function.

For every $x = \arccos(m)$, also $\cos(-x) = m$ since $\cos(-x) =\cos(x)$.
Similarly, if $x$ is a solution, then any $x + 2\pi$ is another valid solution due to periodicity.

We can avoid enumerating all infinitely many solutions by taking the bounds $[l, u]$ of the approximation domain into account.
Since we are only interested in solutions $l \leq x \leq u$ inside the approximation domain, we can restrict our search to the set 
\begin{align}
    \{x + 2\pi k \mid k = k_{min},\dots,k_{max}\}
\end{align}
for $x = \pm \arccos(m)$, where
\begin{align}
    k_{min} &= \left\lceil\frac{l - x}{2\pi}\right\rceil\\
    k_{max} &= \left\lfloor\frac{u - x}{2\pi}\right\rfloor \;.
\end{align}

Indeed, we show that we can restrict our search to evaluation of at most two interior critical points in all cases, where $m \neq 0$.

As a first step, we can utilize curvature information.
The second order derivative is given by $g''(x) = -\text{sin}(x)$ and is always positive on the intervals $[2k\pi, (2k+1)\pi]$ and always negative on the domains $[(2k +1)\pi, (2k+2)\pi]$, for $k \in \Z$. 
Thus, if $\arccos(m) \in [2k\pi, (2k+1)\pi]$, then it is a local maximum, if $\arccos(m) \in [(2k+1)\pi, (2k+2)\pi]$, then it is a local minimum. 
Since the co-domain of the $\arccos(\cdot)$ function is $[0, \pi]$, we are always in the first case, i.e., $x = \arccos(m) \in [2k \cdot \pi, (2k+1) \cdot \pi]$ and $x + 2k\pi$ is therefore a local maximum.
Similarly, we are always in the second case for $x = -\arccos(m)$ and $x + 2k\pi$ is therefore a local minimum. 
Denoting the sets of local maxima and minima as $\mathcal{M}^+$ and $\mathcal{M}^-$ respectively, we obtain
\begin{align*}
    \mathcal{M}^+ &\! =\! \{ \arccos(m)\! +\! 2k\pi \mid k = k_{\text{min}}, \hdots, k_{\text{max}}, -1\! <\! m\! <\! 1 \} \\
    \mathcal{M}^- &\! =\! \{ -\arccos(m)\! +\! 2k\pi \mid k = k_{\text{min}}, \hdots, k_{\text{max}}, -1\! <\! m\! <\! 1 \} 
\end{align*}

Before we continue with our analysis, we state four lemmas about sequences of the local minima and maxima.

\begin{lemma}
    Let $m < 0$, then the sequence of local maxima $g(x_k)$ for maximizers
    \begin{align}
        x_k &= \arccos(m)\! +\! 2k\pi, && k \in \mathbb{Z}
    \end{align}
    is (strictly) monotonically increasing in $k$.
    \label{lemma:max-seq}
\end{lemma}

\begin{proof}
    We now show $g(x_k) < g(x_{k+1})$ for all $k \in \mathbb{Z}$.
    For any $a \in \R$, we have 
    \begin{align}
        & \sin\left( a + 2k \pi \right) - m \cdot \left( a + 2k \pi \right) \\
        \leq&  \sin\left( a + 2(k+1) \pi \right) - m \cdot \left( a + 2(k+1) \pi \right)\\
        \Leftrightarrow& - m \cdot \left( a + 2k \pi \right) \leq - m \cdot \left( a + 2(k+1) \pi \right) \label{eq:max-seq-periodicity}\\
        \Leftrightarrow&~ a + 2k \pi \leq a + 2(k+1) \pi \;, \label{eq:max-seq-neg}
    \end{align}
    where the last statement is obviously true, Equation \eqref{eq:max-seq-periodicity} holds due to periodicity of the $\sin$ function and Equation \eqref{eq:max-seq-neg} holds as $m < 0$.
    The statement then follows from choosing $a = \arccos(m)$.
\end{proof}

The next three lemmas can be proven analogously to Lemma \ref{lemma:max-seq}:
\begin{lemma}
    Let $m < 0$, then the sequence of local minima $g(x_k)$ for minimizers
    \begin{align}
        x_k &= -\arccos(m)\! +\! 2k\pi, && k \in \mathbb{Z}
    \end{align}
    is (strictly) monotonically increasing in $k$.
    \label{lemma:min-seq}
\end{lemma}

\begin{lemma}
    Let $m > 0$, then the sequence of local maxima $g(x_k)$ for maximizers
    \begin{align}
        x_k &= \arccos(m)\! +\! 2k\pi, && k \in \mathbb{Z}
    \end{align}
    is (strictly) monotonically decreasing in $k$.
    \label{lemma:max-seq2}
\end{lemma}

\begin{lemma}
    Let $m > 0$, then the sequence of local minima $g(x_k)$ for minimizers
    \begin{align}
        x_k &= -\arccos(m)\! +\! 2k\pi, && k \in \mathbb{Z}
    \end{align}
    is (strictly) monotonically decreasing in $k$.
    \label{lemma:min-seq2}
\end{lemma}

Now we analyze five cases based on the value of $m$:

\begin{enumerate}
    \item $m \leq -1$: The error function $g(x)$ is monotonically increasing.
    \item $-1 < m < 0$:  The error function itself is not monotonous, but the sequence of local maxima (or minima) strictly increases (Lemmas \ref{lemma:max-seq},\ref{lemma:min-seq}).
    \item $m = 0$: The error function is the same as the $\text{sin}(x)$ function.
    \item $0 < m < +1$: The error function itself is not monotonous, but the sequence of local maxima (or minima) strictly decreases (Lemmas \ref{lemma:max-seq2},\ref{lemma:min-seq2}).
    \item $m \geq +1$: The error function is monotonically decreasing.
\end{enumerate}

In cases (1) and (5) due to the error function being monotonic, minimum and maximum are attained at the boundary values $l, u$ and there are no interior critical points, i.e. $\hat{\mathcal{X}}^+ = \hat{\mathcal{X}}^- = \emptyset$

In case (3), we all points in $\mathcal{M}^+ = \{ \frac{\pi}{2} + k\pi \mid k = k_{\text{min}}, \hdots, k_{\text{max}} \}$ attain the maximum value and all points in $\mathcal{M}^- = \{ -\frac{\pi}{2} + k\pi \mid k = k_{\text{min}}, \hdots, k_{\text{max}} \}$ attain the minimum value.
We set $\mathcal{X} = \mathcal{M}^+ \cup \{ l, u \}$ (for the upper bound) and $\mathcal{X} = \mathcal{M}^- \cup \{ l, u \}$ (when finding the lower bound).

In cases (2) and (4), the error function is not monotonic, but the sequence of maxima and the sequence of minima are monotonic. In case (2), we set $\hat{\mathcal{X}}^+ = \{ 2k_{\text{max}}\pi + \arccos(x) \}$ and $\hat{\mathcal{X}}^- = \{ 2k_{\text{min}}\pi - \arccos(x) \}$. 
In case (4), we use $\hat{\mathcal{X}}^+ = \{ 2k_{\text{min}}\pi + \arccos(x) \}$ and $\hat{\mathcal{X}}^- = \{ 2k_{\text{max}}\pi - \arccos(x) \}$.

\subsubsection{Sqrt}

We have $g(x) = \sqrt{x} - x \cdot m$, with derivative $g'(x) = \frac{1}{2}\sqrt{\frac{1}{x}} - m$.
Since $\sqrt{x}$ needs to be defined, we also assume $x\geq 0$.
The critical points can be derived as
\begin{align}
    g'(x) &= \frac{1}{2}\sqrt{\frac{1}{x}} - m = 0\\
    &\Leftrightarrow \frac{1}{2}\sqrt{\frac{1}{x}} = m\\
    &\Leftrightarrow \sqrt{\frac{1}{x}} = 2m\\
    &\Leftrightarrow \frac{1}{x} = 4m^2 &&, m \geq 0\\
    &\Leftrightarrow x = \frac{1}{4m^2} &&, m > 0
\end{align}

The second order derivative $g''(x) = -\frac{1}{4}\sqrt{\frac{1}{x^3}}$ is strictly negative, thus the error function is concave on the set of reals $\R$. Thus, the computed extrema $x = \frac{1}{4m^2}$ if $m > 0$ is the global maximum of the function.

Since $g(x)$ has no discontinuities or piecewise definitions, so the set of critical points is just
\begin{align}
    \hat{\mathcal{X}}^+ &= \begin{cases}
        \frac{1}{4m^2} & \text{if } m > 0 \\
        \emptyset & \text{otherwise}
    \end{cases} \\
    \hat{\mathcal{X}}^- &= \emptyset \;.
\end{align}

\subsection{Piecewise Linear Approximation}
\label{sssec:pwl-approx}

While the employed widening approach ensures soundness of $f_{\mathit{pwl}}$ for any PWL approximation $\tilde{f}_{\mathit{pwl}}$, we achieve better results, when $\tilde{f}_{\mathit{pwl}}$ closely approximates $f$.
However, minimizing the approximation error may result in a large number of linear segments increasing the runtime later on.
Therefore, we try to make each linear segment as long as possible while maintaining an approximation error smaller than a given tolerance~$\epsilon_{tol}$.
We are aware that finding good PWL approximations with a low number of segments that maintain an error bound is a known problem \cite{Zhou2023pwl}, however, we did not find the exact algorithm we used in the literature, so we state it here.
Our algorithm for constructing such PWL approximations is shown in Alg.~\ref{alg:pwl}.

\begin{algorithm}[t]
    \begin{algorithmic}[1]
        \Require Function $f$, domain $[l, u]$, error bound $\epsilon_{tol}$, $N \in \R_+$
        \Ensure Set $\Xi$ of segment boundaries of $\tilde{f}_{\mathit{pwl}}$ such that $|\tilde{f}_{\mathit{pwl}}(x) - f(x)| \lesssim \epsilon_{tol}$ (approx. smaller than $\epsilon_{tol}$)
        \Function{$\textsc{approxPWL}$}{$f, l, u$}
            \State $lo \gets l$ \Comment{lower limit for root finding}
            \State $\displaystyle hi \gets \frac{u - l}{N}$ \Comment{upper limit for root finding}
            \While{$\textsc{linApproxError}(f, l, hi) < \epsilon_{tol}$}
                \State $lo \gets hi$
                \State $hi \gets \min\left(u, hi + (hi - l)\right)$ \Comment{double step size}
            \EndWhile
            \State $\xi \gets \textsc{root}\left( \textsc{linApproxError}(f, lo, hi) - \epsilon_{tol} \right)$ \\ \Comment{if no root in interior, return $hi$}
            \If{$\xi < u$}
                \Return $\{l, \xi\} \cup \textsc{approxPWL}(f, \xi, u)$
            \Else ~
                \Return $\{l, u\}$
            \EndIf
        \EndFunction
    \end{algorithmic}
    \caption{Construction of the PWL approximation.}
    \label{alg:pwl}
\end{algorithm}

Our algorithm maintains bounds $[lo, hi]$, which contain the next segment boundary.
It starts at the lower bound \(l\) and iteratively evaluates the approximation error over intervals \([l, hi]\), doubling the distance to \(hi\) each time until the error exceeds the tolerance \(\epsilon_{tol}\).
At that point, there must exist a value between the previous and current \(hi\) where the error equals \(\epsilon_{tol}\); this point \(\xi\) is selected as a segment boundary of \(\tilde{f}_{\mathit{pwl}}\), and the process continues on \([\xi, u]\).
Since the subroutines $\textsc{linApproxError}$ and $\textsc{root}$ can be called many times while constructing $\tilde{f}_{\mathit{pwl}}$, we do not compute them using verified approaches, but instead rely on SciPy's \cite{scipy2020} implementation of Brent's methods for optimization and root finding \cite{Brent2002} which yield good approximations in practice.
The linear function in each segment $[\xi_i, \xi_{i+1}]$, for the final PWL approximation as well as in $\textsc{linApproxError}$, is constructed just as the line connecting the end points $(\xi_i, f(\xi_i))$ and $(\xi_{i+1}, f(\xi_{i+1}))$.

\subsection{Proofs}
\label{sec:proofs}

\setcounter{proposition}{0}
\setcounter{lemma}{0}

In this section we provide the proofs for \Cref{prop:preim}, \Cref{prop:imp-env}, \Cref{lemma:upper-max} and \Cref{lemma:upper-min}.

\begin{proposition}
    Let $f: \R \to \R$ be differentiable over the interval $[l, u]$ and let $f'^{-1}(y) = \{x \mid f'(x) = y\}$ be the pre-image of the derivative of $f$. Then
    \begin{align}
        y &= \max_{x \in [l, u]} f(x) - m \cdot x \\
          &= \max \left(\{l, u \} \cup \left(f'^{-1}(m) \cap [l, u] \right)\right)
    \end{align}
\end{proposition}

\begin{proof}
    The maximum of a function over an interval $[l, u]$ can occur either at the boundary or in the interior of the interval.
    If the maximum occurs at $x^* \in (l, u)$, then the derivative at $x^*$ has to vanish:
    \begin{align}
        &\frac{d}{dx} f(x^*) - m \cdot x^* = 0 \\
        \Leftrightarrow& f'(x^*) - m = 0\\
        \Leftrightarrow& f'(x^*) = m
    \end{align}
    The solution set to the last equation is exactly the pre-image of $f'^{-1}(m)$.
    However, $f'^{-1}(m)$ may contain values outside of $[l, u]$. 
    Since we are only interested in solutions inside $[l, u]$, we need to intersect with the interval.
    The overall maximum is then
    \begin{align*}
        \max \left(\{l, u \} \cup \left(f'^{-1}(m) \cap [l, u] \right)\right)
    \end{align*}
\end{proof}

\begin{proposition}
Let $f:[x_l,x_r]\to\mathbb{R}$ be Lipschitz continuous with constant $L_f$ and $g(x)=f(x)-\gamma x$.
Let the \emph{straightforward} Lipschitz envelope use $L_g=L_f+|\gamma|$, and let the
\emph{improved} envelope be obtained by applying the Lipschitz bounds to $f$
and subtracting $\gamma x$.

Then, for all $x\in[x_l,x_r]$, the improved lower envelope
\begin{align}
    g(x)\ge f(x_l)-L_f(x-x_l) - \gamma x&=:\ell^{\mathrm{imp}}_l(x) \\
    g(x)\ge f(x_r)+L_f(x-x_r) - \gamma x&=:\ell^{\mathrm{imp}}_r(x)
\end{align}

is pointwise at least as tight as
the straightforward lower envelope:
\begin{align}
    g(x)\ge g(x_l)-L_g(x-x_l)&=:\ell^{\mathrm{std}}_l(x) \\
    g(x)\ge g(x_r)+L_g(x-x_r)&=:\ell^{\mathrm{std}}_r(x) \;.
\end{align}

Meaning $\ell^{\mathrm{imp}}_l(x) \geq \ell^{\mathrm{std}}_l(x)$ and $\ell^{\mathrm{imp}}_r(x) \geq \ell^{\mathrm{std}}_r(x)$.

Moreover, if $\gamma>0$, the bottom-right
segment~$\ell^{\mathrm{imp}}_r(x)$ is strictly tighter while the bottom-left coincide; if $\gamma<0$,
the bottom-left segment~$\ell^{\mathrm{imp}}_l(x)$ is strictly tighter while the bottom-right coincide.
\end{proposition}

\begin{proof}
For an interval $[x_l,x_r]$ and $g(x)=f(x)-\gamma x$ with $f$ being $L_f$-Lipschitz,
let $z_l=g(x_l)$ and $z_r=g(x_r)$.

\emph{Straightforward envelope.}
Since $g$ is Lipschitz with $L_g=L_f+|\gamma|$, we have for all $x\in[x_l,x_r]$:
\begin{align}
    g(x)\ge & z_l-L_g(x-x_l)=:\ell^{\mathrm{std}}_l(x) \\
    g(x)\ge & z_r+L_g(x-x_r)=:\ell^{\mathrm{std}}_r(x) \;.
\end{align}

\emph{Improved envelope.}
From Lipschitz continuity of $f$,
$f(x)\ge f(x_l)-L_f(x-x_l)$ and $f(x)\ge f(x_r)+L_f(x-x_r)$.
Subtracting $\gamma x$ yields two valid lower bounds for $g$:
\[
g(x)\ge f(x_l)-L_f(x-x_l) - \gamma x=:\ell^{\mathrm{imp}}_l(x) \; ,
\]
and
\[
g(x)\ge f(x_r)+L_f(x-x_r) - \gamma x=:\ell^{\mathrm{imp}}_r(x) \; .
\]

Now rewrite the improved bounds in terms of $f(x_l)$ and $f(x_r)$.

For the left bound,
\begin{align*}
\ell^{\mathrm{imp}}_l(x)
&= f(x_l) - L_f(x-x_l) - \gamma x \\
&= f(x_l) - \gamma x_l - (L_f+\gamma)(x-x_l) \\
&= f(x_l) - (L_f+\gamma)(x-x_l).
\end{align*}

For the right bound,
\begin{align*}
\ell^{\mathrm{imp}}_r(x)
&= f(x_r) + L_f(x-x_r) - \gamma x \\
&= f(x_r) - \gamma x_r + (L_f-\gamma)(x-x_r) \\
&= f(x_r) + (L_f-\gamma)(x-x_r).
\end{align*}

We now distinguish two cases.

\medskip
\noindent\textbf{Case $\gamma\ge 0$.}
Then $L_g=L_f+\gamma$.
Hence
\[
\ell^{\mathrm{imp}}_l(x)
= f(x_l)-(L_f+\gamma)(x-x_l)
= \ell^{\mathrm{std}}_l(x).
\]
For the right bound,
\begin{align}
    \ell^{\mathrm{std}}_r(x)&=f(x_r)+(L_f+\gamma)(x-x_r), \\
    \ell^{\mathrm{imp}}_r(x)&=f(x_r)+(L_f-\gamma)(x-x_r).
\end{align}

Thus
\[
\ell^{\mathrm{imp}}_r(x)-\ell^{\mathrm{std}}_r(x)
= -2\gamma(x-x_r)\ge 0
\]
for all $x\in[x_l,x_r]$.
For $x < x_r$ and $\gamma > 0$, even strict inequality holds.
Hence the right-anchored segment is strictly tighter, while the left-anchored one coincides.

\medskip
\noindent\textbf{Case $\gamma\le 0$.}
Then $L_g=L_f-\gamma$.
Hence
\[
\ell^{\mathrm{imp}}_r(x)
= f(x_r)+(L_f-\gamma)(x-x_r)
= \ell^{\mathrm{std}}_r(x).
\]
For the left bound,
\begin{align}
    \ell^{\mathrm{std}}_l(x)&=f(x_l)-(L_f-\gamma)(x-x_l), \\
    \ell^{\mathrm{imp}}_l(x)&=f(x_l)-(L_f+\gamma)(x-x_l).
\end{align}

Thus
\[
\ell^{\mathrm{imp}}_l(x)-\ell^{\mathrm{std}}_l(x)
= -2\gamma(x-x_l)\ge 0
\]
for all $x\in[x_l,x_r]$.
For $x > x_l$ and $\gamma < 0$, even strict inequality holds.
Hence the left-anchored segment is strictly tighter, while the right-anchored one coincides.

Showing that the improved upper segments are also as tight as the straightforward ones works analogously.
\end{proof}

\begin{lemma}
    Let $l < u$, $f: \R \to \R$, $a_{max}$ the maximum slope of its upper convex hull $H(x)$ over $x \in [l, u]$, $\epsilon > 0$ and $U(x) = (a_{max} + \epsilon)x + b \geq f(x)$ be a linear upper relaxation over $[l, u]$.
    Then there is another linear overapproximation $\hat{U}(x)$ with slope $a_{max}$, s.t. $f(x) \leq \hat{U}(x) < U(x)$ over $x \in (l,u]$.
\end{lemma}

\begin{proof}
    Recall that the upper convex hull of $f(x)$ over $[l, u]$ is the smallest concave function $H(x) \geq f(x)$ for $x \in [l, u]$.

    If $U(x) > H(x)$ for all $x \in (l, u]$, we can construct a dominating function $\hat{U}(x)$ by using a smaller bias term.
    
    Note that $U(x) = (a_{max} + \epsilon) x + b$ cannot touch $H(x)$ in two or more points, otherwise it would be a secant and thus tangent to the upper convex hull.
    It can also not be tangent to $H(x)$ or it would also be part of the upper convex hull.
    Therefore, $U(x) = H(x)$ is only possible for either $x = l$ or $x = u$.

    Note that
    \begin{align}
        H(x) - H(l) \leq a_{max} (x - l) \label{eq:max-slope-hull} \\
        U(x) - U(l) = (a_{max} + \epsilon) (x - l) \label{eq:slope-lin}
    \end{align}
    due to $a_{max}$ being the maximum slope of $H(x)$ and $U(x)$ being a linear function.

    We first assume $U(u) = H(u)$.
    Let $x = u$ in Equations \eqref{eq:max-slope-hull},\eqref{eq:slope-lin}, then
    Since $\epsilon > 0$ and $(u - l) > 0$, we know that $a_{max} (u - l) > (a_{max} + \epsilon) (u - l)$ and thus
    \begin{align*}
        & U(u) - U(l) \geq H(u) - H(l) \\
        \Leftrightarrow \; & U(u) - H(u) > U(l) - H(l) \;.
    \end{align*}
    But since $U(u) = H(u)$ this is equivalent to
    \begin{align}
        0 > U(l) - H(l) \Leftrightarrow H(l) > U(l)
    \end{align}
    which is a contradiction to $H(x)$ being the smallest concave upper bound to $f(x)$.

    Therefore, $U(l) = H(l)$.
    Note that, since they touch at $l$, we have $U(x) = (a_{max} + \epsilon) (x - l) + H(l)$.
    But then, we can construct 
    \begin{align*}
        \hat{U}(x) = \; & a_{max} (x - l)\! +\! H(l) \\
         < \; & a_{max} (x - l)\! +\! H(l)\! +\! \epsilon (x - l) = U(x) \;
    \end{align*}
    for all $x \in (l, u]$.
    Note that 
    \begin{align}
        a_{max} (x - l) + H(l) \geq & \; H(x) \Leftrightarrow a_{max} (x - l) \\ 
        \geq & \; H(x) - H(l) \;,
    \end{align}
    where the right hand side is true by Equation \eqref{eq:max-slope-hull}.
    Therefore, $\hat{U}(x) \geq H(x) \geq f(x)$ and $\hat{U}(x)$ is still an upper relaxation. 
\end{proof}

\begin{lemma}
    \label{lemma:upper-min}
    Let $l < u$, $f: \R \to \R$, $a_{min}$ the minimum slope of its upper convex hull $H(x)$ over $x \in [l, u]$, $\epsilon > 0$ and $U(x) = (a_{min} - \epsilon)x + b \geq f(x)$ be a linear upper relaxation over $[l, u]$.
    Then there is another linear overapproximation $\hat{U}(x)$ with slope $a_{min}$, s.t. $f(x) \leq \hat{U}(x) < U(x)$ over $x \in [l,u)$.
\end{lemma}

\begin{proof}
    By a similar argument to the proof of Lemma \ref{lemma:upper-max}, we know that $U(x) = H(x)$ is only possible for either $x = l$ or $x = u$.

    Now note that
    \begin{align}
        H(u) - H(x) \geq a_{min} (u - x) \label{eq:min-slope-hull}\\
        U(u) - U(x) = (a_{min} - \epsilon) (u - x) \label{eq:min-slope-lin}
    \end{align}

    We first assume that $U(l) = H(l)$. 
    Let $x = l$ in Equations \eqref{eq:min-slope-hull},\eqref{eq:min-slope-lin}, then as $\epsilon > 0$, we know that $a_{min}(u - x) > (a_{min} - \epsilon) (u - x)$ and thus
    \begin{align*}
        & H(u) - H(l) > U(u) - U(l) \\
        \Leftrightarrow \; & H(u) - U(u) > H(l) - U(l) = 0
    \end{align*}
    which would lead to
    \begin{align}
        H(u) - U(u) > 0 \leftrightarrow H(u) > U(u) \geq f(u)
    \end{align}
    and thus contradicts the assumption that $H(u)$ is the smallest concave upper bound to $f(x)$.

    Therefore, $U(u) = H(u)$.
    Since they touch at $u$, we can equivalently rite $U(x) = (a_{min} - \epsilon) (x - u) + H(u)$.
    We can then construct a dominating upper bound
    \begin{align*}
        \hat{U}\!(x) =\; & a_{min}(x - u) + H(u) \\
        <\; & a_{min} (x - u) + H(u) + \epsilon (x - u) = U(x) \;,
    \end{align*}
    where the last step holds because $x - u < 0$ for $x \in [l, u)$.

    Note that
    \begin{align}
        & a_{min}(x - u) + H(u) \geq H(x) \\
        \Leftrightarrow \; & a_{min}(x - u) \geq H(x) - H(u) \\
        \Leftrightarrow \; & a_{min} (u - x) \leq H(u) - H(x)
    \end{align}
    where the right hand side is true by Equation \eqref{eq:min-slope-hull}.
    Therefore, $\hat{U}(x) \geq H(x) \geq f(x)$ and $\hat{U}(x)$ is still an upper relaxation.
\end{proof}

\subsection{Implementation Details of the Piyavskii Method}

In order to implement the Piyavskii method in a tensor-friendly and computationally efficient manner, we introduce several improvements.
First, instead of maintaining an ordered list of trial points, we store a collection of minorant segment entities, each defined by its left and right endpoints and the corresponding minorant intersection point.
This representation avoids the need for sorting or ordered insertion operations on tensors.
During refinement, a selected segment entity is replaced by its left subsegment, while the corresponding right subsegment is appended to the end of the stored lists.
This design significantly simplifies the minorant improvement step.
Second, since tensor concatenation in \texttt{PyTorch} is a costly operation, we preallocate the tensors $x_l, z_l, x_r, z_r, x_m, z_m$ to their maximum required size, namely $n - 1 + \textit{max\_iter}$.
Although this strategy incurs additional memory usage, it enables substantially faster refinement steps by avoiding repeated reallocation.

\subsection{Supplementary Evaluation Material}
\label{sec:supplementary-eval}

\subsubsection{Parameter Settings}
\label{subsec:params}

This section summarizes the parameter values used in our experimental evaluation. 
Unless stated otherwise, all experiments use the same default configuration.

\paragraph{Initialization}
For the gradient-based initialization of the relaxation slopes we perform 
$100$ iterations of gradient descent.

\paragraph{Optimization of relaxation parameters}
During bound optimization we run $20$ iterations of gradient-based optimization with learning rate $0.05$.

\paragraph{Piyavskii refinement}
For constructing the overapproximation we employ the Piyavskii--Shubert algorithm with a maximum of $1000$ iterations or until the tolerance $10^{-4}$ is reached. 
The same tolerance $10^{-4}$ is used for the SLiR approximation procedure.

\paragraph{Additional parameters}
Unless explicitly overridden, the following default parameters are used in our implementation:

\begin{itemize}
    \item initialization method: \texttt{gradient-based}
    \item initialization iterations: $100$
    \item overapproximation tolerance: $\epsilon_{tol} = 10^{-4}$
    \item Piyavksii maximum iterations: $max\_iter = 1000$
    \item Piyavskii tolerance: $\varepsilon = 10^{-4}$
    \item Piyavskii initial points: $n_{ini} = 100$
    \item local Lipschitz constant estimation enabled
    \item improved Piyavskii envelope enabled
    \item sigmoid transformation epsilon: $\epsilon_\sigma = 10^{-6}$
\end{itemize}

All remaining parameters use the default values provided by our implementation.

\subsubsection{Lipschitz Optimization and Bound Tightness}
\label{ssec:bound-tightness}

While the improvements to the Piyavskii algorithm described in Section \ref{subsec:pyavskii-optim} did not lead to a large increase in verified properties for the MNIST benchmark, the effects are noticable, when having a closer look at the verifiable output bounds.

\begin{figure}[ht]
    \includegraphics[width=\linewidth]{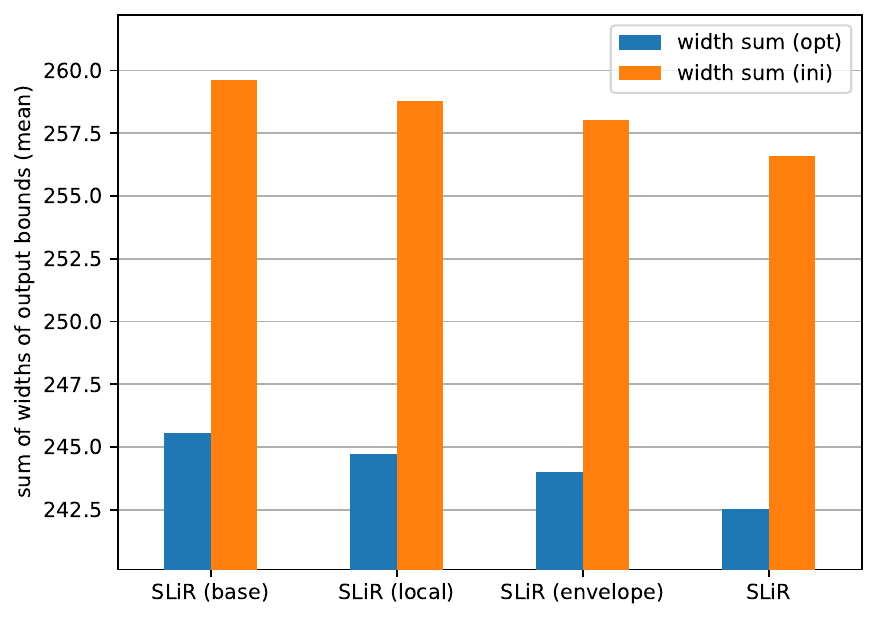}
    \caption{Effect of different improvements on the Piyavskii method on the width of the provable output bounds of the MNIST networks. All modifications improve upon the baseline. Combining both local Lipschitz constants (SLiR (local)) and the improved envelope computation (SLiR (envelope)) leads to further improvements (SLiR).}
    \label{fig:bounds-width-lipschitz}
\end{figure}

Figure \ref{fig:bounds-width-lipschitz} shows the mean of the summed width of the output bounds
\begin{align}
    \sum_i u_i - l_i \;,
\end{align}
where $l_i, u_i$ are the verifiable lower and upper bounds returned by the different configurations of SLiR for the $i$-th output of a network for both after initialization and after optimization.
The mean is taken over all considered images and networks.

Each of our modifications - using local Lipschitz constants and improved envelopes - leads to SLiR producing tighter output bounds than when just using the baseline Piyavskii method.
Moreover, combining our modifications - as used in the final SLiR algorithm - leads to an additional reduction in bounds width.
For the CIFAR benchmarks results are similar (see \Cref{fig:bounds-width-lipschitz-cifar}).

\begin{figure}[ht]
    \includegraphics[width=\linewidth]{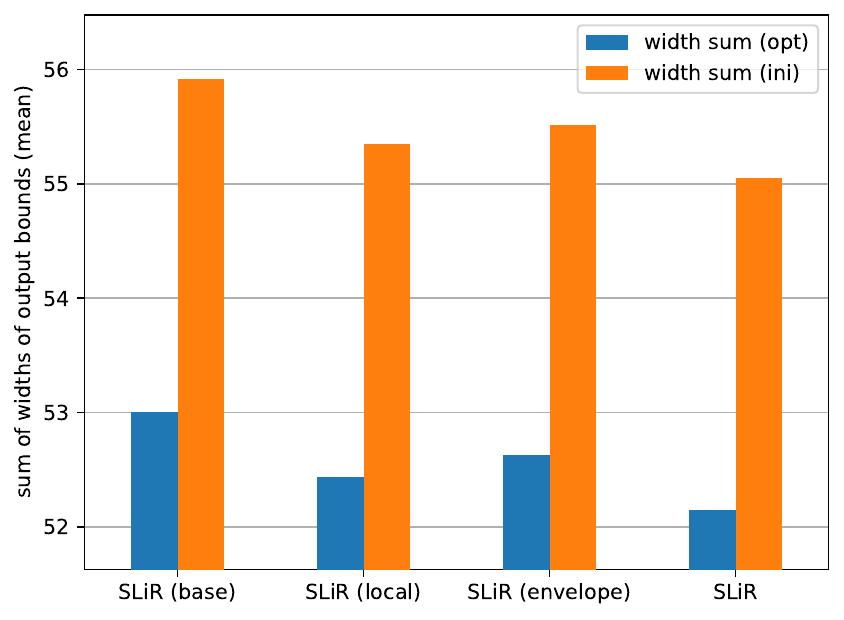}
    \caption{Effect of different improvements on the Piyavskii method on the width of the provable output bounds of the CIFAR networks. All modifications improve upon the baseline. Combining both local Lipschitz constants (SLiR (local)) and the improved envelope computation (SLiR (envelope)) leads to further improvements (SLiR).}
    \label{fig:bounds-width-lipschitz-cifar}
\end{figure}

\subsubsection{Failure Cases of Other Methods}
\label{ssec:failure-cases}

The experimental results in Sections \ref{ssec:comparison-prior-work} showed that our approach --- even after initialization --- can verify significantly more properties than the example guided approach \cite{Paulsen2022example} for the $\lisht$ and $\mish$ activation functions for the MNIST benchmark.

Similarly, we can verify more properties than $\alpha$-CROWN for the $\gelu$ and $\tanh$ activation functions --- both after initialization and after optimization --- even though, $\alpha$-CROWN utilizes manually derived custom relaxations in these cases.

\begin{figure}[ht]
    \includegraphics[width=\linewidth]{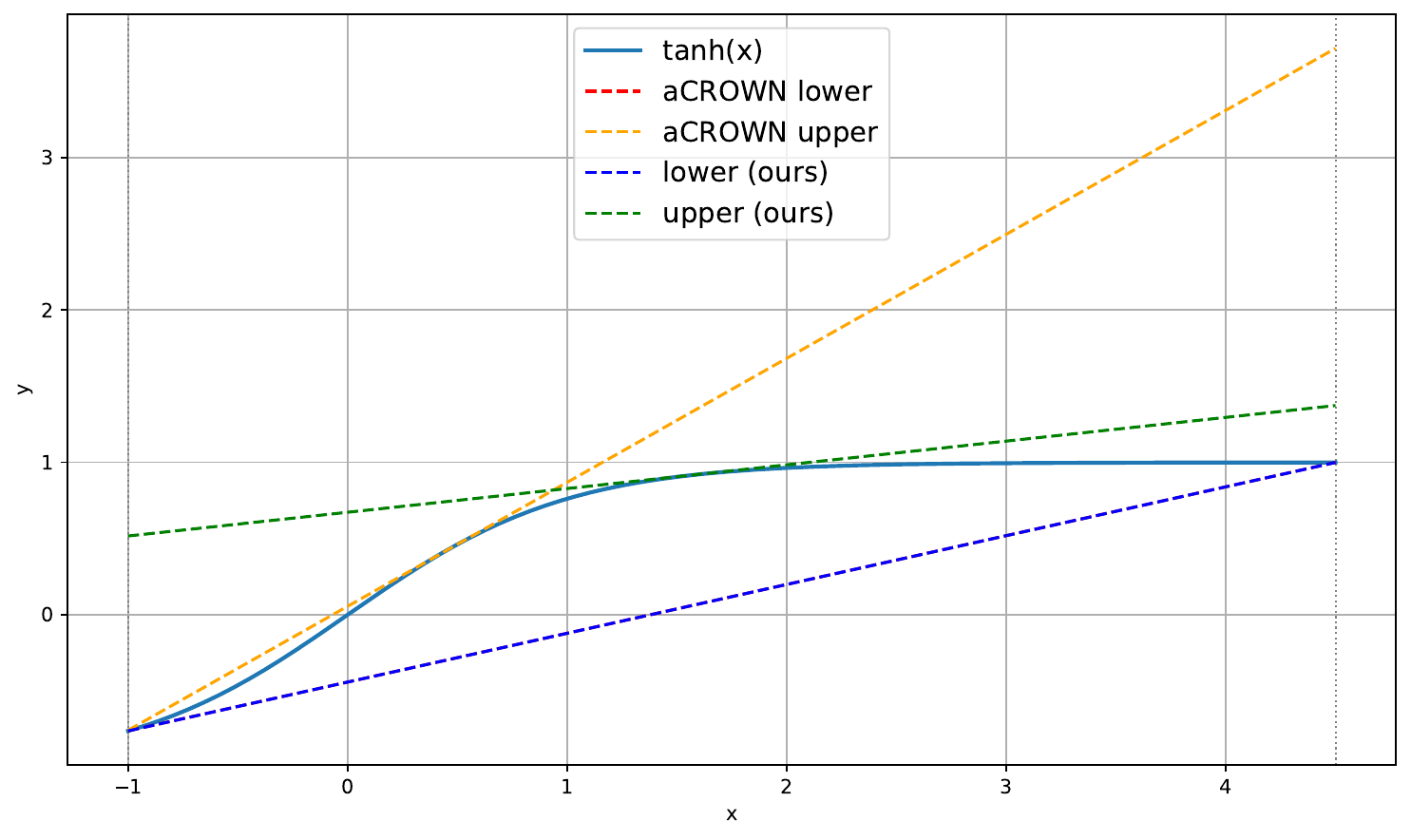}
    \caption{Example where $\alpha$-CROWN fails to find a good initial overapproximation for the $\tanh$ function. Enclosed area: $\alpha$-CROWN: $7.4823$, ours: $4.5441$}
    \label{fig:acrown-init-tanh}
\end{figure}

In this section, we show specific failure cases, where the (initial) relaxations chosen by the example guided approach and $\alpha$-CROWN are not as tight with respect to the area enclosed between the lower and the upper relaxation --- leading to suboptimal performance of those tools for these benchmarks.

\begin{figure*}[t]
    \centering
    \begin{subfigure}{0.3\textwidth}
        \centering
        \includegraphics[width=\textwidth]{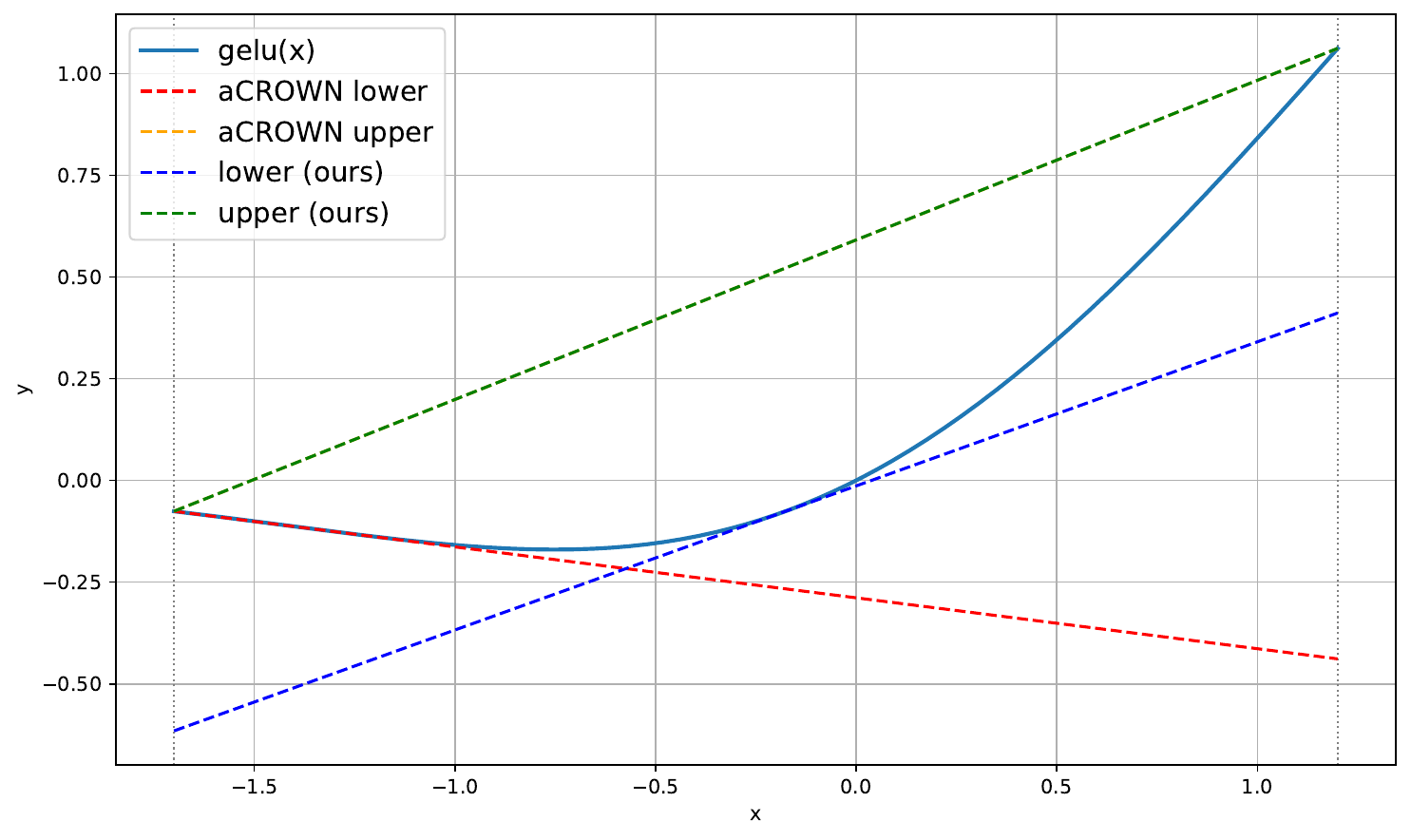}
        \caption{$\alpha$-CROWN: $2.1758$, ours: $1.7261$}
        \label{fig:acrown-init-tangent-left}
    \end{subfigure}
    \hfill
    \begin{subfigure}{0.3\textwidth}
        \centering
        \includegraphics[width=\textwidth]{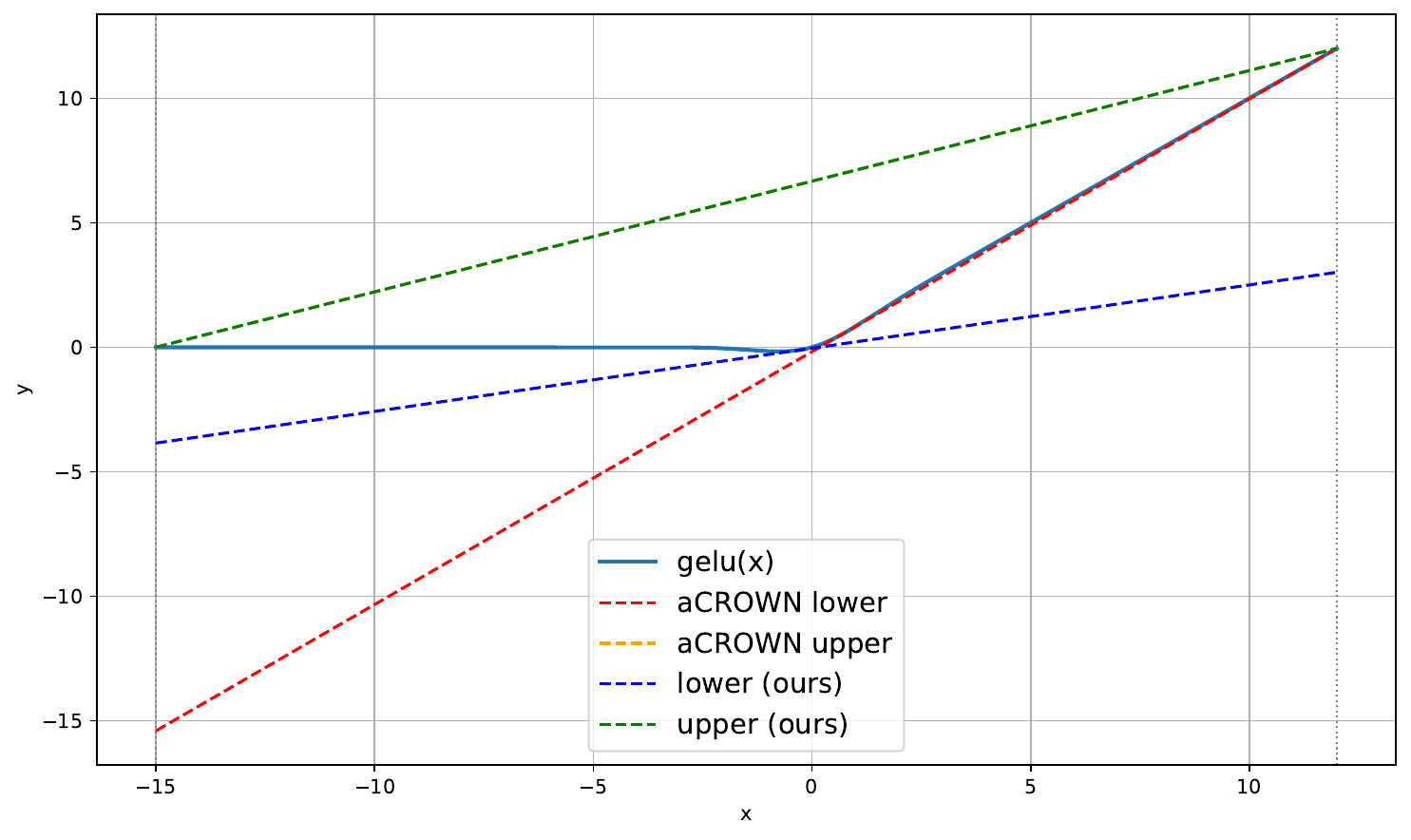}
        \caption{$\alpha$-CROWN: $208.0164$, ours: $173.3838$}
        \label{fig:acrown-init-tangent-right}
    \end{subfigure}
    \hfill
    \begin{subfigure}{0.3\textwidth}
        \centering
        \includegraphics[width=\textwidth]{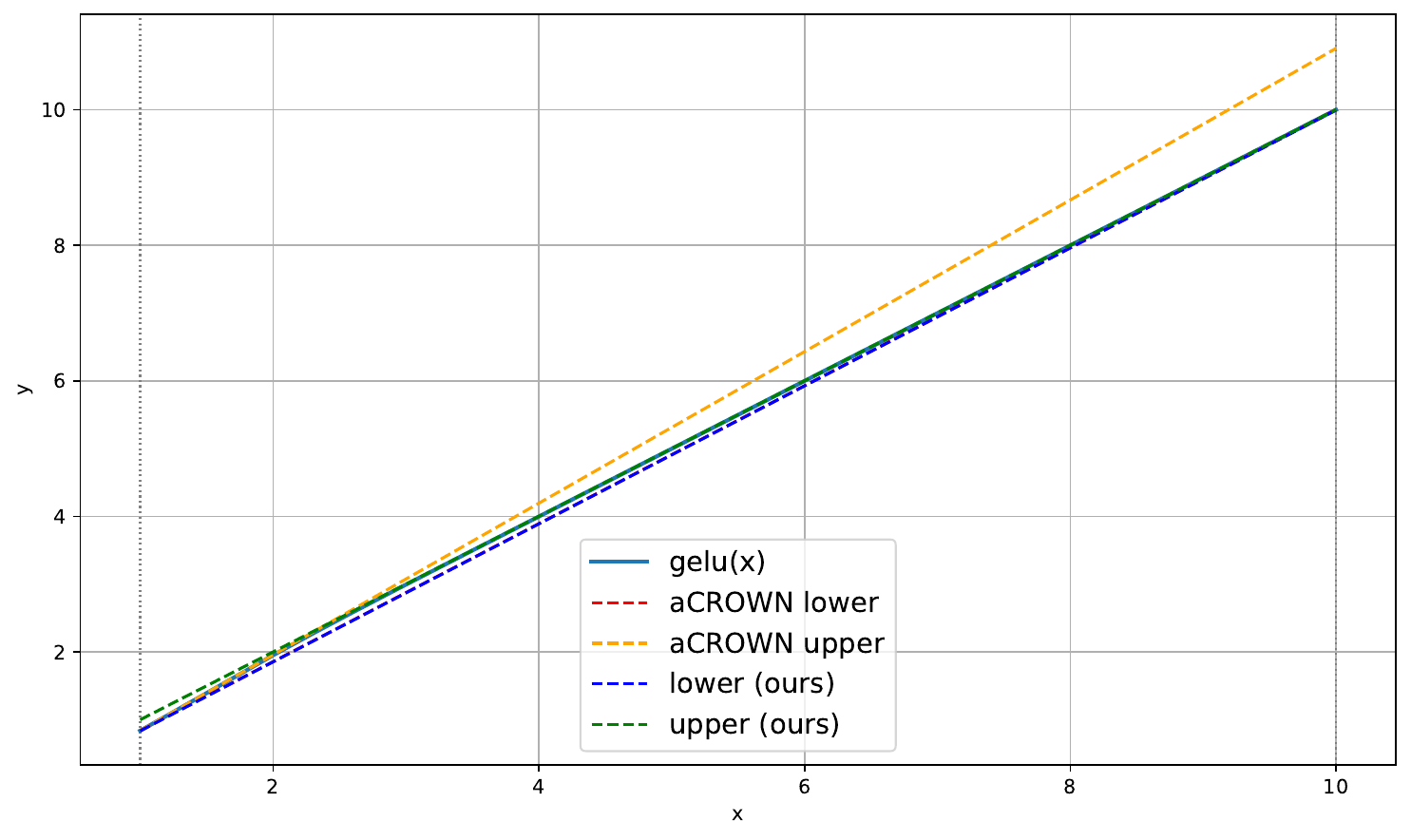}
        \caption{$\alpha$-CROWN: $4.0749$, ours: $0.7659$}
        \label{fig:acrown-init-pos}
    \end{subfigure}
    \caption{Cases where $\alpha$-CROWN fails to find good initial overapproximations for the $\gelu$ function. The area enclosed between the lower and upper relaxations can be found in the caption of each subfigure.}
    \label{fig:acrown-init}
\end{figure*}

While the E-GUIDED approach is generally capable of synthesizing tight linear relaxations, there exist configurations in which the combination of subinterval partitioning and tangent point estimation results in comparatively loose bounds. 
\Cref{fig:eguided-comp} illustrates representative cases of this behavior.

In \Cref{fig:eguided-comp-1,fig:eguided-comp-2}, the E-GUIDED method employs the same subinterval partitioning $l \in [-9.0,-8.0]$ and $u \in [-1.0,0.0]$, together with the same tangent-line template. 
For intervals in this partition, the tangent line is first constructed locally based on the predicted tangent point, and is then shifted upward by the maximum violation computed over the entire subinterval in order to ensure soundness.

In \Cref{fig:eguided-comp-1}, the primary source of looseness is the coarse partitioning of the input interval space. 
Even when the locally constructed tangent is close to tight for the concrete interval, the required global shift—computed with respect to all intervals in the subpartition—introduces a substantial degradation of the upper bound.

In contrast, \Cref{fig:eguided-comp-2} demonstrates a case in which the looseness is mainly attributable to suboptimal tangent point estimation. 
Although the required soundness shift is relatively small, the predicted tangent point itself does not induce a tight upper relaxation, and the resulting bound remains conservative.

Moreover, \Cref{fig:eguided-comp-3} shows that even in situations where the E-GUIDED relaxation appears reasonably tight, our approach can still derive a slightly tighter relaxation, resulting in a smaller enclosed area between lower and upper bounds.

Finally, we observe that the E-GUIDED approach may yield particularly loose bounds when the queried interval lies outside the predefined input universe ($[-10,10]$ in the shared artifact). 
In this case, the method reverts to its decomposition-based bounding procedure rather than employing the synthesized template-based relaxations. 
Such decomposition-based bounds are generally more conservative, and therefore tend to be significantly looser than the bounds obtained within the designated input universe.

\begin{figure*}[t]
    \centering
    \begin{subfigure}[t]{0.32\textwidth}
        \centering
        \includegraphics[width=\linewidth,trim=100 100 50 20,clip]{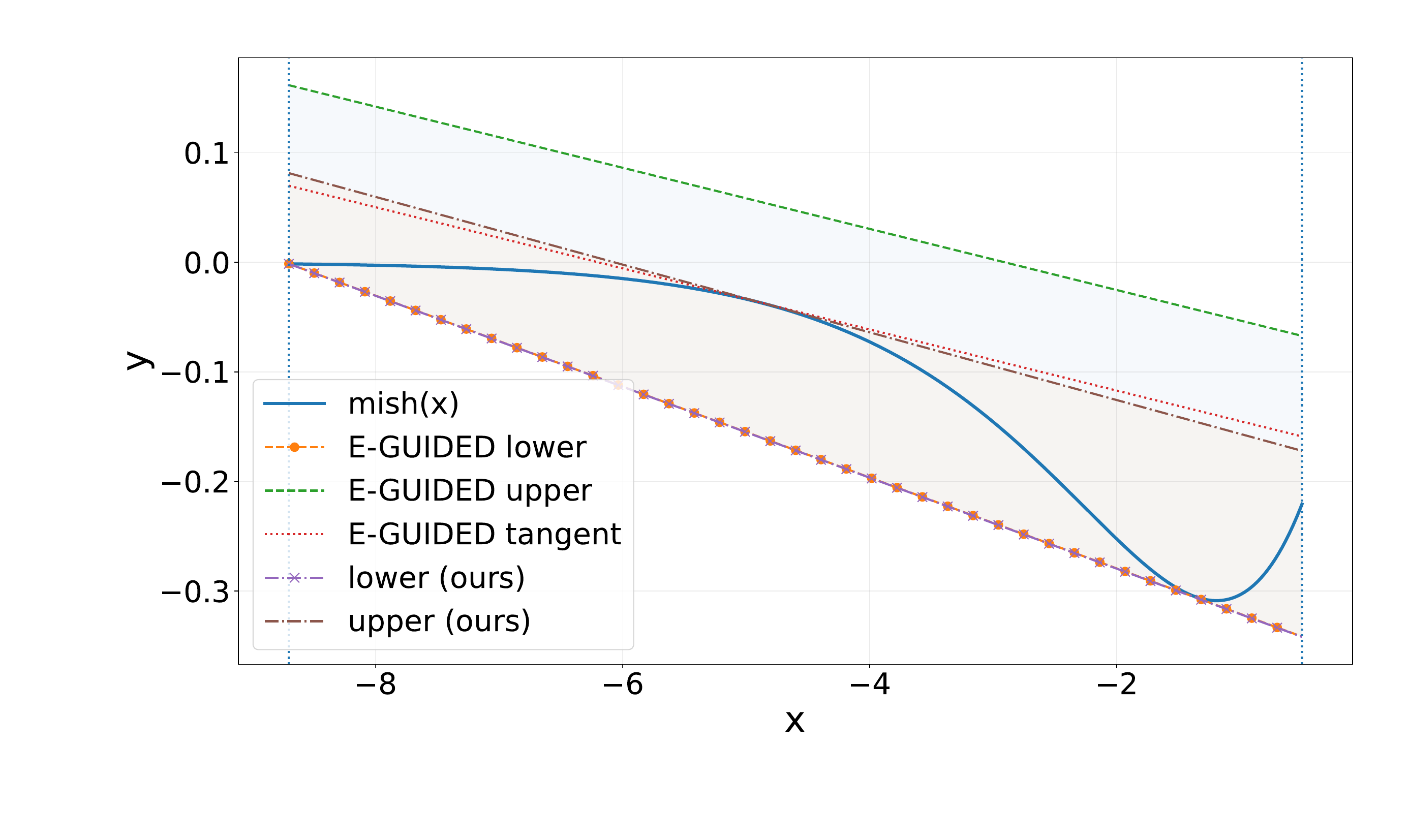}
        \caption{Input domain $[-8.70,-0.50]$. \\ E-GUIDED lower: $y=-0.0415x-0.3624$, E-GUIDED upper: $y=-0.0279x-0.0811$, E-GUIDED tangent: $y=-0.0279x-0.1729$, ours lower: $y=-0.0415x-0.3627$, ours upper: $y=-0.0309x-0.1876$. Relaxation area: E-GUIDED $1.7937$, ours $1.0357$.}
        \label{fig:eguided-comp-1}
    \end{subfigure}\hfill
    \begin{subfigure}[t]{0.32\textwidth}
        \centering
        \includegraphics[width=\linewidth,trim=100 100 50 20,clip]{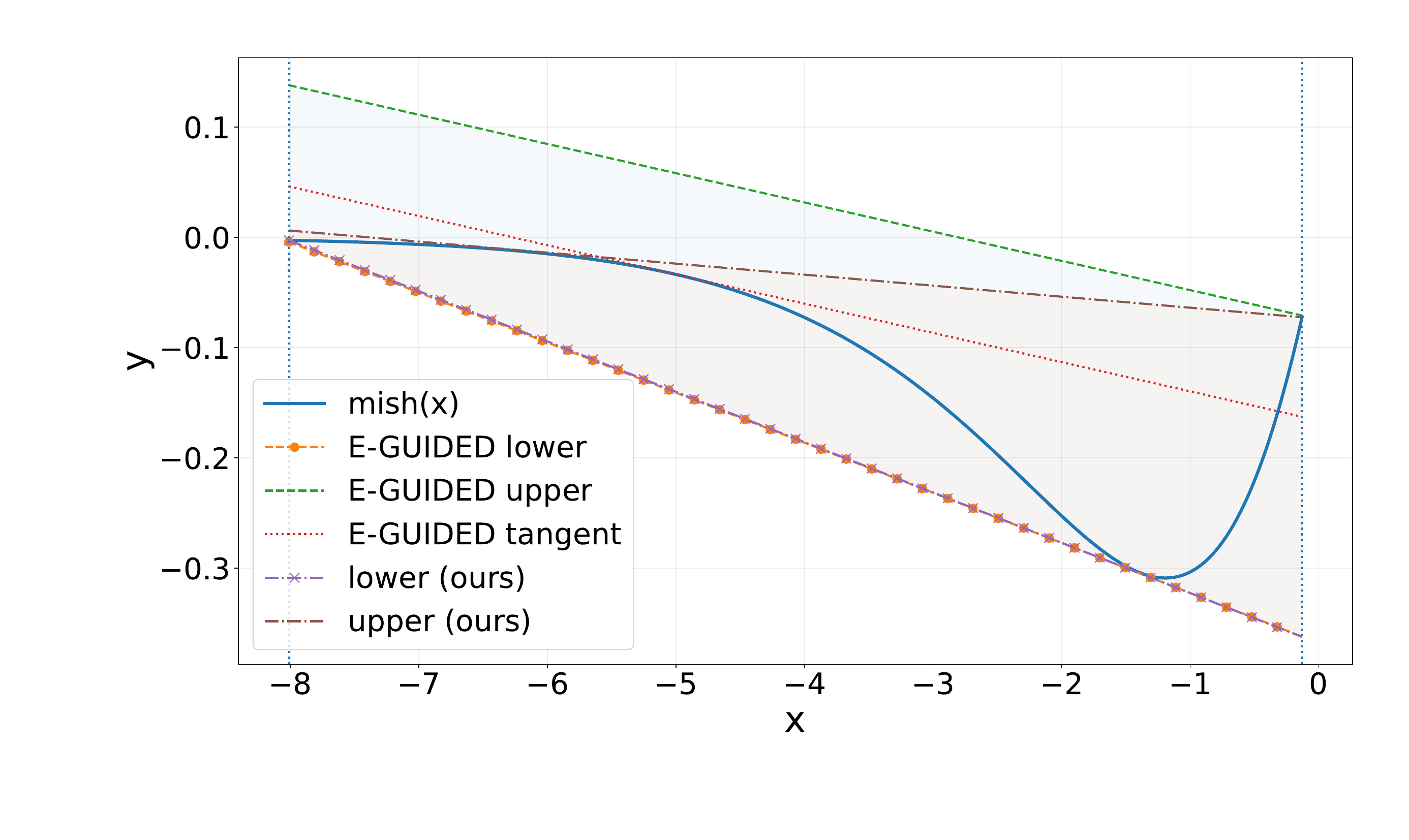}
        \caption{Input domain $[-8.01,-0.13]$. \\ E-GUIDED lower: $y=-0.0454x-0.3678$, E-GUIDED upper: $y=-0.0265x-0.0743$, E-GUIDED tangent: $y=-0.0265x-0.1661$, ours lower: $y=-0.0456x-0.3681$, ours upper: $y=-0.0100x-0.0738$. Relaxation area: E-GUIDED $1.7066$, ours $1.1767$.}
        \label{fig:eguided-comp-2}
    \end{subfigure}\hfill
    \begin{subfigure}[t]{0.32\textwidth}
        \centering
        \includegraphics[width=\linewidth,trim=100 100 50 20,clip]{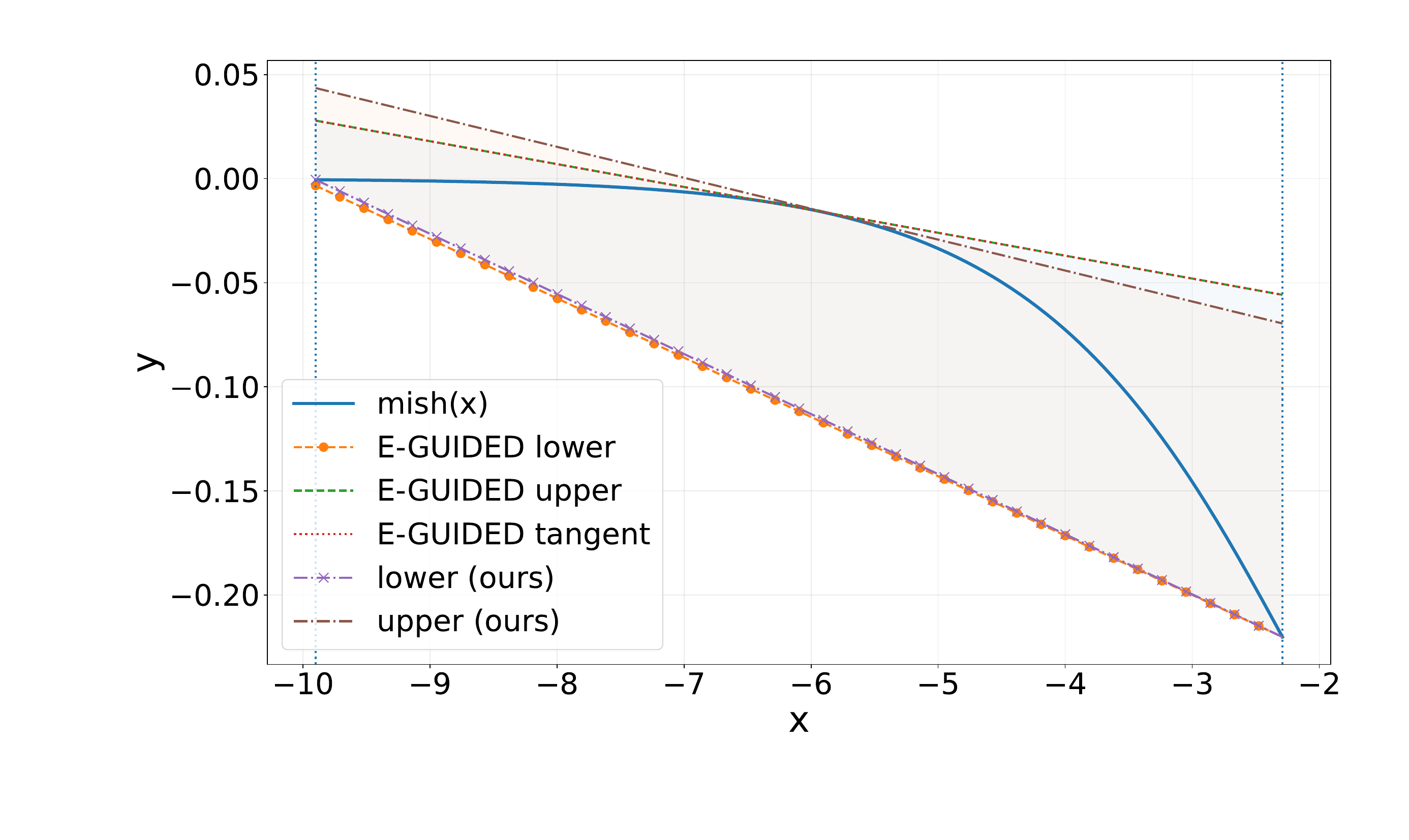}
        \caption{Input domain $[-9.90,-2.29]$. \\ E-GUIDED lower: $y=-0.0285x-0.2855$, E-GUIDED upper: $y=-0.0110x-0.0810$, E-GUIDED tangent: $y=-0.0110x-0.0810$, ours lower: $y=-0.0289x-0.2863$, ours upper: $y=-0.0149x-0.1036$. Relaxation area: E-GUIDED $0.7445$, ours $0.7409$.}
        \label{fig:eguided-comp-3}
    \end{subfigure}

    \caption{Cases where the E-GUIDED approach fails to obtain a tight overapproximation of the $\mathrm{mish}(x)$ activation. 
    The area enclosed between the lower and upper relaxations is reported in the caption of each subfigure. 
    In \Cref{fig:eguided-comp-1,fig:eguided-comp-2}, both methods produce identical lower relaxations, while in \Cref{fig:eguided-comp-3} our method yields a slightly tighter lower bound. 
    The improvement in \Cref{fig:eguided-comp-1,fig:eguided-comp-2} stems primarily from significantly tighter upper bounds. 
    In \Cref{fig:eguided-comp-3}, although the bounds are closer, our approach still achieves a smaller overall relaxation area.}
  \label{fig:eguided-comp}
\end{figure*}

\end{document}